\documentclass[10pt]{article}

\usepackage[utf8]{inputenc} 
\usepackage[T1]{fontenc}    
\usepackage{hyperref}       
\usepackage{url}            
\usepackage{booktabs}       
\usepackage{amsfonts}       
\usepackage{nicefrac}       
\usepackage{microtype}      
\usepackage{xcolor}         

\usepackage{natbib}

\usepackage{geometry}

\title{Provably Efficient UCB-type Algorithms For Learning Predictive State Representations}

\author{
  Ruiquan Huang\thanks{
  The Pennsylvania State University,
  State College, PA 16801},\quad 
   Yingbin Liang\thanks{The Ohio State University, Columbus, OH 43210},\quad 
   Jing Yang$^*$
}
\date{}

\usepackage{amsmath}
\usepackage{amssymb}
\usepackage{algorithm}
\usepackage{algpseudocode}
\usepackage{fontenc}
\usepackage{graphicx}

\usepackage{cleveref}

\usepackage{mathtools}
\usepackage{bbm}

\newtheorem{lemma}{Lemma}
\newtheorem{proposition}{Proposition}
\newtheorem{theorem}{Theorem}
\newtheorem{corollary}{Corollary}
\newtheorem{definition}{Definition}
\newtheorem{remark}{Remark}
\newtheorem{assumption}{Assumption}

\allowdisplaybreaks

\newenvironment{proof}[1]{\medskip\par\noindent
	{\bf Proof:\,}\,#1}{{\mbox{\,$\blacksquare$}\par}}

\newcommand{\Eb}{\mathbb{E}}
\newcommand{\Pb}{\mathbb{P}}
\newcommand{\Qb}{\mathbb{Q}}
\newcommand{\Dc}{\mathcal{D}}
\newcommand{\Ec}{\mathcal{E}}

\newcommand{\mbf}{\mathbf{m}}
\newcommand{\Mbf}{\mathbf{M}}

\newcommand{\TV}{\mathtt{TV}}

\newcommand{\jing}[1]{{\color{magenta}(JY: #1)}}

\newcommand{\hrq}[1]{{\color{blue}#1}}

\begin{document}

\maketitle

\begin{abstract}
The general sequential decision-making problem, which includes Markov decision processes (MDPs) and partially observable MDPs (POMDPs) as special cases, aims at maximizing a cumulative reward by making a sequence of decisions based on a history of observations and actions over time. Recent studies have shown that the sequential decision-making problem is statistically learnable if it admits a low-rank structure modeled by predictive state representations (PSRs). Despite these advancements, existing approaches typically involve oracles or steps that are computationally intractable. On the other hand, the upper confidence bound (UCB) based approaches, which have served successfully as computationally efficient methods in bandits and MDPs, have not been investigated for more general PSRs, due to the difficulty of optimistic bonus design in these more challenging settings. This paper proposes the first known UCB-type approach for PSRs, featuring a novel bonus term that upper bounds the total variation distance between the estimated and true models. We further characterize the sample complexity bounds for our designed UCB-type algorithms for both online and offline PSRs. In contrast to existing approaches for PSRs, our UCB-type algorithms enjoy computational tractability, last-iterate guaranteed near-optimal policy, and guaranteed model accuracy.

\end{abstract}

\section{Introduction}\label{sec:intro}

As a general framework of reinforcement learning (RL), the sequential decision-making problem aims at maximizing a cumulative reward by making a sequence of decisions based on a history of observations and actions over time.
This framework is powerful to include and generalize Markov Decision Processes (MDPs) and Partially Observable Markov Decision Processes (POMDPs), and captures a wide range of real-world applications such as recommender systems \citep{li2010contextual,wu2021partially},  business management \citep{de2009inventory}, economic simulation \citep{zheng2020ai}, robotics \citep{akkaya2019solving}, strategic games \citep{brown2018superhuman,vinyals2019grandmaster}, and medical diagnostic systems \citep{hauskrecht2000planning}. 

However, tackling POMDPs alone presents significant challenges, not to mention general sequential decision-making problems. Many hardness results have been developed \citep{mossel2005learning, mundhenk2000complexity, papadimitriou1987complexity,vlassis2012computational,krishnamurthy2016pac}, showing that learning POMDPs is already computationally and statistically intractable in the worst case. The reason is that the non-Markovian property of these problems implies that the sufficient statistics or belief about the current environmental state encompasses all observations and actions from past interactions with the environment. This dramatically increases the computational burden and statistical complexity, since even for a finite observation space and action space, the possibilities of the beliefs of the environment state are exponentially large in terms of the number of observations and actions.

In order to tackle these challenges, recent research has introduced various structural conditions for POMDPs and general sequential decision-making problems, such as reactiveness \citep{jiang2017contextual}, decodability \citep{du2019provably}, revealing conditions \citep{liu2022partially},  hindsight observability \citep{lee2023learning}, and low-rank representations with regularization conditions \citep{zhan2022pac,liu2022optimistic,chen2022partially}.  These conditions have opened up new possibilities for achieving polynomial sample complexities in general sequential decision-making problems. Among them, predictive state representations (PSRs) \citep{littman2001predictive} has been proved to capture and generalize a rich subclass of sequential decision-making problems such as MDPs and observable POMDPs. 
Yet, most existing solutions for PSRs involve oracles that might not be computationally tractable. For instance, Optimistic MLE (OMLE) \citep{liu2022optimistic} involves a step that maximizes the optimal value function over a confidence set of models. Typically, such a confidence set does not exhibit advantageous structures, resulting in a potentially combinatorial search within the set. Another popular posterior sampling based approach  \citep{agarwal2022model,zhong2022posterior} requires to maintain a distribution over the entire set of models, which is highly memory inefficient. In addition, most existing results lack a last-iterate guarantee and produce only mixture policies, which often exhibit a very large variance in practical applications. On the other hand, the upper confidence bound (UCB) based approach has been proved to be computationally efficient and provide last-iterate guarantee in many decision-making problems such as bandits \citep{auer2002finite} and MDPs \citep{menard2021fast}. However, due to the non-Markovian property of POMDPs and general sequential decision-making problems, designing an explicit UCB is extremely challenging. To the best of our knowledge, such a design has been seldomly explored in POMDPs and beyond. Thus, we are motivated to address the following important open question:

    { \it Q1: Can we design a UCB-type algorithm for learning PSRs that (a) is both computationally tractable and statistically efficient, and (b) enjoys the last-iterate guarantee?}

Another important research direction in RL is offline learning \citep{lange2012batch}, where the learning agent has access to a pre-collected dataset and aims to design a favorable policy without any interaction with the environment. While offline MDPs have been extensively studied \citep{jin2021pessimism,xiong2022nearly,xie2021bellman}, there exist very limited studies of offline POMDPs from the theoretical perspective \citep{guo2022provably,lu2022pessimism}.  To our best knowledge, offline learning for a more general model of PSRs has never been explored. Thus, we will further address the following research question:

    { \it Q2: Can we design a UCB-type algorithm for {\bf \em offline} learning of  PSRs with guaranteed policy performance and sample efficiency?}


\textbf{Main Contributions:}
We provide affirmative answers to both aforementioned questions by making the following contributions. 
\begin{list}{$\bullet$}{\topsep=0.ex \leftmargin=0.15in \rightmargin=0.1in \itemsep =0.01in}
\item We introduce the first known UCB-type approach to learning PSRs with only a regularization assumption, characterized by a novel bonus term that upper bounds the total variation distance between the estimated and true models. The bonus term is designed based on a new confidence bound induced by a new model estimation guarantee for PSRs, and is computationally {tractable}.
\item We theoretically characterize the performance of our UCB-type algorithm for learning PSRs online, called PSR-UCB. In contrast to existing approaches, PSR-UCB is computationally tractable with only supervised learning oracles, guarantees a near-optimal policy in the last iteration, and ensures model accuracy. When the rank of the PSR is small, our sample complexity matches the best known upper bound in terms of the rank and the accuracy level. 
\item We further extend our UCB-type approach to the offline setting, 
and propose the PSR-LCB algorithm. We then develop an upper bound on the performance difference between the output policy of PSR-LCB and any policy covered by the behavior policy. The performance difference scales in $O(C_{\infty}/\sqrt{K})$, where $C_{\infty}$ is the coverage coefficient and $K$ is the size of the offline dataset. This is the first known sample complexity result on offline PSRs. 

\item  Technically, we develop two key properties for PSRs to establish the sample complexity guarantees: (a) a new estimation guarantee on the distribution of future observations conditioned on empirical samples, enabled by the {\it stable} model estimation step, and (b) a new relationship between the empirical UCB and the ground-truth UCB. We believe these insights advance the current understanding of PSRs, and will benefit future studies on this topic. 

\end{list}

\section{Related Work}

{\bf Learning MDPs and POMDPs.}
The MDP is a basic model in RL that assumes Markovian property in the model dynamics, i.e. the distribution of the future states only depends on the current system state. Researchers show that learning tabular MDPs (with finite state and action spaces) is both computationally and statistically efficient in both online setting \citep{auer2008near,azar2017minimax,dann2017unifying,agrawal2017optimistic,jin2018q,li2021breaking} and offline setting \citep{jin2021pessimism,rashidinejad2021bridging,yin2021near,chang2021mitigating}. Learning MDPs with function approximations is also well-studied by establishing favorable statistical complexity and computation efficiency \citep{jin2020provably,wagenmaker2022instance,zanette2020learning,zhou2021nearly,agarwal2020flambe,uehara2021representation,du2021bilinear,foster2021statistical,jin2021bellman,jiang2017contextual,wang2020reinforcement,jin2021pessimism,xiong2022nearly,xie2021bellman}. Notably, several algorithms designed for learning MDPs with general function approximations can be extended to solve a subclass of POMDPs. In particular, OLIVE \citep{jiang2017contextual} and GOLF \citep{jin2021bellman}, which are originally designed for MDPs with low Bellman rank and low Bellman-Eluder dimension,  respectively, can efficiently learn reactive POMDPs, where the optimal policy only depends on the current observation. Besides,  \citet{du2019provably,efroni2022provable} study decodable RL where the observations determine the underlying states, and \citet{kwon2021rl} investigate latent MDPs where there are multiple MDPs determined by some latent variables.


To directly address the partial observability in POMDPs, some works assume exploratory data or reachability property and provide polynomial sample complexity for learning these POMDPs \citep{guo2016pac,azizzadenesheli2016reinforcement,xiong2022sublinear}. Others tackle the challenge of exploration and exploitation tradeoff in POMDPs by considering various sub-classes of POMDPs such as low-rank POMDPs \citep{wang2022embed}, observable POMDPs \citep{golowich2022learning,golowich2022planning}, hindsight observability \citep{lee2023learning}, and weakly-revealing POMDPs \citep{liu2022sample,liu2022partially}. Furthermore, 
\citet{liu2022sample,uehara2022computationally} propose computationally efficient algorithms for POMDPs with deterministic latent transitions.
Notably, \cite{golowich2022learning} propose a provably efficient algorithm for learning observable tabular POMDPs without computationally intractable oracles. Finally, \citet{lu2022pessimism} study the offline POMDPs in the presence of confounders, and  \citet{guo2022provably} provide provably efficient algorithm for offline linear POMDPs. After the initial submission of this work, we notice that a UCB-type algorithm has been studied by \citet{guo2023provably} under low-rank $L$-step decodable POMDPs, which is a subclass of the PSRs considered here.

{\bf Learning PSRs and general sequential decision-making problems.}
The PSR is first introduced by \citet{littman2001predictive,singh2012predictive} and considered as a general representation to model dynamic systems. A line of research \citep{boots2011closing,hefny2015supervised,jiang2018completing,zhang2022reinforcement} obtain polynomial sample complexity with observability assumption and spectral techniques. Later, \citet{zhan2022pac} demonstrate that learning regular PSRs is sample efficient and can avoid poly$(|\mathcal{O}|^m)$ in the sample complexity. \citet{uehara2022provably} propose a PO-bilinear class that captures a rich class of tractable RL problems with partial observations, including weakly revealing POMDPs and PSRs and design an actor-critic style algorithm.
For the works most closely related to ours, \citet{liu2022optimistic} propose a universal algorithm known as OMLE, which is capable of learning PSRs and its generalizations under certain conditions; \citet{chen2022partially} enhance the sample complexity upper bounds for three distinct algorithms, including OMLE, the model-based posterior sampling, and the estimation-to-decision type algorithm \citep{foster2021statistical}; \citet{zhong2022posterior} address the general sequential decision-making problem by posterior sampling under a newly proposed low generalized Eluder coefficient. However, as elaborated in \Cref{sec:intro}, those approaches are not efficient in terms of computational complexity or memory.

\section{Preliminaries}\label{sec:pre}
{\bf Problem Setting.} We consider a finite horizon episodic sequential decision-making problem, defined by a tuple $\mathtt{P} = (\mathcal{O},  \mathcal{A}, H, \Pb, R)$, where $\mathcal{O}$ represents the observation space, $\mathcal{A}$ is a finite action space, $H$ is the number of time steps within an episode, $\Pb = \{\Pb_h\}$ determines the model dynamics, i.e., $\Pb_h(o_h|o_1,\ldots,o_{h-1},a_1,\ldots,a_{h-1})$, where $o_t\in\mathcal{O}$ is the observation at time step $t$, and $a_t\in\mathcal{A}$ is the action taken by the agent at time step $t$ for all $t\in\{1,\ldots,h\}$,  and $R: (\mathcal{O}\times\mathcal{A})^H \rightarrow [0,1]$ is the reward function defined on trajectories of one episode. We denote a historical trajectory at time step $h$ as $\tau_h:=(o_1,a_1,\ldots,o_h,a_h)$, and denote a future trajectory as $\omega_h:=(o_{h+1},a_{h+1},\ldots, o_H,a_H)$. The set of all $\tau_h$ is denoted by $\mathcal{H}_h = (\mathcal{O}\times\mathcal{A})^{h}$ and the set of all future trajectories is denoted by $\Omega_h = (\mathcal{O}\times\mathcal{A})^{H-h}$.
In addition, let $\omega_h^o = (o_{h+1},\ldots,o_H)$ and $\omega_h^a = (a_{h+1},\ldots,a_H)$ be the observation sequence and the action sequence contained in $\omega_h$, respectively. Similarly, for a history $\tau_h$, we denote $\tau_h^o$ and $\tau_h^a$ as the observation and action sequences in $\tau_h$, respectively. Notably, the general framework of sequential decision-making problem subsume not only fully observable MDPs but  also POMDPs as special cases, because in MDPs, $\Pb_h(o_h|\tau_{h-1}) = \Pb_h(o_h|o_{h-1},a_{h-1})$ and in POMDPs, $\Pb_h(o_h|\tau_{h-1})$ can be factorized as $\Pb_h(o_h|\tau_{h-1}) = \sum_s\Pb_h(o_h|s)\Pb_h(s|\tau_{h-1})$, where $s$ represents unobserved states. 

The interaction between an agent and $\mathtt{P}$ proceeds as follows. At the beginning of each episode, the environment initializes a fixed observation $o_1$ at time step 1. After observing $o_1$, the agent takes action $a_1$, and the environment transits to $o_2$, which is sampled according to the distribution $\Pb_1(o_2|o_1,a_1)$. Then, at any time step $h\geq 2$, due to the non-Markovian nature of the problem, the agent takes action $a_h$ based on all past information $(\tau_{h-1},o_h)$, and the environment transits to $o_{h+1}$, sampled from $\Pb_h(o_{h+1}|\tau_{h})$. The interaction terminates after time step $H$.

The policy $\pi=\{\pi_h\}$ of the agent is a collection of $H$ distributions where $\pi_h(a_h|\tau_{h-1},o_h)$ is the probability of choosing action $a_h$ at time step $h$ given the history $\tau_{h-1}$ and the current observation $o_h$. For simplicity, we use $\pi(\tau_h) = \pi(a_h|o_h,\tau_{h-1})\cdots\pi(a_1|o_1)$ 
to denote the probability of the sequence of actions $\tau_h^a$ given the observations $\tau_h^o$. We denote $\Pb^{\pi}$ as the distribution of trajectories induced by policy $\pi$ under dynamics $\Pb$.  The value  
of a policy $\pi$ under $\Pb$ and the reward $R$ is denoted by $V_{\Pb,R}^{\pi} = \Eb_{\tau_H\sim\Pb^{\pi}}[R(\tau_H)]$.

The goal of the agent is to find an $\epsilon$-optimal policy $\hat{\pi}$ that satisfies $\max_{\pi}V_{\Pb,R}^{\pi} - V_{\Pb,R}^{\hat{\pi}} \leq \epsilon$.
Since finding a near-optimal policy for a general decision-making problem incurs exponentially large sample complexity in the worst case, in this paper we follow the line of research in \cite{zhan2022pac,chen2022partially,zhong2022posterior} and focus on the {\it low-rank} class of problems. To define a low-rank problem, we introduce the dynamic matrix $\mathbb{D}_h \in \mathbb{R}^{|\mathcal{H}_h|\times |\Omega_h|}$ for each $h$, where the entry at the $\tau_h$-th row and $\omega_h$-th column of $\mathbb{D}_h$ is $\Pb(\omega_h^o, \tau_h^o | \tau_h^a, \omega_h^a)$.

\begin{definition}[Rank-$r$ sequential decision-making problem]
    A sequential decision-making problem is rank $r$ if for any $h$, the model dynamic matrix $\mathbb{D}_h$ has rank $r$.
\end{definition}

{\bf Predictive State Representation (PSR).} To exploit the low-rank structure, we assume that for each $h$, there exists a set of future trajectories, namely, core tests (known to the agent) $\mathcal{Q}_h = \{\mathbf{q}_{h}^{1},\ldots, \mathbf{q}_h^{d_h}\} \subset \Omega_h$, 
such that the submatrix restricted to these tests $\mathbb{D}_h[\mathcal{Q}_h]$ has rank $r$, where $d_h\geq r$ is a positive integer. This special set $\mathcal{Q}_h$ allows the system dynamics to be factorized as $\Pb(\omega_h^o,\tau_h^o | \tau_h^a,\omega_h^a) =  \mbf(\omega_h)^{\top}\psi(\tau_h)$, where $\mathbf{m}(\omega_h), \psi(\tau_h)\in\mathbb{R}^{d_h}$ and the $\ell$-th coordinate of $\psi(\tau_h)$ is the joint probability of $\tau_h$ and the $\ell$-th core test $\mathbf{q}_h^{\ell}$. Mathematically, if we use $\mathbf{o}_h^{\ell}$ and $\mathbf{a}_h^{\ell}$ to denote the observation sequence and the action sequence of $\mathbf{q}_h^{\ell}$, respectively, then $ \Pb(\mathbf{o}_h^{\ell}, \tau_h^o | \tau_h^a, \mathbf{a}_h^{\ell}) =  [\psi(\tau_h)]_{\ell}$. By Theorem C.1 in \cite{liu2022optimistic}, any low-rank decision-making problem admits a (self-consistent) predictive state representation $\theta = \{\phi_h,\Mbf_h\}_{h=1}^H$ given core tests $\{\mathcal{Q}_h\}_{h=0}^{H-1}$, such that for any $\tau_h\in\mathcal{H}_h, \omega_h\in\Omega_h$,
\begin{align*}
    &\psi(\tau_h) = \Mbf_h(o_h,a_h)\cdots\Mbf_1(o_1,a_1)\psi_0,\quad \mbf(\omega_h)^{\top} = \phi_H^{\top}\Mbf_H(o_H,a_H)\cdots\Mbf_{h+1}(o_{h+1},a_{h+1})\\
    & \sum_{o_{h+1}}\phi_{h+1}^{\top}\Mbf_{h+1}(o_{h+1},a_{h+1}) = \phi_h^{\top},\quad \Pb(o_h,\ldots,o_1 | a_1,\dots, a_h) = \phi_h^{\top} \psi(\tau_h),
\end{align*}
where $\Mbf_h: \mathcal{O}\times\mathcal{A}\rightarrow\mathbb{R}^{d_h\times d_{h-1}}$, $\phi_h\in\mathbb{R}^{d_h}$, and $\psi_0\in\mathbb{R}^{d_0}$. For ease of presentation, we assume $\psi_0$ is known to the agent\footnote{The sample complexity of learning $\psi_0$ if it is unknown is relatively small compared with learning other parameters.}. Notably, the normalized version of $\psi(\tau_h)$ with respect to $\phi_h^{\top}\psi(\tau_h)$, denoted as $\bar{\psi}(\tau_h) = \psi(\tau_h)/\phi_h^{\top}\psi(\tau_h)$, is known as the prediction vector \citep{littman2001predictive} or prediction feature of $\tau_h$, since $[\bar{\psi}(\tau_h)]_{\ell} = \Pb(\mathbf{o}_h^{\ell} | \tau_h, \mathbf{a}_h^{\ell})$. As illustrated in \Cref{sec:online PSR,sec:offline PSR}, the prediction feature plays an important role in our algorithm design.

In the following context, we use $\Pb_{\theta}$ to indicate the model determined by the representation $\theta = \{\phi_h,\Mbf_h\}_{h=1}^H$. For simplicity, we denote $V_{\Pb_{\theta}, R}^{\pi} = V_{\theta,R}^{\pi}$. Moreover, let $\mathcal{Q}_h^A = \{\mathbf{a}_h^{\ell}\}_{\ell=1}^{d_h}$ be the set of action sequences that are part of core tests, constructed by eliminating any repeated action sequence. $\mathcal{Q}_h^A$, known as the set of core action sequences, plays a crucial role during the exploration process. Selecting from these sequences is sufficient to sample core tests, leading to accurate estimate of $\theta$.

We further assume that the PSRs studied in this paper are well-conditioned, as specified in \Cref{assmp:well-condition}. Such an assumption and its variants are commonly adopted in the study of PSRs \citep{liu2022optimistic,chen2022partially,zhong2022posterior}.

\begin{assumption}[$\gamma$-well-conditioned PSR]\label{assmp:well-condition}
A PSR $\theta$ is said to be $\gamma$-well-conditioned if
\begin{align}
    \forall h,~~ \max_{x\in\mathbb{R}^{d_h}:\|x\|_1\leq 1} \max_{\pi} {\max_{\tau_h}}\sum_{\omega_h}\pi(\omega_h|{\tau_h})|\mbf(\omega_h)^{\top}x|\leq \frac{1}{\gamma}.
\end{align}
\end{assumption}
\Cref{assmp:well-condition} requires that the error of estimating $\theta$ does not significantly blow up when the estimation error $x$ of estimating the probability of core tests is small.

{\bf Notations.} We denote the complete set of model parameters as $\Theta$ and the true model parameter as $\theta^*$. For a vector $x$, $\|x\|_A$ stands for $\sqrt{x^{\top}Ax}$, and the $i$-th coordinate of $x$ is represented as $[x]_i$. For functions $\mathbb{P}$ and $\mathbb{Q}$ (not necessarily probability measures) over a set $\mathcal{X}$, the total variation distance between them is $\mathtt{D}_{\TV}\left(\Pb(x),\Qb(x)\right) = \sum_x|\Pb(x)-\Qb(x)|$, while the hellinger-squared distance is defined as $\mathtt{D}_{\mathtt{H}}^2 (\Pb(x),\Qb(x)) = \frac{1}{2}\sum_x (\sqrt{\Pb(x)} - \sqrt{\Qb(x)} )^2$. Note that our definition of the total variation distance is slightly different from convention by a constant factor. We define $d = \max_h d_h$ and $Q_A = \max_h|\mathcal{Q}_h^A|$. 
We use $\nu_h(\pi,\pi')$  to denote the policy that takes $\pi$ at the initial $h-1$ steps and switches to $\pi'$ from the $h$-th step. Lastly, $\mathtt{u}_{\mathcal{X}}$ represents the uniform distribution over the set $\mathcal{X}$.

\section{Online Learning for Predivtive State Representations}\label{sec:online PSR}
In this section, we propose a model-based algorithm PSR-UCB, 
which features three main novel designs: (a) a stable model estimation step controlling the quality of the estimated model such that its prediction features of the empirical data are useful in the design of UCB, 
(b) an upper confidence bound that captures the uncertainty of the estimated model,
and (c) a termination condition that guarantees the last-iterate model is near-accurate and the corresponding greedy policy is near-optimal.  

\subsection{Algorithm}
The pseudo-code for the PSR-UCB algorithm is presented in \Cref{alg:computable PSR}. We highlight the key idea of PSR-UCB: in contrast to the design of UCB-type algorithms in MDPs that leverage low-rank structures, we exploit the physical meaning of the prediction feature $\bar{\psi}(\tau_h)$ to design bonus terms that enable efficient exploration.  In particular,  {``physical meaning'' refers to that each coordinate of a prediction feature of $\tau_h$ represents the probability of visiting a core test conditioned on $\tau_h$ and taking the corresponding core action sequence.} Therefore, it suffices to explore a set of $\tau_h$ whose prediction features can span the entire feature space, and use core action sequences to learn those ``base'' features.  
We next elaborate the main steps of the algorithm in greater detail as follows.

{\bf Exploration.} 
At each iteration $k$, PSR-UCB constructs a greedy policy $\pi^{k-1}$ based on a previous dataset $\Dc^{k-1} = \{\Dc_{h}^{k-1}\}_{h=0}^{H-1}$, together with an estimated model $\hat{\theta}^{k-1}$. Intuitively, to enable efficient exploration, $\pi^{k-1}$ is expected to sample $\tau_h$ that ``differs'' the most from previous collected samples $\tau_h\in\Dc_h^{k-1}$. How to quantify such differences forms the foundation of our algorithm design and will be elaborated later.

Then, for each $h\in[H]$, PSR-UCB first uses $\pi^{k-1}$ to get a sample $\tau_{h-1}^{k,h}$, then follows the policy $\mathtt{u}_{\mathcal{Q}_{h-1}^{\exp}}$ to get $\omega_{h-1}^{k,h}$, where $\mathcal{Q}_{h-1}^{\exp} = (\mathcal{A}\times \mathcal{Q}_h^A) \cup \mathcal{Q}_{h-1}^A$.  {Here, superscripts $k,h$ represent the index of episodes.} In other words, PSR-UCB adopts the policy $\nu_h(\pi^{k-1}, \mathtt{u}_{\mathcal{Q}_{h-1}^{\exp}})$ to collect a sample trajectory $(\tau_{h-1}^{k,h}, \omega_{h-1}^{k,h})$, which, together with $\Dc_{h-1}^{k-1}$, forms the dataset $\Dc_{h-1}^k$. 

The importance of uniformly selecting actions from $\mathcal{Q}_{h-1}^{\exp}$ can be explained as follows. First, the actions in $\mathcal{Q}_{h-1}^A$ assist us to learn the prediction feature $\bar{\psi}^*(\tau_{h-1}^{k,h})$ as its $\ell$-th coordiate equals to $\Pb_{\theta^*}(\mathbf{o}_{h-1}^{\ell} | \tau_{h-1}^{k,h}, \mathbf{a}_{h-1}^{\ell})$. Second, the action sequence $(a_h,\mathbf{a}_h^{\ell}) \in \mathcal{A}\times \mathcal{Q}_h^A$ helps us to estimate $\Mbf_h^*(o_h,a_h)\bar{\psi}^*(\tau_{h-1}^{k,h})$ because its $\ell$-th coordinate represents $\Pb_{\theta^*}( \mathbf{o}_h^{\ell},o_h |\tau_{h-1}^{k,h}, a_h, \mathbf{a}_h^{\ell} )$. Therefore, by uniform exploration on $\mathcal{Q}_{h-1}^{\exp}$ given that $\tau_{h-1}^{k,h}$ differs from previous dataset $\Dc_{h-1}^{k-1}$, PSR-UCB collects the most informative samples for estimating the true model $\theta^*$.

{\bf Stable model estimation.} With the updated dataset $\Dc^k = \{\Dc_h^{k}\}$, we estimate the model by maximizing the log-likelihood functions with constraints. Specifically, PSR-UCB extracts any model $\hat{\theta}^k$ from $\mathcal{B}^k$ defined as:
\begin{align}
    &\Theta_{\min}^k = \left\{\theta: \forall h, (\tau_h,\pi)\in\Dc_h^k,~~ \Pb_{\theta}^{\pi}(\tau_h) \geq    p_{\min} \right\},\nonumber\\
    &\mathcal{B}^k = \left\{\theta \in \Theta_{\min}^k:  \sum_{(\tau_H,\pi)\in\Dc^k} \log\Pb_{\theta}^{\pi}(\tau_H) \geq  \max_{\theta'\in\Theta_{\min}^k}\sum_{(\tau_H,\pi)\in\Dc^k} \log\Pb_{\theta'}^{\pi}(\tau_H) - \beta \right\}, \label{eqn: confidence set}
\end{align}
where the estimation margin $\beta$ is a pre-defined constant. Here the threshold probability $p_{\min}$ is sufficiently small to guarantee that, with high probability, $\theta^*\in\Theta_{\min}^k$. Note that compared to existing vanilla maximum likelihood estimators \citep{liu2022optimistic,liu2022partially,chen2022partially}, PSR-UCB has an additional constraint $\Theta_{\min}^k$. This is a crucial condition, as $\tau_h^{k,h+1}$ is sampled to infer the conditional probability $\Pb_{\theta^*}(\omega_h^o|\tau_h^{k,h+1},\omega_h^a)$ or the prediction feature $\bar{\psi}^*(\tau_h^{k,h+1})$, and learning this feature is useless or even harmful  {if $\Pb_{\hat{\theta}^k}^{\pi^{k-1}}(\tau_h^{k,h+1})$ is too small, as it is proved that $\theta^*\in\Theta_{\min}^k$ (\Cref{proposition: propobability of empirical data is high}). 
}

{\bf Design of UCB with prediction features.} From the discussion of the previous two steps, we see that the prediction features are vital since (a) actions in $\mathcal{Q}_h^{\exp}$ can efficiently explore the coordinates of these features, and (b) constraint $\Theta_{\min}^k$ ensures the significance of the learned features of the collected samples. Therefore, in the next round of exploration, our objective is to sample $\tau_h$ whose prediction feature exhibits the greatest dissimilarity compared with those of the previously collected samples $\tau_h\in\Dc_h^{k}$. Towards that, PSR-UCB constructs an upper confidence bound $V_{\hat{\theta}^k,\hat{b}^k}^{\pi}$ for $\mathtt{D}_{\TV}  (\Pb_{\hat{\theta}^k}^{\pi} (\tau_H) , \Pb_{\theta }^{\pi}(\tau_H)  )$, where the bonus $\hat{b}^k$ is defined as:
\begin{align}\label{eqn: bonus}
    &\hat{U}_{h}^k = \lambda I + \sum_{\tau_h'\in\Dc_h^{k}} \bar{\hat{\psi}}^k ({\tau_h'}) \bar{\hat{\psi}}^k ({\tau_h'})^{\top},\quad \hat{b}^k(\tau_H) = \min\left\{ \alpha \sqrt{ \sum_{h=0}^{H-1} \left\| \bar{\hat{\psi}}^k(\tau_h) \right\|^2_{(\hat{U}_{h}^k)^{-1}}}, 1 \right\},
\end{align}
with pre-specified regularizer $\lambda$ and UCB coefficient $\alpha$. Note that large $\hat{b}^k(\tau_H)$ indicates that $\tau_H$ is ``perceived'' to be under explored since the {\it estimated} prediction feature $\bar{\hat{\psi}}^k(\tau_h)$ is significantly different from $\{\bar{\hat{\psi}}^k(\tau_h)\}_{ \tau_h\in\Dc_h^k}$.     

\begin{algorithm}[ht]
\caption{Learning Predictive State Representation with Upper Confidence Bound (PSR-UCB)}\label{alg:computable PSR}
\begin{algorithmic}[1]
\State {\bf Input:} threshold probability $p_{\min}$, estimation margin $\beta$, regularizer $\lambda$, UCB coefficient $\alpha$.

\For{$k=1,\ldots, K$}

\For{$h=1,..., H$}
 
\State {\small Use $\nu(\pi^{k-1},\mathtt{u}_{\mathcal{Q}_{h-1}^{\exp} }  )$ to collect data  $\tau_{H}^{k,h} = (\omega_{h-1}^{k,h}, \tau_{h-1}^{k,h}) $.}

\State  {\small $\Dc_{h-1}^{k}\leftarrow \Dc_{h-1}^{ k-1 } \cup \{ ( \tau_H^{k,h} ,  \nu_h(\pi^{k-1}, \mathtt{u}_{\mathcal{Q}_{h-1}^{\exp} } ) ) \}$. } 
\EndFor

\State $\Dc^k = \{\Dc_h^k\}_{h=0}^{H-1}$. Extract any $\hat{\theta}^k\in\mathcal{B}^k$ according to \Cref{eqn: confidence set}.

\State Define bonus function $\hat{b}^{k}(\tau_H) $ according to \Cref{eqn: bonus}.

\State Solve $\pi^{k} = \arg\max_{\pi} V_{\hat{\theta}^{k},\hat{b}^k}^{\pi}$. 
\If{$ V_{\hat{\theta}^{k},\hat{b}^k}^{\pi^{k}} \leq \epsilon/2 $}
\State $ \theta^{\epsilon} = \hat{\theta}^{k}$, {\bf break.} 
\EndIf

\EndFor

\State 
{\bf Output:}  $\bar{\pi}=\arg\max_{\pi} V_{ \theta^{\epsilon},R}^\pi $.

\end{algorithmic}
\end{algorithm}
\vspace{-0.05in}

{\bf  Design of the greedy policy and last-iterate guaranteed termination.} The construction of UCB implies that $\hat{\theta}^k$ is highly uncertain on the trajectories $\tau_H$ with large bonuses. Thus, PSR-UCB finds a greedy policy $\pi^k = \arg\max_{\pi} V_{\hat{\theta}^k,\hat{b}^k}^{\pi}$ and terminates if $V_{\hat{\theta}^k,\hat{b}^k}^{\pi}$ is sufficiently small, indicating that the estimated model $\hat{\theta}^k$ is sufficiently accurate on any trajectory. Otherwise, $\pi^k$ serves for the next iteration $k+1$, as it tries to sample the most dissimilar $\bar{\hat{\psi}}^k(\tau_h)$ compared with the previous samples $\bar{\hat{\psi}}^k(\tau_h^{k,h+1})$ for efficient exploration. We remark that the termination condition favors the last-iterate guarantee, where a single model and a greedy policy are identified. Compared with algorithms that output a mixture policy (e.g. uniform selection in a large policy set), this guarantee may present lower variance in practical applications.


\begin{remark}[Computation]
Our algorithm only calls the MLE oracle $H$ times in each episode to construct the confidence model class in \Cref{eqn: confidence set}. Similar to \cite{guo2023provably}, our UCB based planning avoids searching for an optimistic policy that maximizes the optimal value over a large model class, thus is more amendable to practical implementation. 
\end{remark}

\begin{remark}[Reward-free PSRs]
    Our algorithm naturally handles reward-free PSRs since during the exploration, no reward information is needed. 
    In the reward-free setting, \Cref{alg:computable PSR} exhibits greater advantage compared with existing reward-free algorithms \citep{liu2022optimistic,chen2022partially} for PSRs, which require a potentially combinatorial optimization oracle over {\em a pair} of models in the (non-convex) confidence set and examination of the total variance distance\footnote{Computation complexity of calculating total variation distance is still an open question \citep{bhattacharyya2022approximating} and only some specific forms can be approximated in polynomial time.}, and are computationally intractable in general \citep{golowich2022learning}. 
\end{remark}

\subsection{Theoretical Results}
In this section, we present the theoretical results for PSR-UCB. To make a general statement, we first introduce the notion of the optimistic net of the parameter space. 

\begin{definition}[Optimistic net]\label{def:optimistic net}
    Consider two bounded functions $\Pb$ and $\bar{\Pb}$ over a set $\mathcal{X}$. Then, $\bar{\Pb}$ is $\varepsilon$-optimistic over $\Pb$ if (a) $\bar{\Pb}(x)\geq \Pb(x), \forall x\in\mathcal{X}$, and (b) $\mathtt{D}_{\TV}(\Pb,\bar{\Pb})\leq \varepsilon$. The $\varepsilon$-optimistic net of $\Theta$ is a smallest finite space $\bar{\Theta}_{\varepsilon}$ so that for all $\theta\in\Theta$, there exists $\bar{\theta}\in\bar{\Theta}_{\varepsilon}$ when $\Pb_{\bar{\theta}}$ is $\varepsilon$-optimistic over $\Pb_{\theta}$.
\end{definition}

Note that if $\Theta$ is the parameter space for tabular PSRs (including finite observation and action spaces) with rank $r$, we have $|\bar{\Theta}_{\varepsilon}|\leq r^2|\mathcal{O}||\mathcal{A}|H^2\log\frac{H|\mathcal{O}||\mathcal{A}|}{\varepsilon}$ (see \Cref{prop: optimistic net of tabular PSR} or Theorem 4.7 in \citet{liu2022optimistic}). Now, we are ready to present the main theorem for PSR-UCB.

\begin{theorem}\label{thm:online PSR}
    Suppose \Cref{assmp:well-condition} holds. Let $p_{\min} = O(\frac{\delta}{KH|\mathcal{O}|^H|\mathcal{A}|^H })$\footnote{$|\mathcal{O}|$ can be the cardinality of $\mathcal{O}$ if it is finite, or the measure of $\mathcal{O}$ if it is a measurable set with positive and bounded measure.}, $\beta = O(\log|\bar{\Theta}_{\varepsilon}|)$, where $\varepsilon=O(\frac{p_{\min}}{KH})$, $\lambda = \frac{\gamma |\mathcal{A}|^2Q_A \beta  \max\{ \sqrt{r}, Q_A\sqrt{ H}/\gamma \}   }{  \sqrt{d H} }$, and $\alpha = O\left( \frac{Q_A\sqrt{ H d}}{\gamma^2} \sqrt{\lambda } + \frac{|\mathcal{A}|Q_A\sqrt{\beta}}{\gamma}\right)$. Then, with probability at least $1-\delta$, PSR-UCB~ outputs a model $\theta^{\epsilon}$ and a policy $\bar{\pi}$ that satisfy
    \begin{align}
        V_{\theta^*,R}^{\pi^*} - V_{\theta^*,R}^{\bar{\pi}} \leq \epsilon, \text{ and }~~ \forall \pi,~~\mathtt{D}_{\TV} \left( \Pb_{ \theta^{\epsilon} }^{\pi}(\tau_H), \Pb_{\theta^*}^{\pi}(\tau_H)\right) \leq \epsilon.
    \end{align}
    In addition, PSR-UCB terminates with a sample complexity of 
    \[ \tilde{O}\left( \left(  r + \frac{Q_A^2 H}{\gamma^2}\right) \frac{r d H^3 |\mathcal{A}|^2 Q_A^4 \beta}{\gamma^4\epsilon^2} \right).\]
\end{theorem}

The following remarks highlight a few insights conveyed by \Cref{thm:online PSR}. {\bf First,} \Cref{thm:online PSR} explicitly states that PSR-UCB features last-iterate guarantee, i.e., the guaranteed performance is on the last output of the algorithm. This is in contrast to the previous studies \citep{liu2022partially,liu2022optimistic,chen2022partially} on POMDP and/or PSRs, where the performance guarantee is on a mixture of policies obtained over the entire execution of algorithms. Such policies often have a large variance. {\bf Second,}
in the regime with a low PSR rank such that $r<\frac{ Q_A^2H}{\gamma^2}$, {our result matches the best known sample complexity \citep{chen2022partially} in terms of rank $r$ and $\epsilon$.} {\bf Third,} 
thanks to the explicitly constructed bonus function $\hat{b}$, whose computation complexity is polynomial in terms of the iteration number and the dimension $d$, PSR-UCB is computationally tractable, provided that there exist oracles for planning in POMDPs (e.g. Line 9 in PSR-UCB) \citep{papadimitriou1987complexity,guo2023provably} and maximum likelihood estimation. 

{\it Proof Sketch of \Cref{thm:online PSR}.} The proof relies on {three} main steps  corresponding to three new technical developments, respectively. {\bf (a) A new MLE guarantee.} Due to the novel {\it stable} model estimation design, we establish a new MLE guarantee (which has not been developed in the previous studies) that the total variation distance between the conditional distributions of future trajectories conditioned on $(\tau_h,\pi)\in\Dc_h^k$ is small. Mathematically, $\sum_{ (\tau_h,\pi)\in\Dc_h^k }\mathtt{D}_{\TV}  (\Pb_{\hat{\theta}^k}^{\pi} (\omega_h|\tau_h) , \Pb_{\theta^* }^{\pi}( \omega_h | \tau_h) )$ is upper bounded by a constant. {\bf (b) A new confidence bound.} Based on (a) and the observation that the estimation error of the prediction feature $\bar{\hat{\psi}}^k(\tau_h)$ is upper bounded by the total variation distance $\mathtt{D}_{\TV}  (\Pb_{\hat{\theta}^k}^{\pi} (\omega_h|\tau_h) , \Pb_{\theta^* }^{\pi}( \omega_h | \tau_h) )$, the estimation error of other prediction features is captured by the function $\|\bar{\hat{\psi}}^k(\tau_h)\|_{(\hat{U}_h^k)^{-1}}$ up to some constant, leading to the valid UCB design. 
{\bf (c) A new relationship between the empirical bonus and the ground-truth bonus.}  To characterize the sample complexity, we need to show that $V_{\hat{\theta}^k,\hat{b}^k}^{\pi^k}$ can be small for some $k$. One approach is to prove that 
$\sum_{k}V_{\hat{\theta}^k,\hat{b}^k}^{\pi^k}$ is sub-linear. We validate this sub-linearity by establishing $\hat{b}^k \leq O\left(\sum_h \|\bar{\psi}^*(\tau_h)\|_{(U_h^k)^{-1}}\right)$, where $U_h^k = \lambda I + \sum_{\tau_h\in\Dc_h^k} \bar{\psi}^*(\tau_h)\bar{\psi}^*(\tau_h)^{\top}$. This inequality bridges the empirical and the ground-truth bonuses, and yields the final result when combined with the elliptical potential lemma \citep{carpentier2020elliptical}. The complete proof can be found in \Cref{sec:online PSR}.

When PSR is specialized to $m$-step decodable POMDPs \citep{liu2022optimistic}, as elaborated in \Cref{sec:POMDP}, the sample complexity of PSR-UCB does not depend on $d$ with slight modification of the algorithm. 

\begin{corollary}\label{coro:POMDP}
  When PSR is specialized to $m$-step decodable POMDPs, for each $k$, there exists a set of matrices $\{\hat{\mathbf{G}}_h^k\}_{h=1}^H \subset \mathbb{R}^{d\times r}$ such that if we replace $\bar{\hat{\psi}}^k(\tau_h)$ by $(\hat{\mathbf{G}}_{h+1}^k)^{\dagger}\bar{\hat{\psi}}^k(\tau_h)$  in \Cref{eqn: bonus} for any $\tau_h$, then PSR-UCB terminates with a sample complexity of poly$(r,1/\gamma, Q_A, H, |\mathcal{A}|,\beta)/\epsilon^2$.
  
\end{corollary}

\section{Offline Learning for Predictive State Representations}\label{sec:offline PSR}

In this section, we develop a computationally tractable algorithm PSR-LCB to learn PSRs in the offline setting. In offline PSRs, a dataset $\Dc$ contains $K$ pre-collected trajectories that are independently sampled from a behavior policy $\pi^b$. We slightly generalize the learning goal to be finding a policy that can compete with any target policy whose coverage coefficient is finite. We highlight that PSR-LCB is the first offline algorithm for learning PSRs. When specialized to POMDPs, our proposed algorithm enjoys better computational complexity than confidence region based algorithms \citep{guo2022provably}. 

\hrq{}
\subsection{Algorithm}

We design a PSR-LCB algorithm for offline PSRs, whose pseudo-code is presented in \Cref{alg:offline PSR}.  In the offline setting, we also leverage the prediction feature $\bar{\psi}(\tau_h)$ to design lower confidence bound (LCB) of the true value function $V_{\theta^*,r}^{\pi}$.  We explain the main steps of the algorithm in detail as follows.


{\bf stable model estimation.} Inspired by PSR-UCB, where the prediction features $\bar{\psi}(\tau_h)$ at each step $h$ are learned through separate datasets, PSR-LCB first randomly and evenly divides $\Dc$ into $H$ datasets $\Dc_0,\ldots, \Dc_{H-1}$. The goal of this division is to separately learn $\bar{\psi}^*(\tau_h)$ for each $h\in\{0,1,\ldots,H-1\}$.  Then, PSR-LCB extracts a model $\hat{\theta}$ from $\mathcal{B}_{\Dc}$ defined as:
\begin{align}
    & \Theta_{\min} = \left\{\theta: \forall h,  \tau_h \in \Dc_h,~~ \Pb_{\theta}^{\pi^b}(\tau_h) \geq    p_{\min} \right\}, \nonumber \\
    &\mathcal{B}_{\Dc} = \left\{\theta \in \Theta_{\min}:  \sum_{\tau_H \in\Dc} \log\Pb_{\theta}^{\pi^b}(\tau_H) \geq  \max_{\theta'\in\Theta_{\min}}\sum_{ \tau_H  \in\Dc} \log\Pb_{\theta'}^{\pi^b} (\tau_H) - \hat{\beta} \right\}, \label{eqn:offline confidence set}
\end{align}
where the estimation margin $\hat{\beta}$ is a pre-specified constant, and $p_{\min}$ guarantees that, with high probability, $\theta^*\in\Theta_{\min}$. 

Following similar reasons of the design of PSR-UCB, the constraint $\Theta_{\min}$ controls the quality of the estimated model such that the prediction features of the behavior samples are non-negligible and useful for learning.

{\bf Design of LCB with prediction features.} Given the estimated model $\hat{\theta}$, PSR-LCB constructs an upper confidence bound of  $\mathtt{D}_{\TV} \left(\Pb_{\hat{\theta} }^{\pi} (\tau_H) , \Pb_{\theta }^{\pi}(\tau_H) \right)$ in the form of $V_{\hat{\theta} ,\hat{b} }^{\pi}$, where $\hat{b}(\tau_H)$ is defined as: 
\begin{align}\label{eqn: LCB bonus}
    &\hat{U}_{h}  = \hat{\lambda} I + \sum_{\tau_h\in\Dc_h } \bar{\hat{\psi}} (\tau_h) \bar{\hat{\psi}} (\tau_h)^{\top},\quad \hat{b} (\tau_H) = \min\left\{ \hat{\alpha} \sqrt{ \sum_h \left\| \bar{\hat{\psi}} (\tau_h) \right\|^2_{(\hat{U}_{h} )^{-1}}}, 1 \right\},
\end{align}
with pre-defined regularizer $\hat{\lambda}$ and LCB coefficient $\hat{\alpha}$. Note that we refer to $\hat{\alpha}$ as the LCB coefficient, as we adopt the pessimism principle in offline learning. Differently from PSR-UCB, we aim to select a policy where the estimated model exhibits the least uncertainty. Therefore, the output policy should allocate a high probability to $\tau_H$ if $\hat{b}(\tau_H)$ is small.

{\bf Output policy design.} Building on the discussion above, PSR-LCB outputs a policy $\bar{\pi} = \arg\max_{\pi}  V_{ \hat{\theta}, R  }^{\pi} - V_{ \hat{\theta}, \hat{b} }^{\pi}$, which maximizes a {\it lower confidence bound} of $V_{\theta^*,R}^{\pi}$.

\vspace{-0.05in}
\begin{algorithm}[ht]
\caption{Offline Predictive State Representations with Lower Confidence Bound (PSR-LCB)}\label{alg:offline PSR}
\begin{algorithmic}[1]
\State {\bf Input:} Offline dataset $\Dc$, threshold probability $p_{\min}$, estimation margin $\hat{\beta}$, regularizer $\hat{\lambda}$, LCB coefficient $\hat{\alpha}$.

\State Estimate model  $\hat{\theta}\in \mathcal{B}$ according to \Cref{eqn:offline confidence set}.

\State Construct $\hat{b}$ according to \Cref{eqn: LCB bonus}.

\State Output
    $\hat{\pi} = \arg\max_{\pi}  V_{ \hat{\theta}, R }^{\pi} - V_{ \hat{\theta}, \hat{b} }^{\pi}$.

\end{algorithmic}
\end{algorithm}

\subsection{Theoretical Results}
In this section, we develop the theoretical guarantee for PSR-LCB. To capture the distribution shift in offline PSRs, we require that the $\ell_{\infty}$ norm of the ratio between the probabilities over the entire trajectory under the target policy and under the behavior policy is finite. Mathematically, if $\pi$ is the target policy, then,
   $ C_{\pi^b,\infty}^{\pi} :=    \max_h \max_{\tau_h}\frac{ \Pb_{\theta^*}^{\pi}(\tau_h) }{\Pb_{\theta^*}^{\pi^b}(\tau_h)}  < \infty.$ 
We note that, for general sequential decision-making problems, no Markovian property or hidden state is assumed. Therefore, defining the coverage assumption on a single observation-action pair as in \cite{xie2021bellman} for offline MDP, or on hidden states as in \cite{guo2022provably} for offline POMDP, is not suitable for PSRs.

\begin{theorem}\label{thm:offline PSR}
Suppose \Cref{assmp:well-condition} holds. Let $\iota = \min_{\mathbf{a}_h\in\mathcal{Q}_h^{\exp}} \pi^b(\mathbf{a}_h) $, $p_{\min} = O(\frac{\delta}{KH(|\mathcal{O}||\mathcal{A}|)^H})$, $\varepsilon=O(\frac{p_{\min}}{KH})$, $\hat{\beta} = O(\log|\bar{\Theta}_{\varepsilon}|)$, $ \hat{\lambda}  = \frac{\gamma C_{\pi^b,\infty}^{\pi}\hat{\beta} \max\{\sqrt{r}, Q_A\sqrt{H}/\gamma\}}{\iota^2Q_A\sqrt{dH}}$, and $\hat{\alpha} = O\left(\frac{  Q_A\sqrt{dH} }{\gamma^2}\sqrt{\lambda} + \frac{  \sqrt{\beta} }{ \iota \gamma} \right)$. Then, with probability at least $1-\delta$, the output $\bar{\pi}$ of \Cref{alg:offline PSR} satisfies that
   \begin{align}
       \forall \pi,~~ V_{\theta^*, R }^{\pi} - V_{\theta^*,R}^{\bar{\pi}} \leq \tilde{O}\left( \left(\sqrt{r} + \frac{ Q_A\sqrt{H}}{\gamma}\right)\frac{ C_{\pi^b,\infty}^{\pi}  Q_A H^2\sqrt{d}  }{  \iota  \gamma^2 }\sqrt{\frac{r\hat{\beta}}{K}} \right).
   \end{align}
\end{theorem}

\Cref{thm:offline PSR} states that for any target policy with finite coverage coefficient $C_{\pi^b,\infty}^{\pi}$, the performance degradation of the output policy $\bar{\pi}$ with respect to the target policy $\pi$ is at most $\tilde{O}\big( \frac{C_{\pi^b,\infty}^{\pi}}{\sqrt{K}} \big)$, which is negligible if the size of the offline dataset $K$ is sufficiently large. Thee full proof is in ~\Cref{sec:offline PSR}.

\if{0}
{\it Proof Sketch of \Cref{thm:offline PSR}.} Similar to the online setting, the proof first shows that $V_{\hat{\theta}, \hat{b}}^{\pi}$ serves as a valid upper confidence bound of the total variation distance $\mathtt{D}_{\TV}\left(\Pb_{\hat{\theta}}^{\pi}, \Pb_{\theta^*}^{\pi}\right)$ for any policy $\pi$. We emphasize that the validity is again due to the development of the estimation guarantee of the distributions of future trajectories conditioned on empirical samples $\tau_h\in\Dc_h$. Then, we translate the empirical bonus term $\hat{b}$ to $\sum_h\|\bar{\psi}^*(\tau_h)\|_{U_h^{-1}}$, where $U_h = \lambda I + \sum_{\tau_h\in\Dc_h}\bar{\psi}^*(\tau_h)\bar{\psi}^*(\tau_h)^{\top}$. Since $\bar{\psi}^*$ is a fixed representation function, we leverage the coverage assumption to change the distribution induced by the target policy $\pi$ to the distribution induced by the behavior policy $\pi^b$, and then apply Azuma-Hoeffding's inequality to obtain the final result. 
\fi

\section{Conclusion}
We studied learning predictive state representations (PSRs) for low-rank sequential decision-making problems. We developed a novel upper confidence bound for the total variation distance of the estimated model and the true model that enables both computationally  {tractable with MLE oracles} and statistically efficient learning for PSRs with only supervised learning oracles. Specifically, we proposed PSR-UCB for online learning for PSRs with last-iterate guarantee, i.e. producing not only a near-optimal policy, but also a near-accurate model. The statistical efficiency was validated by a polynomial sample complexity in terms of the model parameters. In addition, we extended this UCB-type approach to the offline setting and proposed PSR-LCB. Our theoretical result offers an initial perspective on offline PSRs by demonstrating that PSR-LCB outputs a policy that can compete with any policy with finite coverage coefficient.

\if{0}
\subsubsection*{Author Contributions}
If you'd like to, you may include  a section for author contributions as is done
in many journals. This is optional and at the discretion of the authors.

\subsubsection*{Acknowledgments}
Use unnumbered third level headings for the acknowledgments. All
acknowledgments, including those to funding agencies, go at the end of the paper.
\fi

\bibliography{PSR}
\bibliographystyle{apalike}

\newpage
\appendix

\linewidth\hsize {\centering
	{\Large\bfseries Provably Efficient UCB-type Algorithms For Learning PSRs:  Supplementary Materials \par}}

\section{Properties of PSRs}

In this section, we present a few important properties of PSRs, which will be intensively used in the algorithm analysis. 

First, for any model $\theta = \{\phi_h, \Mbf_h(o_h,a_h)\}$, we have the following identity
\begin{align}
    &\Mbf_h(o_h,a_h) \bar{\psi}(\tau_{h-1}) = \Pb_{\theta}(o_h|\tau_{h-1})\bar{\psi}(\tau_h).\label{eqn:Mpsi is conditional probability}
\end{align}

The following proposition is directly adapted from Lemma C.3 in \citet{liu2022optimistic}. Note that $\psi_0$ is known to the agent.
\begin{proposition}[TV-distance $\leq$ Estimation error]\label{prop: TV distance less than estimation error}
Consider two $\gamma$-well-conditioned PSRs $\theta,\hat{\theta}\in\Theta$. We have
    \begin{align*}
        &\mathtt{D}_{\TV} \left( \Pb_{\hat{\theta}}^{\pi} , \Pb_{\theta}^{\pi} \right) \leq     \sum_{h=1}^H \sum_{\tau_H} \left| \hat{\mathbf{m}}(\omega_h)^{\top} \left( \hat{\Mbf}(o_h,a_h) - \Mbf(o_h,a_h) \right) \psi(\tau_{h-1}) \right| \pi(\tau_H),  \\
        &\mathtt{D}_{\TV} \left( \Pb_{\hat{\theta}}^{\pi} , \Pb_{\theta}^{\pi} \right) \leq   \sum_{h=1}^H \sum_{\tau_H}\left| \mathbf{m}(\omega_h)^{\top} \left( \hat{\Mbf}(o_h,a_h) - \Mbf(o_h,a_h) \right) \hat{\psi}(\tau_{h-1}) \right| \pi(\tau_H) .
    \end{align*}
\end{proposition}

The next proposition characterizes well-conditioned PSRs \citep{zhong2022posterior}, which is obtained by noting that $\mbf(\mathbf{q}_h^{\ell})^{\top} $ is the $\ell$-th row of $ \Mbf_{h+1}(o_{h+1},a_{h+1})$.
\begin{proposition}\label{prop: well-condition PSR M}
    For well-conditioned  (self-consistent) PSR $\theta$, we have
    \begin{align*}
        \max_{x\in\mathbb{R}^{d_{h-1}}: \|x\|_1=1} \max_{\pi} \sum_{\omega_h}\pi(\omega_h) \left\| \Mbf_{h+1}(o_{h+1},a_{h+1})x \right\|_1 \leq \frac{Q_A}{\gamma}.
    \end{align*}
\end{proposition}

The following proposition characterizes the log-cardinality of the minimal optimistic net of the rank-$r$ PSRs. We note that this notion is closely related to the bracketing number, a typical complexity measure in the maximum likelihood estimation (MLE) anslysis \citep{geer2000empirical}. The proof follows directly from Theorem C.9 in \citet{liu2022optimistic}.
\begin{proposition}[Optimistic net for tabular PSRs]\label{prop: optimistic net of tabular PSR}
    Given any $\epsilon$, there exists a finite parameter space $\bar{\Theta}_{\epsilon}$ satisfying the following property: for any $\theta\in\Theta$, we can find a $\bar{\theta}\in\bar{\Theta}_{\epsilon}$ associated with a measure  $    \Pb_{\bar{\theta}}$ such that
    \begin{align*}
        &\forall \pi,h,~~ \Pb_{\bar{\theta}}^{\pi}(\tau_h) \geq \Pb_{\theta}^{\pi}(\tau_h),\\
        &\forall \pi, h,~~  \sum_{\tau_h}\left| \Pb_{\bar{\theta}}^{\pi}(\tau_h) -  \Pb_{\theta}^{\pi}(\tau_h)\right|\leq \epsilon.
    \end{align*}
    Moreover, $\log|\bar{\Theta}_{\epsilon}|\leq 2r^2OAH^2\log\frac{OA}{\epsilon}$.
\end{proposition}

\if{0}
\begin{align}
    &\Pb_{\theta}^{\pi}(\tau_h) = \mathbf{v}_{h}^{\top}\mathbf{B}(o_{h},a_{h})\cdots\mathbf{B}(o_1,a_1)\psi_0 \pi(\tau_h)
\end{align}

Now let $\bar{\Theta}$ be a $\epsilon_0$ cover of the the parameters with respect to the $\ell_{\infty}$-norm, since all parameters are bounded.

\begin{align}
    \big[\bar{\mathbf{B}}(o_h,a_h) \big]_{i,j} = \bigg\lceil \frac{\big[ \mathbf{B}(o_h,a_h) \big]_{i,j}}{\epsilon_0} \bigg\rceil \epsilon_0.
\end{align}

Then, we have, for any $\theta$, there exists a $\bar{\theta}$ such that

\begin{align}
    &\forall h,\tau_h,  0\leq \Pb_{ \theta}^{\pi}(\tau_h) - \Pb_{\bar{\theta}}^{\pi}(\tau_h)   \leq \epsilon_0 (OA)^{H}
\end{align}

Let $\epsilon_0 = \frac{\epsilon}{(OA)^{H}}$, we have 
\begin{align}
        &\forall \pi, h,~~  \sum_{\tau_h}\left| \Pb_{\bar{\theta}}^{\pi}(\tau_h) -  \Pb_{\theta}^{\pi}(\tau_h)\right|\leq \epsilon
    \end{align}

The covering number is 
\begin{align}
    |\bar{\Theta}| = \left(\frac{(OA)^H}{\epsilon_0}\right)^{r^2OAH} = \left(\frac{OA}{\epsilon}\right)^{2r^2OAH^2}
\end{align}

\fi

\section{General MLE Analysis}\label{sec:general MLE analysis}

In this section, we present four general propositions that characterize the performance of MLE.

We start with a proposition that states that the log-likelihood of the true model is relatively high compared to any model. 
\begin{proposition}\label{proposition: log likelihood of true model is large}
    Fix $\varepsilon<\frac{1}{KH}$.  With probability at least $1-\delta$, for any $\bar{\theta}\in\bar{\Theta}_{\varepsilon}$ and any $k\in[K]$, the following two inequalities hold:
    \[
    \begin{aligned}
        &\forall \bar{\theta}\in\bar{\Theta}_{\varepsilon},~~  \sum_h\sum_{(\tau_h,\pi)\in\Dc_h}\log \Pb_{\bar{\theta}}^{\pi}(\tau_h)  - 3\log\frac{K|\bar{\Theta}_{\varepsilon}|}{\delta}\leq \sum_{h}\sum_{(\tau_h,\pi)\in\Dc_h^k} \log \Pb_{\theta^*}^{\pi}(\tau_h), \\
        &\forall \bar{\theta}\in\bar{\Theta}_{\varepsilon},~~ \sum_{(\tau_H,\pi)\in\Dc^k}\log \Pb_{\bar{\theta}}^{\pi}(\tau_H)  - 3\log\frac{K|\bar{\Theta}_{\varepsilon}|}{\delta}\leq \sum_{(\tau_H,\pi)\in\Dc}\log \Pb_{\theta^*}^{\pi}(\tau_H).   
    \end{aligned}
    \]
\end{proposition}

\begin{proof} We start with the first inequality. Suppose the data in $\Dc_h^k$ is indexed by $t$. Then,
\begin{align*}
    \Eb&\left[ \exp\left(\sum_h\sum_{(\tau_h,\pi)\in\Dc_h^k}\log\frac{ \Pb_{\bar{\theta}}^{\pi}( \tau_h ) }{\Pb_{\theta^*}^{\pi}(\tau_h )} \right)\right]\\
    & = \Eb\left[ \prod_{t\leq k}\prod_h \frac{ \Pb_{\bar{\theta}}^{\pi^t}(\tau_h^t) }{ \Pb_{ \theta^*}^{\pi^t}(\tau_h^t) } \right] \\
    & = \Eb\left[ \prod_{t\leq k-1}\prod_h \frac{ \Pb_{\bar{\theta}}^{\pi^t}(\tau_h^t) }{ \Pb_{ \theta^*}^{\pi^t}(\tau_h^t) } \Eb\left[ \frac{ \Pb_{\bar{\theta}}^{\pi^{k}}(\tau_h^{k}) }{ \Pb_{ \theta^*}^{\pi^{k}}(\tau_h^{k}) } \right] \right] \\
    & = \Eb\left[ \prod_{t\leq k-1}\prod_h \frac{ \Pb_{\bar{\theta}}^{\pi^t}(\tau_h^t) }{ \Pb_{ \theta^*}^{\pi^t}(\tau_h^t) } \prod_h\sum_{\tau_h}   \Pb_{\bar{\theta}}^{\pi^{k}}(\tau_h )   \right] \\
    &\overset{(a)}\leq \left(1 + \varepsilon\right)^H \Eb\left[ \prod_{t\leq k-1}\prod_h \frac{ \Pb_{\bar{\theta}}^{\pi^t}(\tau_h^t) }{ \Pb_{ \theta^*}^{\pi^t}(\tau_h^t) }   \right] \\
    &\leq \left(1+\varepsilon\right)^{KH} \\
    &\overset{(b)}\leq e,
\end{align*}
where $(a)$ follows because  $\sum_{\tau_H}|\Pb_{\bar{\theta}}^{\pi}(\tau_h) - \Pb_{\theta}^{\pi}(\tau_h) | \leq \varepsilon$, and $(b)$ follows because $\varepsilon\leq \frac{1}{KH}$.

By the Chernoff bound and the union bound over $\bar{\Theta}_{\epsilon}$ and $k\in[K]$, with probability at least $1-\delta$, we have,
\begin{align*}
    \forall k\in[K], \bar{\theta}\in\bar{\Theta}_{\epsilon},~~ \sum_h\sum_{(\tau_h,\pi)\in\Dc_h^k}\log\frac{ \Pb_{\bar{\theta}}^{\pi}( \tau_h ) }{\Pb_{\theta^*}^{\pi}(\tau_h )} \leq 3\log\frac{K\left|\bar{\Theta}_{\epsilon}\right|}{\delta},
\end{align*}
which  yields the first result of this proposition.

To show the second inequality, we follow an argument similar to that for the first inequality. We have
\begin{align*}
    \Eb  \left[ \exp\left( \sum_{(\tau_H,\pi)\in\Dc^k} \log\frac{ \Pb_{ \theta }^{\pi}(\tau_H) }{\Pb_{ \theta^* }^{\pi}(\tau_H)}  \right)\right] &\leq \Eb\left[ \exp\left( \sum_{(\tau_H,\pi)\in\Dc^k} \log \frac{ \Pb_{ \bar{\theta} }^{\pi}(\tau_H) }{\Pb_{ \theta^* }^{\pi}(\tau_H)}  \right)\right]  \\
    &\overset{(a)}\leq (1+\varepsilon)^{KH}\leq e,
\end{align*}
where $(a)$ follows from the tower rule of the expectation and because $\sum_{\tau_H}\Pb_{\bar{\theta}}^{\pi}(\tau_H)\leq \varepsilon$.

Thus, with probability at least $1-\delta$, for any $k\in[K]$ and any $\bar{\theta}\in\Theta$, the following inequality holds
\begin{align*}
    \sum_{(\tau_H,\pi)\in\Dc^k} \log\frac{ \Pb_{ \theta }^{\pi}(\tau_H) }{\Pb_{ \theta^* }^{\pi}(\tau_H)} \leq 3\log\frac{K|\bar{\Theta}_{\epsilon}|}{\delta},
\end{align*}
which completes the proof.
\end{proof}

The following proposition upper bounds the total variation distance between the conditional distributions over the future trajectory conditioned on the empirical history trajectories. This proposition is crucial to ensure that the  model estimated by PSR-UCB is accurate on those sample trajectories.

\begin{proposition}\label{proposiiton: empirical distance less than log likelihood difference}
Fix $p_{\min}$ and $\varepsilon\leq \frac{p_{\min}}{KH}$.  
Let $\Theta_{\min}^k = \{\theta: \forall h, (\tau_h,\pi)\in\Dc_h^k, ~~ \Pb_{\theta}^{\pi}(\tau_h)  \geq p_{\min}\}$. Consider the following event
    \begin{align*}
        \Ec_{\omega} &= \left\{ \forall k\in[K], \forall \theta\in\Theta_{\min}^k,  ~~  \sum_h \sum_{ (\tau_h,\pi) \in\Dc_h^k } \mathtt{D}_{\TV}^2 \left( \Pb_{\theta}^{\pi}(\omega_h|\tau_h ), \Pb_{\theta^*}^{\pi}(\omega_h|\tau_h ) \right) \right. \\
        &\hspace{4cm}\quad \left. \leq  6\sum_h\sum_{ (\tau_H,\pi)\in\Dc_h^k }\log\frac{\Pb_{\theta^*}^{ \pi  }(\tau_H)}{\Pb_{ \theta}^{ \pi  }(\tau_H)} + 31\log\frac{K\left|\bar{\Theta}_{\varepsilon}\right|}{\delta}  \right\}.
    \end{align*}
Then, $\Pb\left(\Ec_{\omega} \right) \geq 1- \delta.$
\end{proposition}
\begin{proof}
We start with a general upper bound on the total variation distance between two conditional distributions. Note that for any $\theta,\theta'\in\Theta\cup\bar{\Theta}_{\epsilon}$ and fixed $(\tau_h,\pi)$, we have
\begin{align*}
    \mathtt{D}_{\TV}&\left( \Pb_{\theta}^{\pi}(\omega_h|\tau_h ), \Pb_{\theta'}^{\pi}(\omega_h|\tau_h )\right) \\
    & = \sum_{\omega_h} \bigg| \frac{\Pb_{\theta'}^{ \pi }(\omega_h,\tau_h ) \Pb_{\theta}^{\pi }(\tau_h ) - \Pb_{\theta}^{ \pi }(\omega_h,\tau_h )\Pb_{\theta'}^{\pi }(\tau_h )  }{\Pb_{\theta}^{\pi }(\tau_h )\Pb_{\theta'}^{\pi }(\tau_h )} \bigg|\\
    & = \sum_{\omega_h} \left| \frac{ \left(\Pb_{\theta'}^{ \pi }(\omega_h,\tau_h ) - \Pb_{\theta}^{ \pi }(\omega_h,\tau_h ) \right) \Pb_{\theta}^{\pi }(\tau_h ) + \Pb_{\theta}^{ \pi }(\omega_h,\tau_h ) \left(\Pb_{\theta}^{\pi }(\tau_h ) -\Pb_{\theta'}^{\pi }(\tau_h )  \right) }{\Pb_{\theta}^{\pi }(\tau_h ) \Pb_{\theta'}^{\pi }(\tau_h )} \right|\\
    &\leq \frac{|\Pb_{\theta}^{\pi }(\tau_h ) - \Pb_{\theta'}^{\pi }(\tau_h )|}{\Pb_{\theta'}^{\pi }(\tau_h ) } + \frac{1}{\Pb_{\theta'}^{\pi }(\tau_h )}\sum_{\omega_h} \left|\left(\Pb_{\theta'}^{ \pi }(\omega_h,\tau_h ) - \Pb_{\theta}^{ \pi }(\omega_h,\tau_h ) \right)\right|\\
    &\leq \frac{2}{\Pb_{\theta'}^{\pi }(\tau_h )} \mathtt{D}_{\TV}\left( \Pb_{\theta}^{ \pi }(\tau_H), \Pb_{\theta'}^{ \pi }(\tau_H) \right).
\end{align*}

By symmetry, we also have
\begin{align*}
    \mathtt{D}_{\TV}\left( \Pb_{\theta}^{\pi}(\omega_h|\tau_h ), \Pb_{\theta'}^{\pi}(\omega_h|\tau_h )\right) \leq \frac{2}{ \max\left\{\Pb_{\theta}^{\pi }(\tau_h ), \Pb_{\theta'}^{\pi }(\tau_h ) \right\} } \mathtt{D}_{\TV}\left( \Pb_{\theta}^{ \pi }(\tau_H), \Pb_{\theta'}^{ \pi }(\tau_H) \right).
\end{align*}

We replace $\theta'$ by a $\bar{\theta} \in\bar{\Theta}_{\varepsilon}$  that is $\varepsilon$-optimistic over $\theta$ (recall \Cref{def:optimistic net}), i.e. $
    \mathtt{D}_{\TV}\left( \Pb_{\theta}^{\pi} , \Pb_{\bar{\theta}}^{\pi}\right) \leq \varepsilon,$ and $\Pb_{\bar{\theta}}^{\pi}(\tau_h)\geq\Pb_{\theta}^{\pi}(\tau_h)$ holds for any $\pi$ and $\tau_h$. 
Then, due the construction of $\Theta_{\min}^k$, we have
\[\forall (\tau_h,\pi)\in\Dc_h^k,~~ \mathtt{D}_{\TV}\left( \Pb_{\theta}^{\pi}(\omega_h|\tau_h), \Pb_{\bar{\theta}}^{\pi}(\omega_h|\tau_h) \right) \leq \frac{2\epsilon}{p_{\min}}\leq \frac{2}{KH}, \]
which implies
\begin{align*}
    \sum_h\sum_{(\tau_h,\pi)\in\Dc_h^k} & \mathtt{D}_{\TV}^2\left( \Pb_{\theta}^{\pi}(\omega_h|\tau_h ), \Pb_{\theta^*}^{\pi}(\omega_h|\tau_h )\right) \\
    &\overset{(a)}\leq \sum_h\sum_{(\tau_h,\pi)\in\Dc_h^k}  2\mathtt{D}_{\TV}^2\left( \Pb_{\theta}^{\pi}(\omega_h|\tau_h ), \Pb_{\bar{\theta}}^{\pi}(\omega_h|\tau_h )\right) +  2\mathtt{D}_{\TV}^2\left( \Pb_{\bar{\theta}}^{\pi}(\omega_h|\tau_h ), \Pb_{\theta^*}^{\pi}(\omega_h|\tau_h )\right)\\
    &\leq \frac{4 }{KH} + 2 \sum_h\sum_{(\tau_h,\pi)\in\Dc_h^k}   \mathtt{D}_{\TV}^2\left( \Pb_{\bar{\theta}}^{\pi}(\omega_h|\tau_h ), \Pb_{\theta^*}^{\pi}(\omega_h|\tau_h )\right).
\end{align*}
Here $(a)$ follows because the total variation distance satisfies the triangle inequality and  $(a+b)^2\leq 2a^2 + 2b^2$.

Moreover, note that 
\begin{align*}
    \mathtt{D}_{\TV}^2& \left( \Pb_{\bar{\theta}}^{\pi}(\omega_h|\tau_h ), \Pb_{\theta^*}^{\pi}(\omega_h|\tau_h )\right)\\
    &\overset{(a)}\leq 4(2 + 2/(KH))\mathtt{D}^2_{\mathtt{H}}  \left( \Pb_{\bar{\theta}}^{\pi}(\omega_h|\tau_h ), \Pb_{\theta^*}^{\pi}(\omega_h|\tau_h ) \right) \\
    &  \leq  6\left( 1+\frac{1}{KH} - \mathop{\Eb}_{\omega_h\sim \Pb_{\theta^*}^{\pi} } \sqrt{\frac{\Pb_{\bar{\theta}}^{\pi}(\omega_h|\tau_h )}{\Pb_{\theta^*}^{\pi}(\omega_h|\tau_h )}} \right)\\
    & \overset{(b)}\leq - 6 \log \mathop{\Eb}_{\omega_h\sim \Pb_{\theta^*}^{\pi}(\cdot|\tau_h ) } \sqrt{\frac{\Pb_{\bar{\theta}}^{\pi}(\omega_h|\tau_h )}{\Pb_{\theta^*}^{\pi}(\omega_h|\tau_h )}} + \frac{6}{KH},
\end{align*}
where $(a)$ is due to \Cref{lemma:TV and hellinger} and $(b)$ follows because $1-x\leq -\log x$ for any $x>0$.

Thus, the summation of the total variation distance between conditional distributions conditioned on $(\tau_h,\pi)\in\Dc_h^k$  can be upper bounded by 
\begin{align*}
    \sum_h\sum_{(\tau_h,\pi)\in\Dc_h^k} & \mathtt{D}_{\TV}^2\left( \Pb_{\theta}^{\pi}(\omega_h|\tau_h ), \Pb_{\theta^*}^{\pi}(\omega_h|\tau_h )\right)\\
    &\leq \frac{18}{KH} - 12\sum_h\sum_{(\tau_h,\pi)\in\Dc_h^k} \log \mathop{\Eb}_{\omega_h\sim \Pb_{\theta^*}^{\pi}(\cdot|\tau_h ) } \sqrt{\frac{\Pb_{\bar{\theta}}^{\pi}(\omega_h|\tau_h )}{\Pb_{\theta^*}^{\pi}(\omega_h|\tau_h )}}.
\end{align*}

In addition, we have
\begin{align*}
    &\mathop{\Eb}_{ \substack{\forall h,(\tau_h,\pi)\in\Dc_h^k,\\ \omega_h \sim \Pb_{\theta^*}^{\pi}(\cdot|\tau_h) }} \left[\exp\left( \frac{1}{2}\sum_h\sum_{(\omega_h,\tau_h,\pi)\in\Dc_h^k} \log\frac{\Pb_{\bar{\theta}}^{\pi}(\omega_h |\tau_h )}{\Pb_{\theta^*}^{\pi}(\omega_h |\tau_h )} - \sum_h\sum_{(\tau_h,\pi)\in\Dc_h^k} \log \mathop{\Eb}_{\omega_h\sim \Pb_{\theta^*}^{\pi}(\cdot|\tau_h ) } \sqrt{\frac{\Pb_{\bar{\theta}}^{\pi}(\omega_h|\tau_h )}{\Pb_{\theta^*}^{\pi}(\omega_h|\tau_h )}} \right)\right]\\
    & \quad\quad =  \frac{ \mathop{\Eb}_{ \substack{\forall h,(\tau_h,\pi)\in\Dc_h^k,\\ \omega_h \sim \Pb_{\theta^*}^{\pi}(\cdot|\tau_h) } } \left[\prod_h\prod_{(\omega_h,\tau_h)\in\Dc_h^k} \sqrt{\frac{ \Pb_{\bar{\theta}}^{\pi}(\omega_h |\tau_h ) }{\Pb_{\theta^*}^{\pi}(\omega_h |\tau_h )} }\right] }{\prod_h\prod_{(\tau_h,\pi)\in\Dc_h^k} \mathop{\Eb}_{\omega_h\sim \Pb_{\theta^*}^{\pi}(\cdot|\tau_h ) } \left[\sqrt{\frac{\Pb_{\bar{\theta}}^{\pi}(\omega_h|\tau_h )}{\Pb_{\theta^*}^{\pi}(\omega_h|\tau_h )}} \right] }  = 1, 
\end{align*}
where the last equality is due to the conditional independence of $\omega_h \in \Dc_h^k$ given $(\tau_h,\pi)\in\Dc_h^k$.

Therefore, by the Chernoff bound, with probability $1-\delta$, we have

\begin{align*}
    -\sum_h\sum_{(\tau_h,\pi)\in\Dc_h^k} \log \mathop{\Eb}_{\omega_h\sim \Pb_{\theta^*}^{\pi}(\cdot|\tau_h ) } \sqrt{\frac{\Pb_{\bar{\theta}}^{\pi}(\omega_h|\tau_h )}{\Pb_{\theta^*}^{\pi}(\omega_h|\tau_h )}} \leq \frac{1}{2} \sum_h\sum_{(\omega_h,\tau_h,\pi)\in\Dc_h^k}\log\frac{\Pb_{\theta^*}^{\pi}(\omega_h |\tau_h )}{\Pb_{\bar{\theta}}^{\pi}(\omega_h |\tau_h )} + \log\frac{1}{\delta}.
\end{align*}

Taking the union bound over $\bar{\Theta}_{\epsilon}$, $k\in[K]$, and rescaling $\delta$, we have, with probability at least $1-\delta$, $\forall k\in[K]$, the following inequality holds:
\begin{align*}
    \sum_h &\sum_{(\tau_h,\pi)\in\Dc_h^k} \mathtt{D}_{\TV}^2\left(\Pb_{\theta}^{\pi}(\omega_h|\tau_h ), \Pb_{\theta^*}^{\pi}(\omega_h|\tau_h ) \right) \\
    &\leq \frac{18 }{KH} +  6\sum_h\sum_{(\omega_h,\tau_h,\pi)\in\Dc_h^k}\log\frac{\Pb_{\theta^*}^{\pi}(\omega_h |\tau_h )}{\Pb_{\bar{\theta}}^{\pi}(\omega_h |\tau_h )} + 12\log\frac{K\left|\bar{\Theta}_{\varepsilon}\right|}{\delta}\\
    &\leq 6\sum_h\sum_{(\omega_h, \tau_h,\pi)\in\Dc_h^k}\log\frac{\Pb_{\theta^*}^{\pi}( \omega_h ,\tau_h )}{\Pb_{\bar{\theta}}^{\pi}(\omega_h ,\tau_h )} + 6\sum_h\sum_{(\tau_h,\pi)\in\Dc_h^k}\log\frac{ \Pb_{\bar{\theta}}^{\pi}( \tau_h ) }{\Pb_{\theta^*}^{\pi}(\tau_h )} + 13\log\frac{K\left|\bar{\Theta}_{\varepsilon}\right|}{\delta}.  
\end{align*}

Note that, following from \Cref{proposition: log likelihood of true model is large}, with probability at least $1-\delta$, we have for any $ k\in[K]$,
\begin{align*}
 \sum_h\sum_{(\tau_h,\pi)\in\Dc_h^k}\log\frac{ \Pb_{\bar{\theta}}^{\pi}( \tau_h ) }{\Pb_{\theta^*}^{\pi}(\tau_h )} \leq 3\log\frac{K\left|\bar{\Theta}_{\varepsilon}\right|}{\delta}.
\end{align*}

Hence, combining with the optimistic property of $\bar{\theta}$ and rescaling $\delta$, we have that the following inequality holds with probability at least $1-\delta$:
\begin{align*}
    \sum_h\sum_{(\tau_h,\pi)\in\Dc_h^k} \mathtt{D}_{\TV}^2\left(\Pb_{\theta}^{\pi}(\omega_h|\tau_h ), \Pb_{\theta^*}^{\pi}(\omega_h|\tau_h ) \right)  \leq  6\sum_h\sum_{(\tau_H,\pi)\in\Dc_h^k}\log\frac{\Pb_{\theta^*}^{ \pi  }(\tau_H )}{\Pb_{ \theta}^{ \pi  }(\tau_H )} +  31\log\frac{K\left|\bar{\Theta}_{\varepsilon}\right|}{\delta},
\end{align*}
which yields the final result. 
\end{proof}

The following proposition is standard in the MLE analysis, and we provide the full analysis here for completeness.
\begin{proposition}\label{proposition: hellinger distance less than log likelihood distance}
Fix $\varepsilon<\frac{1}{K^2H^2}$.   Define the following event:
\begin{align*}
    \Ec_{\pi} =  \left\{ \forall \theta\in\Theta, \forall k\in[K], ~~ \sum_{\pi\in\Dc^k}  \mathtt{D}_{\mathtt{H}}^2  ( \Pb_{\theta}^{\pi}(\tau_H) , \Pb_{\theta^*}^{\pi} (\tau_H) ) \leq \frac{1}{2}\sum_{(\tau_H,\pi)\in\Dc^k} \log\frac{ \Pb_{ \theta^*}^{\pi}(\tau_H) }{\Pb_{ \theta }^{\pi}(\tau_H)} +  2\log\frac{K|\bar{\Theta}_{\varepsilon}|}{\delta} \right\}.
\end{align*}
We have $\Pb(\Ec_{\pi}) \geq 1-\delta.$
\end{proposition}
\begin{proof}
First, by the construction of $\bar{\Theta}_{\varepsilon}$, for any $\theta$, let $\bar{\theta}$ be optimistic over $\theta$, i.e., $\sum_{\tau_H}\left|\Pb_{ \theta}^{\pi}(\tau_H) -  \Pb_{\bar{\theta}}^{\pi}(\tau_H) \right|\leq\varepsilon$. We translate the distance between $\theta$ and $\theta^*$ to the distance between $\bar{\theta}$ and $\theta^*$ as follows.
\begin{align*}
    \mathtt{D}_{\mathtt{H}}^2 & ( \Pb_{\theta}^{\pi}(\tau_H) , \Pb_{\theta^*}^{\pi} (\tau_H) ) \\
    & = 1 -  \sum_{\tau_H} \sqrt{ \Pb_{\theta}^{\pi}(\tau_H) \Pb_{\theta^*}^{\pi} (\tau_H)  }\\
    & = 1 -  \sum_{\tau_H} \sqrt{ \Pb_{\bar{\theta}}^{\pi}(\tau_H) \Pb_{\theta^*}^{\pi} (\tau_H)  +  \left(\Pb_{ \theta}^{\pi}(\tau_H) -  \Pb_{\bar{\theta}}^{\pi}(\tau_H) \right)\Pb_{\theta^*}^{\pi} (\tau_H)} \\
    &\overset{(a)} \leq 1 - \sum_{\tau_H} \sqrt{ \Pb_{\bar{\theta}}^{\pi}(\tau_H) \Pb_{\theta^*}^{\pi} (\tau_H)  } + \sum_{\tau_H}\sqrt{ \left|\Pb_{ \theta}^{\pi}(\tau_H) -  \Pb_{\bar{\theta}}^{\pi}(\tau_H) \right| \Pb_{\theta^*}^{\pi} (\tau_H)} \\
    &\overset{(b)} \leq - \log \mathop{\Eb}_{\tau_H\sim\Pb_{\theta^*}^{\pi}(\cdot)} \sqrt{ \frac{ \Pb_{ \bar{\theta} }^{\pi}(\tau_H) }{ \Pb_{\theta^*}^{\pi}(\tau_H) } } + \sqrt{\sum_{\tau_H}\left|\Pb_{ \theta}^{\pi}(\tau_H) -  \Pb_{\bar{\theta}}^{\pi}(\tau_H) \right| } \\
    & \leq  - \log \mathop{\Eb}_{\tau_H\sim\Pb_{\theta^*}^{\pi}(\cdot)} \sqrt{ \frac{ \Pb_{ \bar{\theta} }^{\pi}(\tau_H) }{ \Pb_{\theta^*}^{\pi}(\tau_H) } } + \sqrt{\varepsilon},
\end{align*}
where $(a)$ follows because $\sqrt{a+b} \geq\sqrt{a} - \sqrt{|b|}$ if $a>0$ and $a+b>0$, and $(b)$ follows from the Cauchy's inequality and the fact that $1-x\leq - \log x$.

Hence, in order to upper bound $\sum_{\pi\in\Dc_k} \mathtt{D}_{\mathtt{H}}^2 ( \Pb_{\theta}^{\pi}(\tau_H) , \Pb_{\theta^*}^{\pi} (\tau_H) ) $, it suffices to upper bound 
\\
$\sum_{\pi\in\Dc_k}  - \log \mathop{\Eb}_{\tau_H\sim\Pb_{\theta^*}^{\pi}(\cdot)} \sqrt{ \frac{ \Pb_{ \bar{\theta} }^{\pi}(\tau_H) }{ \Pb_{\theta^*}^{\pi}(\tau_H) } }.$ To this end, we observe that,
\begin{align*}
 \Eb&\left[\exp\left( \frac{1}{2} \sum_{(\tau_H,\pi)\in\Dc^k} \log\frac{ \Pb_{\bar{\theta}}^{\pi}(\tau_H) }{\Pb_{\theta^*}^{\pi}(\tau_H)}  - \sum_{ \pi \in\Dc^k}\log \mathop{\Eb}_{\tau_H\sim\Pb_{\theta^*}^{\pi}(\cdot)} \sqrt{ \frac{ \Pb_{\theta}^{\pi}(\tau_H) }{ \Pb_{\theta^*}^{\pi}(\tau_H) } } \right)\right] \\
    &\quad \overset{(a)}= \frac{ \Eb\left[\prod_{(\tau_H,\pi)\in\Dc^k} \sqrt{ \frac{ \Pb_{\theta}^{\pi}(\tau_H) }{ \Pb_{\theta^*}^{\pi}(\tau_H) } } \right] }{ \Eb\left[\prod_{(\tau_H,\pi)\in\Dc^k} \sqrt{ \frac{ \Pb_{\theta}^{\pi}(\tau_H) }{ \Pb_{\theta^*}^{\pi}(\tau_H) } } \right] }  = 1,
\end{align*}
where $(a)$ follows because $(\tau_H,\pi)\in\Dc^k$ form a filtration.

Then, by the Chernoff bound, 
we have
\begin{align*}
    \Pb&\left( \frac{1}{2} \sum_{(\tau_H,\pi)\in\Dc^k} \log\frac{ \Pb_{\bar{\theta}}^{\pi}(\tau_H) }{\Pb_{\theta^*}^{\pi}(\tau_H)}  - \sum_{ \pi \in\Dc^k}\log \mathop{\Eb}_{\tau_H\sim\Pb_{\theta^*}^{\pi}(\cdot)} \sqrt{ \frac{ \Pb_{\theta}^{\pi}(\tau_H) }{ \Pb_{\theta^*}^{\pi}(\tau_H) } }  \geq \log\frac{1}{\delta} \right) \\
    & = \Pb\left( \exp\left(\frac{1}{2} \sum_{(\tau_H,\pi)\in\Dc^k} \log\frac{ \Pb_{\bar{\theta}}^{\pi}(\tau_H) }{\Pb_{\theta^*}^{\pi}(\tau_H)}  - \sum_{ \pi \in\Dc^k}\log \mathop{\Eb}_{\tau_H\sim\Pb_{\theta^*}^{\pi}(\cdot)} \sqrt{ \frac{ \Pb_{\theta}^{\pi}(\tau_H) }{ \Pb_{\theta^*}^{\pi}(\tau_H) } } \right) \geq \frac{1}{\delta} \right)\\
    &\overset{(a)}\leq \delta \Eb\left[\exp\left( \frac{1}{2} \sum_{(\tau_H,\pi)\in\Dc^k} \log\frac{ \Pb_{\bar{\theta}}^{\pi}(\tau_H) }{\Pb_{\theta^*}^{\pi}(\tau_H)}  - \sum_{ \pi \in\Dc^k}\log \mathop{\Eb}_{\tau_H\sim\Pb_{\theta^*}^{\pi}(\cdot)} \sqrt{ \frac{ \Pb_{\theta}^{\pi}(\tau_H) }{ \Pb_{\theta^*}^{\pi}(\tau_H) } } \right)\right]\\
    &\leq \delta,
\end{align*}
where $(a)$ is due to the Markov inequality.

Finally, rescaling $\delta$ to $\delta/(K|\bar{\Theta}_{\varepsilon}|)$ and taking the union bound over $\bar{\Theta}_{\epsilon}$ and $k\in[K]$, we conclude that, with probability at least $1-\delta$, $\forall \theta\in\Theta, k\in[K]$,
\begin{align*}
   ~~\sum_{\pi\in\Dc^k} & \mathtt{D}_{\mathtt{H}}^2   ( \Pb_{\theta}^{\pi}(\tau_H) , \Pb_{\theta^*}^{\pi} (\tau_H) ) \\
    &\leq KH\sqrt{\varepsilon} + \frac{1}{2} \sum_{(\tau_H,\pi)\in\Dc^k} \log\frac{ \Pb_{ \theta^*}^{\pi}(\tau_H) }{\Pb_{\bar{\theta}}^{\pi}(\tau_H)} + \log\frac{K|\bar{\Theta}_{\epsilon}|}{\delta} \\
    &\overset{(a)}\leq  \frac{1}{2}\sum_{(\tau_H,\pi)\in\Dc^k} \log\frac{ \Pb_{ \theta^*}^{\pi}(\tau_H) }{\Pb_{ \theta }^{\pi}(\tau_H)} + 2\log\frac{K|\bar{\Theta}_{\epsilon}|}{\delta}, 
\end{align*}
where $(a)$ follows because $\varepsilon\leq \frac{1}{K^2H^2}$.
\end{proof}
\if{0}
Thus, applying the Chernoff bound and taking the union bound over $\bar{\Theta}_{\epsilon}$ and $k\in[K]$, with probability at least $1-\delta$, we have \jing{can you simply start with $\Pb[\Ec_\pi]...$?}
\begin{align*}
   ~~\sum_{\pi\in\Dc^k} & \mathtt{D}_{\mathtt{H}}^2   ( \Pb_{\theta}^{\pi}(\tau_H) , \Pb_{\theta^*}^{\pi} (\tau_H) ) \\
    &\leq KH\sqrt{\varepsilon} + \frac{1}{2} \sum_{(\tau_H,\pi)\in\Dc^k} \log\frac{ \Pb_{ \theta^*}^{\pi}(\tau_H) }{\Pb_{\bar{\theta}}^{\pi}(\tau_H)} + \log\frac{K|\bar{\Theta}_{\epsilon}|}{\delta} \\
    &\overset{(a)}\leq  \frac{1}{2}\sum_{(\tau_H,\pi)\in\Dc^k} \log\frac{ \Pb_{ \theta^*}^{\pi}(\tau_H) }{\Pb_{ \theta }^{\pi}(\tau_H)} + 2\log\frac{K|\bar{\Theta}_{\epsilon}|}{\delta}, \quad \forall k\in[K],
\end{align*}
where $(a)$ follows because $\varepsilon\leq \frac{1}{K^2H^2}$.
\fi

The following proposition states that the constraint $\Theta_{\min}^k$ does not rule out the true model.
\begin{proposition}\label{proposition: propobability of empirical data is high}
Fix $p_{\min} \leq \frac{\delta}{KH(|\mathcal{O}||\mathcal{A}|)^H}$. Consider the following event:
\begin{align*}
    \Ec_{\min} = \bigg\{\forall k\in[K], \forall h,(\tau_h,\pi)\in\Dc_h^k,~~ \Pb_{\theta^*}^{\pi}(\tau_h) \geq p_{\min}\bigg\} = \bigcap_{k\in[K]} \left\{ \theta^*\in\Theta_{\min}^k \right\}.
\end{align*}
    We have $\Pb(\Ec_{\min}) \geq  1- \delta$.
\end{proposition}
\begin{proof}
For any  $k\in[K]$, $h\in\{0,1,\ldots,H-1\}$ and $(\tau_h,\pi)\in \Dc_h^k$, we have
\begin{align*}
    \Pb &\left( \Pb_{\theta^*}^{\pi }(\tau_h ) < p_{\min} \right) \\
    & = \Eb \left[ \Pb\left( \Pb_{\theta^*}^{\pi }(\tau_h ) < p_{\min} \big| \pi  \right)\right] \\
    & = \Eb\left[ \sum_{\tau_h } \Pb_{\theta^*}^{\pi }(\tau_h^t)\mathbbm{1}\left\{\Pb_{\theta^*}^{\pi^t}(\tau_h^t) < p_{\min} \right\} \right] \\
    &\leq \Eb\left[\sum_{\tau_h } p_{\min}\right] \\
    &\leq \frac{\delta}{KH} .
\end{align*}

Thus, taking the union bound over $k\in[K]$, $h\in\{0,1,\ldots,H-1\}$ and $(\tau_h,\pi)\in \Dc_h$, we conclude that 
\begin{align*}
    \Pb(\Ec_{\min}) \geq 1 - \delta.
\end{align*} 
\end{proof}

\section{Proof of \Cref{thm:online PSR} (for Online PSR-UCB)}\label{sec:online PSR}
In this section, we present the full analysis for the online algorithm PSR-UCB to show \Cref{thm:online PSR}. In particular, the proof of \Cref{thm:online PSR} consists of three main steps. {\bf Step 1.} We prove that the estimated model $\hat{\theta}^k$ is not only accurate over the exploration policies, but also accurate conditioned on emipircal samples $\tau_h\in\Dc_h^k$. {\bf Step 2.} Building up on the first step, we are able to show that $V_{\hat{\theta}^k,\hat{b}^k}^{\pi}$ is a valid upper bound on the total variation distance between $\hat{\theta}^k$ and $\theta^*$. {\bf Step 3.} Based on a newly developed inequality that translates $V_{\hat{\theta}^k,\hat{b}^k}^{\pi}$ to the ground-truth prediction feature $\bar{\psi}^*(\tau_h)$, we show that the summation of $V_{\hat{\theta}^k,\hat{b}^k}^{\pi}$ over the iteration $k$ grows sublinear in time. This finally characterizes the  optimality of the output policy, the accuracy of the output model and the sample complexity of PSR-UCB, and finishes the proof.

Before we proceed, we use $\Ec$ to denote the event $\Ec_{\omega}\cap\Ec_{\pi}\cap \Ec_{\min}$, where these three events are defined in \Cref{sec:general MLE analysis}. Due to \Cref{proposiiton: empirical distance less than log likelihood difference,proposition: hellinger distance less than log likelihood distance,proposition: propobability of empirical data is high} and union bound, we immediately have $\Pb(\Ec)\geq 1-3\delta.$

\subsection{Step 1: Estimation Guarantee} 

We show that the estimated model is accurate with the past exploration policies and dataset.
\begin{lemma}\label{lemma:online MLE guarantee}
Under event $\Ec$, the following two inequalities hold:
    \[\left\{
    \begin{aligned}
        &\sum_h\sum_{(\tau_h,\pi)\in\Dc_h^k} \mathtt{D}_{\TV}^2 \left( \Pb_{\hat{\theta}^k}^{\pi}(\omega_h|\tau_h), \Pb_{\theta^*}^{\pi}(\omega_h|\tau_h) \right) \leq 7\beta, \\
        &\sum_{\pi\in\Dc^k} \mathtt{D}_{\mathtt{H}}^2\left( \Pb_{\hat{\theta}^k}^{\pi}(\tau_H), \Pb_{\theta^*}^{\pi} (\tau_H) \right) \leq 7\beta,
    \end{aligned}
    \right.\]
    where $\beta = 31\log\frac{K|\bar{\Theta}_{\varepsilon}|}{\delta}$ and $\varepsilon\leq\frac{\delta}{K^2H^2(|\mathcal{O}||\mathcal{A}|)^H}$.
\end{lemma}

\begin{proof}
To show the first inequality, note that by the selection of $\hat{\theta}^k$, we have $\hat{\theta}^k\in\mathcal{B}^k$ (defined in \Cref{eqn: confidence set}).

Following from \Cref{proposiiton: empirical distance less than log likelihood difference}, we have
\begin{align*}
    \sum_h & \sum_{(\tau_h,\pi)\in\Dc_h^k} \mathtt{D}_{\TV}^2 \left( \Pb_{\hat{\theta}^k}^{\pi}(\omega_h|\tau_h), \Pb_{\theta^*}^{\pi}(\omega_h|\tau_h) \right) \\
    &\leq 6 \sum_h\sum_{(\tau_H,\pi)\in\Dc_h^k} \log \Pb_{\theta^*}^{\pi}(\tau_H) - 6\sum_h\sum_{(\tau_H,\pi)\in\Dc_h^k}  \log \Pb_{\bar{\theta}^k}^{\pi}(\tau_H) + 31\log\frac{K|\bar{\Theta}_{\varepsilon}|}{\delta}\\
    &\overset{(a)}\leq  6\max_{\theta'\in\Theta_{\min}^k}\sum_h\sum_{(\tau_H,\pi)\in\Dc_h^k} \log \Pb_{\theta'}^{\pi}(\tau_H) - 6\sum_h\sum_{(\tau_H,\pi)\in\Dc_h^k}  \log \Pb_{\bar{\theta}^k}^{\pi}(\tau_H) + 31\log\frac{K|\bar{\Theta}_{\varepsilon}|}{\delta}\\
    &\leq 7\beta,
\end{align*}
where $(a)$ follows from $\theta^*\in\Theta_{\min}^k$ (\Cref{proposition: propobability of empirical data is high}).

To show the second inequality,  \Cref{proposition: hellinger distance less than log likelihood distance} implies that
\begin{align*}
    \sum_{\pi\in\Dc^k}& \mathtt{D}_{\mathtt{H}}^2\left( \Pb_{\hat{\theta}^k}^{\pi}(\tau_H), \Pb_{\theta^*}^{\pi} (\tau_H) \right) \\
    &\leq   \sum_{\pi\in\Dc^k} \log\Pb_{\theta^*}^{\pi}(\tau_H) - \sum_{\pi\in\Dc^k}\log\Pb_{\hat{\theta}^k}^{\pi}(\tau_H) + 2\log\frac{K|\bar{\Theta}_{\varepsilon}|}{\delta}\\
    &\overset{(a)}\leq \max_{ \theta'\in\Theta_{\min}^k } \sum_{\pi\in\Dc^k} \log\Pb_{\theta'}^{\pi}(\tau_H) - \sum_{\pi\in\Dc^k}\log\Pb_{\hat{\theta}^k}^{\pi}(\tau_H) + 2\log\frac{K|\bar{\Theta}_{\varepsilon}|}{\delta}\\
    &\leq 7\beta,
\end{align*}
where $(a)$ follows from $\theta^*\in\Theta_{\min}^k$ (\Cref{proposition: propobability of empirical data is high}).
\end{proof}

\subsection{Step 2: UCB for Total Variation Distance} 

Following from \Cref{prop: TV distance less than estimation error}, the total variation distance between two PSRs is controlled by the estimation error. Hence, we first characterize the estimation error of $\Mbf_h^*(o_h,a_h)$ in the following lemma.
\begin{lemma}\label{lemma:Estimation error of M less than feature score}
Under event $\Ec$, for any $k\in[K]$ and any policy $\pi$, we have 
\begin{align}
        &\sum_{\tau_H}\left| \mathbf{m}^*(\omega_{h})^{\top} \left(\hat{\Mbf}_h^{k}(o_h,a_h) - \Mbf_h^*(o_h,a_h) \right) \hat{\psi}_{h-1}^{k}(\tau_{h-1}) \right| \pi(\tau_H) \nonumber\\
        &\qquad\leq \Eb_{\tau_{h-1}\sim \Pb_{\hat{\theta}^{(k)}}^{\pi}}\left[ \alpha_{h-1}^k \left\|\bar{\hat{\psi}}_{h-1}^{k}(\tau_{h-1})\right\|_{\left(\hat{U}_{h-1}^{k}\right)^{-1}} \right],\label{eqn:estimation error less than bonus}
    \end{align}
where 
\begin{align*}
    &\hat{U}_{h-1}^{k} = \lambda I + \sum_{\tau_{h-1}\in\Dc_{h-1}^k } \left[ \bar{\hat{\psi}}^{k}(\tau_{h-1} )\bar{\hat{\psi}}^{k}(\tau_{h-1} )^{\top}\right],\\
    &\alpha_{h-1}^k = \frac{4\lambda Q_A^2 d }{\gamma^4} + \frac{|\mathcal{A}|^2 Q_A^2 }{\gamma^2} \sum_{\tau_{h-1}\in\Dc_{h-1}^k} \mathtt{D}_{\TV}^2\left(    \Pb_{\hat{\theta}^k }^{\mathtt{u}_{\mathcal{Q}_{h-1}^{\exp} }}(\omega_{h-1} | \tau_{h-1}  ) , \Pb_{ \theta^* }^{\mathtt{u}_{\mathcal{Q}_{h-1}^{\exp} } } (\omega_{h-1} | \tau_{h-1}  )  \right).
\end{align*}
\end{lemma}
\begin{proof}
We index future trajectory $\omega_{h-1} = (o_{h},a_h,\ldots,o_H,a_H)$ by $i$, and history trajectory $\tau_{h-1}$ by $j$. For simplicity, we denote $\mathbf{m}^*(\omega_{h})^{\top}\left(\hat{\Mbf}_h^{(k)}(o_h,a_h) - \Mbf_h^*(o_h,a_h) \right)  $ as $w_{i}^{\top}$, and $ \bar{\hat{\psi}}_{h-1}^{(k)}(\tau_{h-1}) = \frac{ \hat{\psi}_{h-1}^{k}(\tau_{h-1})}{\hat{\phi}_{h-1}^{k\top}\hat{\psi}_{h-1}^{k}(\tau_{h-1})}$ as $x_{j}$. We also denote $\pi(\omega_{h-1}|\tau_{h-1})$ as $\pi_{i|j}$.

Then, the LHS of \Cref{eqn:estimation error less than bonus} can be written as
    \begin{align*}
        \text{LHS} &= \sum_{\tau_H} \left| \mathbf{m}^* (\omega_h)^{\top} \left(\hat{\Mbf}_h^{k}(o_h,a_h) - \Mbf_h^*(o_h,a_h) \right) \hat{\psi}_{h-1}^{k}(\tau_{h-1}) \right| \pi(\tau_H) \\
        &\overset{(a)} =  \sum_{i }\sum_{j }|w_{i}^{\top}x_{j}|\pi_{i|j}\Pb_{ \hat{\theta}^k }^{\pi}(j)\\
        & =  \sum_{j }  \sum_{i } (\pi_{i|j}\cdot\mathtt{sgn}(w_{i}^{\top}x_{j})\cdot w_{i})^{\top} x_{j}\cdot\Pb_{ \hat{\theta}^k }^{\pi}(j)\\
        & =  \sum_{j} \left( \sum_{i} \pi_{i|j} \cdot \mathtt{sgn}(w_{i}^{\top}x_{j} )\cdot w_{i} \right)^{\top} x_{j}\cdot\Pb_{ \hat{\theta}^k }^{\pi}(j)\\
        &\overset{(b)}\leq \mathop{\Eb}_{j\sim\Pb_{ \hat{\theta}^k }^{\pi}}\left[ \big\| x_j \big\|_{\left(\hat{U}_{h-1}^{k}\right)^{-1}}  \sqrt{ \left\|\sum_{i} \pi_{i|j} \cdot \mathtt{sgn}(w_{i}^{\top}x_{j} )\cdot w_{i} \right\|_{\hat{U}_{h-1}^{k}}^2 } \right],
    \end{align*}
where $(a)$ follows because $\Pb_{\theta}^{\pi}(\tau_h) = \phi_h^{\top}\psi(\tau_h)$, and $(b)$ follows from the Cauchy's inequality.

Fix $\tau_{h-1} = j_0$. We aim to upper bound the term $ I_1:=  \left\|\sum_{i} \pi_{i|j_0} \cdot \mathtt{sgn}(w_{i}^{\top}x_{j_0} )\cdot w_{i} \right\|_{\hat{U}_{h-1}^{k}}^2$. Then, we have
\begin{align*}
    &I_1  =  \underbrace{\lambda \left\| \sum_{i} \pi_{i|j_0} \cdot \mathtt{sgn}(w_{i}^{\top}x_{j_0} )\cdot w_{i} \right\|_2^2}_{I_2} + \underbrace{\sum_{j\in\Dc_{h-1}^{\tau}}   \left[ \left(\sum_{i} \pi_{i|j_0} \cdot \mathtt{sgn}(w_{i}^{\top}x_{j_0} )\cdot w_{i}\right)^{\top} x_{j} \right]^2  }_{I_3}.
\end{align*}

We first upper bound the first term $I_2$ as follows:
\begin{align*}
    \sqrt{I_2} &= \sqrt{\lambda} \max_{x\in\mathbb{R}^{d_{h-1}}: \|x\|_2=1} \left|\sum_{i} \pi_{i|j_0}\mathtt{sgn}(w_i^{\top}x_{j_0})w_i^{\top}  x\right|\\
    & \leq \sqrt{\lambda} \max_{x\in\mathbb{R}^{d_{h-1}}:\|x\|_2=1} \sum_{\omega_{h-1}} \left| \mathbf{m}^*(\omega_h)^{\top} \left( \hat{\Mbf}_h^{k}(o_h,a_h) - \Mbf_h^*(o_h,a_h) \right)   x \right| \pi(\omega_{h-1}|j_0)\\
    & \overset{(a)}\leq \sqrt{\lambda} \max_{x\in\mathbb{R}^{d_{h-1}}:\|x\|_2=1} \sum_{\omega_{h-1}}\left|\mathbf{m}^{*}(\omega_h)^{\top} \hat{\Mbf}_h^{k}(o_h,a_h)   x \right| \pi(\omega_{h-1}|j_0) \\
    & \quad + \sqrt{\lambda} \max_{x\in\mathbb{R}^{d_{h-1}}:\|x\|_2=1} \sum_{\omega_{h-1}}\left|\mathbf{m}^*(\omega_h)^{\top}  \Mbf_h^*(o_h,a_h)   x \right| \pi(\omega_{h-1}|j_0) \\
    & \overset{(b)}\leq \frac{\sqrt{\lambda}}{\gamma}\max_{x\in\mathbb{R}^{d_{h-1}}:\|x\|_2=1} \sum_{o_h,a_h}\left\| \hat{\Mbf}_h^{k}(o_h,a_h)   x \right\|_1 \pi(a_h|o_h,j_0)  + \frac{\sqrt{\lambda}}{\gamma} \max_{x\in\mathbb{R}^{d_{h-1}}:\|x\|_2=1} \| x\|_1 \\
    & \overset{(c)}\leq \frac{2Q_A\sqrt{d\lambda}}{\gamma^2},
\end{align*}
where $(a)$ follows because $|a+b|\leq |a|+|b|$, $(b)$ follows from \Cref{assmp:well-condition}, and $(c)$ follows from \Cref{prop: well-condition PSR M} and because $\max_{x\in\mathbb{R}^{d_{h-1}}:\|x\|_2=1}\| x\|_1 \leq \sqrt{d}$.

We next upper bound the second term $I_3$ as follows.
\begin{align*}
    I_3 &\leq \sum_{\tau_{h-1}\in \Dc_{h-1}^{k}}   \left( \sum_{\omega_{h-1}} \left| \mathbf{m}^* (\omega_{h})^{\top}\left(\hat{\Mbf}_h^{k}(o_h,a_h) -  \Mbf_h^*(o_h,a_h) \right)\bar{\hat{\psi}}^{k}(\tau_{h-1}) \right| \pi(\omega_{h-1}|j_0) \right)^2   \\
    & \overset{(a)}\leq   \sum_{\tau_{h-1}\in \Dc_{h-1}^{k}}   \left( \sum_{\omega_{h-1}} \left|\mathbf{m}^*(\omega_h)^{\top} \left( \hat{\Mbf}_h^{k}(o_h,a_h) \bar{\hat{\psi}}^{k}(\tau_{h-1}) - \Mbf_h^*(o_h,a_h)  \bar{\psi}^*(\tau_{h-1}) \right) \right| \pi(\omega_{h-1}|j_0) \right.   \\
    &\hspace{3cm} +    \left. \sum_{\omega_{h-1}}\left| \mathbf{m}^*(\omega_h)^{\top}  \Mbf_h^*(o_h,a_h) \left(\bar{\hat{\psi}}^{(k)}(\tau_{h-1}) -  \bar{\psi}^*(\tau_{h-1}) \right) \right| \pi(\omega_{h-1}|j_0) \right)^2     \\
    & \overset{(b)}\leq  \sum_{\tau_{h-1}\in \Dc_{h-1}^{k} }   \left( \frac{1}{\gamma} \sum_{o_h,a_h} \left\| \Pb_{\hat{\theta}^{k}}(o_h|\tau_{h-1})\bar{\hat{\psi}}_{h}^{k}(\tau_{h}) - \Pb_{\theta}(o_h|\tau_{h-1}) \bar{\psi}_{h}^*(\tau_{h}) \right\|_1   \pi(a_h|o_h,j_0) \right. \\
    & \hspace{3cm} + \left. \frac{1}{\gamma}  \sum_{\tau_{h-1}\in \Dc_{h-1}^{\tau}}      \left\| \bar{\hat{\psi}}_{h-1}^{k}(\tau_{h-1}) -  \bar{\psi}_{h-1}^*(\tau_{h-1})  \right\|_1  \right)^2   \\
    & \overset{(c)}=  \frac{1}{\gamma^2} \sum_{\tau_{h-1}\in \Dc_{h-1}^{k} }   \left(  \sum_{o_h,a_h}  \sum_{\ell=1}^{|\mathcal{Q}_h|} \left| \Pb_{\hat{\theta}^k}(\mathbf{o}_h^{\ell},o_h|\tau_{h-1},a_h,\mathbf{a}_h^{\ell}) - \Pb_{ \theta^* }(\mathbf{o}_h^{\ell},o_h|\tau_{h-1},a_h,\mathbf{a}_h^{\ell}) \right|\pi(a_h|o_h,j_0) \right.    \\
    & \hspace{3cm} +  \left.   \sum_{\ell=1}^{|\mathcal{Q}_{h-1}|}   \left| \Pb_{\hat{\theta}^k}(\mathbf{o}_{h-1}^{\ell} | \tau_{h-1}, \mathbf{a}_{h-1}^{\ell} ) - \Pb_{ \theta^* } (\mathbf{o}_{h-1}^{\ell} | \tau_{h-1}, \mathbf{a}_{h-1}^{\ell} )  \right|    \right)^2  \\
    &\leq \frac{1}{\gamma^2} \sum_{\tau_{h-1}\in \Dc_{h-1}^{k} }   \left( \ \sum_{\mathbf{a}_{h-1}\in \mathcal{Q}_h^{\exp} }   \sum_{\omega_{h-1}^o }\left|\Pb_{\hat{\theta}^k}(\omega^o_{h-1}|\tau_{h-1}, \mathbf{a}_{h-1} ) - \Pb_{\theta^*}(\omega^o_{h-1}|\tau_{h-1}, \mathbf{a}_{h-1} ) \right|    \right)^2  \\
    &\leq \frac{4|\mathcal{A}|^2Q_A^2}{\gamma^2} \sum_{\tau_{h-1}\in\Dc_{h-1}^{k}} \mathtt{D}_{\TV}^2\left(    \Pb_{\hat{\theta}^k }^{\mathtt{u}_{\mathcal{Q}_{h-1}^{\exp} }}(\omega_{h-1} | \tau_{h-1}  ) , \Pb_{ \theta^* }^{\mathtt{u}_{\mathcal{Q}_{h-1}^{\exp} } } (\omega_{h-1} | \tau_{h-1}  )  \right),
\end{align*} 
where $(a)$ follows because $|a+b|\leq |a|+|b|$, $(b)$ follows from \Cref{eqn:Mpsi is conditional probability} and \Cref{assmp:well-condition}, and $(c)$ follows from the physical meaning of the prediction feature, i.e.  $[\bar{\psi}(\tau_h)]_{\ell} = \Pb_{\theta}(\mathbf{o}_h^{\ell}|\tau_h,\mathbf{a}_h^{\ell})$.

By combining the upper bounds for $I_2$ and $I_3$, we conclude that
\begin{align*}
        I_1\leq \frac{4\lambda Q_A^2 d}{\gamma^4} + \frac{ 4|\mathcal{A}|^2Q_A^2}{\gamma^2}\sum_{\tau_{h-1}\in\Dc_{h-1}^{k}} \mathtt{D}_{\TV}^2\left(    \Pb_{\hat{\theta}^k }^{\mathtt{u}_{\mathcal{Q}_{h-1}^{\exp} }}(\omega_{h-1} | \tau_{h-1}  ) , \Pb_{ \theta^* }^{\mathtt{u}_{\mathcal{Q}_{h-1}^{\exp} } } (\omega_{h-1} | \tau_{h-1}  )  \right) = \alpha_{h-1}^k,
    \end{align*}
which completes the proof.
\end{proof}

The following lemma validates that $V_{\hat{\theta}^k,\hat{b}^k}^{\pi}$ is an upper bound on the total variation distance between the estimated model $\hat{\theta}^k$ and the true model $\theta^*$.
\begin{lemma}
    Under event $\Ec$, for any $\pi$, we have
    \begin{align}
        \mathtt{D}_{\TV}\left( \Pb_{\hat{\theta}^k}^{\pi} (\tau_H), \Pb_{\theta^*}^{\pi} (\tau_H) \right) \leq \alpha \Eb_{\tau_{H}\sim \Pb_{\hat{\theta}^{k}}^{\pi}} \left[  \sqrt{ \sum_{h=0}^{H-1}\left\|\bar{\hat{\psi}}^{k}(\tau_{h})\right\|_{ ( \hat{U}_h^k )^{-1} }^2 } \right],
    \end{align}
where 
\[ \alpha^2 =   \frac{4 \lambda H Q_A^2 d }{\gamma^4} + \frac{28|\mathcal{A}|^2Q^2_{A}\beta}{\gamma^2}. \]
\end{lemma}
\begin{proof}
We proceed the proof as follows:
\begin{align*}
    \sum_{h}(\alpha_{h-1}^k)^2 
    & \leq \frac{4\lambda HQ_A^2d }{\gamma^4} + \frac{ 4|\mathcal{A}|^2Q_A^2}{\gamma^2} \sum_h\sum_{(\tau_{h-1},\pi)\in\Dc_{h-1}^{k}} \mathtt{D}_{\TV}^2\left( \Pb_{\hat{\theta}^k }^{ \pi }(\omega_{h-1} | \tau_{h-1}  ) , \Pb_{ \theta^* }^{ \pi } (\omega_{h-1} | \tau_{h-1}  )  \right) \\
    &\overset{(a)}\leq \frac{ 4 \lambda H Q_A^2 d }{\gamma^4} + \frac{ 28\beta |\mathcal{A}|^2 Q_A^2 }{\gamma^2},
\end{align*}
where $(a)$ follows from \Cref{lemma:online MLE guarantee}.
The proof then follows directly from \Cref{prop: TV distance less than estimation error}, \Cref{lemma:Estimation error of M less than feature score} and the Cauchy's inequality.
\end{proof}

Note that the reward function $R$ is within $[0,1]$. We hence obtain the following corollary.
\begin{corollary}[UCB]\label{corollary: valid UCB}
Under event $\Ec$, for any $k\in[K]$ and any reward $R$, we have
\begin{align*}
    \left|V_{\hat{\theta}^{k} , R }^{\pi} - V_{\theta^* , R }^{\pi} \right| \leq V_{\hat{\theta}^k, \hat{b}^{k}},
\end{align*}
where $\hat{b}^k(\tau_H) = \min\left\{  \alpha \sqrt{ \sum_{h }  \left\|\bar{\hat{\psi}}_h^{k}(\tau_h)\right\|^2_{ (\hat{U}_h^{k} )^{-1}} } , 1 \right\}$.
\end{corollary}

\subsection{Step 3: Sublinear Summation}

To prove that $\sum_k V_{\hat{\theta}^k,\hat{b}^k}^{\pi^k}$ is sublinear, i.e., scales as $O(\sqrt{K})$, we first prove the following lemma that relates the estimated feature and the ground-truth feature via the total variation distance between the estimated model and the true model. 
\begin{lemma}\label{lemma:empirical bouns less than true bonus}
Under the event $\Ec$, for any $k\in[K]$ and any policy $\pi$, we have
    \begin{align*}
        \mathop{\Eb}_{\tau_H\sim\Pb_{\theta^*}^{\pi}}&\left[\sqrt{ \sum_{h=0}^{H-1} \left\|\bar{\hat{\psi}}^{k}(\tau_h) \right\|^2_{\left(\hat{U}_h^{k}\right)^{-1}} } \right] \\
        &\leq \left(1 + \frac{2|\mathcal{A}|Q_A\sqrt{7r\beta}}{\sqrt{\lambda}} \right)   \sum_{h=0}^{H-1} \mathop{\Eb}_{\tau_h\sim\Pb_{\theta^*}^{\pi}} \left[ \left\|\bar{\psi}^* (\tau_h) \right\|_{\left(U_h^{k}\right)^{-1}} \right] + \frac{2HQ_A}{\sqrt{\lambda}}  \mathtt{D}_{\TV}\left(  \Pb_{\theta^*}^{\pi}(\tau_H) , \Pb_{\hat{\theta}^k}^{\pi}(\tau_H)\right).
    \end{align*}
\end{lemma}

\begin{proof}
Recall that 
\begin{align*}
    \hat{U}_h^k = \lambda I + \sum_{\tau\in\Dc_h^k} \bar{\hat{\psi}}^{k}(\tau_h)\bar{\hat{\psi}}^{k}(\tau_h)^{\top}.
\end{align*}

We define the ground-truth counterpart of $\hat{U}_h^k$ as follows:
\begin{align*}
    U_h^k = \lambda I + \sum_{\tau\in\Dc_h^k} \bar{\psi}^*(\tau_h) \bar{ \psi}^*(\tau_h)^{\top}.
\end{align*}

Then, following from \Cref{lemma:transfer vector score}, we have
\begin{align*}
     &\sqrt{ \sum_{h=0}^{H-1} \left\| \bar{\hat{\psi}}^{k}(\tau_h) \right\|^2_{\left(\hat{U}_h^{k}\right)^{-1}} }\\
     &\quad \leq \sum_{h=0}^{H-1} \left\|\bar{\hat{\psi}}^{k}(\tau_h) \right\|_{\left(\hat{U}_h^{k}\right)^{-1}} \\
     &\quad \leq \frac{1}{\sqrt{\lambda}} \sum_{h=0}^{H-1} \left\| \bar{\hat{\psi}}^{k}(\tau_h) - \bar{\psi}^* (\tau_h) \right\|_2 + \sum_{h=0}^{H-1} \left(1 + \frac{ \sqrt{r} \sqrt{ \sum_{\tau_h\in\Dc_h^k } \left\| \bar{\hat{\psi}}^{k}(\tau_h) - \bar{\psi}^* (\tau_h) \right\|_2^2 } }{ \sqrt{\lambda} } \right) \left\|\bar{\psi}^* (\tau_h) \right\|_{\left(U_h^{k}\right)^{-1}}.
\end{align*}

Furthermore, note that
\begin{align*}
    \left\| \bar{\hat{\psi}}^{k}(\tau_h) - \bar{\psi}^* (\tau_h) \right\|_2 
    & \leq \left\| \bar{\hat{\psi}}^{k}(\tau_h) - \bar{\psi}^* (\tau_h) \right\|_1 \\
    &\overset{(a)} \leq 2|\mathcal{A}|Q_A \mathtt{D}_{\TV}\left( \Pb_{ \hat{\theta}^k }^{ \mathtt{u}_{\mathcal{Q}_h^{\exp}} }(\omega_h|\tau_h), \Pb_{   \theta^* }^{ \mathtt{u}_{\mathcal{Q}_h^{\exp}} }(\omega_h|\tau_h) \right),
\end{align*}
where $(a)$ follows from the physical meaning of the prediction feature.

Following from \Cref{lemma:online MLE guarantee}, we conclude that
\begin{align*}
    &\sqrt{ \sum_{h=0}^{H-1} \left\|\bar{\hat{\psi}}^{k}(\tau_h) \right\|^2_{\left(\hat{U}_h^{k}\right)^{-1}} }\\
    &\quad  \leq \frac{1}{\sqrt{\lambda}} \sum_{h=0}^{H-1} \left\| \bar{\hat{\psi}}^{k}(\tau_h) - \bar{\psi}^* (\tau_h) \right\|_2 + \left(1 + \frac{2 |\mathcal{A}|Q_A\sqrt{7r\beta}   }{ \sqrt{\lambda} } \right) \sum_{h=0}^{H-1}  \left\|\bar{\psi}^* (\tau_h) \right\|_{\left(U_h^{k}\right)^{-1}}.
\end{align*}

For the first term, taking expectation, we have
\begin{align*}
    \sum_{h=0}^{H-1} & \Eb_{\tau_h\sim \Pb_{\theta^*}^{\pi}} \left[  \left\| \bar{\hat{\psi}}^{k}(\tau_h) - \bar{\psi}^*(\tau_h) \right\|_1   \right] \\
    &\leq  \sum_{h=0}^{H-1} \sum_{\tau_h} \left( \left\| \bar{\hat{\psi}}^{k}(\tau_h)\left( \Pb_{\theta}^{\pi}(\tau_h) - \Pb_{\hat{\theta}^k}^{\pi}(\tau_h) \right) + \bar{\hat{\psi}}^{k}(\tau_h) \Pb_{\hat{\theta}^k}^{\pi}(\tau_h) - \bar{\psi}^*(\tau_h) \Pb_{\theta^*}^{\pi}(\tau_h) \right\|_1 \right)\\
    &\leq  \sum_{h=0}^{H-1} \sum_{\tau_h} \left( \left\| \bar{\hat{\psi}}^{k}(\tau_h)  \right\|_1 \left| \Pb_{\theta}^{\pi}(\tau_h) - \Pb_{\hat{\theta}^k}^{\pi}(\tau_h) \right|  + \left\|  \hat{\psi}^{k}(\tau_h)   - \psi^*(\tau_h)   \right\|_1 \pi(\tau_h)\right)\\
    &\overset{(a)}\leq 2\sum_{h=0}^{H-1} Q_A \mathtt{D}_{\TV}\left(  \Pb_{\theta^*}^{\pi}(\tau_h) , \Pb_{\hat{\theta}^k}^{\pi}(\tau_h)\right) \\
    & \leq 2 H Q_A \mathtt{D}_{\TV}\left(  \Pb_{\theta^*}^{\pi}(\tau_H) , \Pb_{\hat{\theta}^k}^{\pi}(\tau_H)\right),
\end{align*}
where $(a)$ follows from $\|\bar{\hat{\psi}}^k(\tau_h)\|_1 \leq |\mathcal{Q}_h^A|\leq Q_A$, and the physical meaning of $\psi(\tau_h)$.

Thus,
\begin{align*}
        \mathop{\Eb}_{\tau_H\sim\Pb_{\theta^*}^{\pi}}&\left[\sqrt{ \sum_{h=0}^{H-1} \left\|\bar{\hat{\psi}}^{k}(\tau_h) \right\|^2_{\left(\hat{U}_h^{k}\right)^{-1}} } \right] \\
        &\leq \left(1 + \frac{2|\mathcal{A}|Q_A\sqrt{7r\beta}}{\sqrt{\lambda}} \right)   \sum_{h=0}^{H-1} \mathop{\Eb}_{\tau_h\sim\Pb_{\theta^*}^{\pi}} \left[ \left\|\bar{\psi}^* (\tau_h) \right\|_{\left(U_h^{k}\right)^{-1}} \right] + \frac{2HQ_A}{\sqrt{\lambda}}  \mathtt{D}_{\TV}\left(  \Pb_{\theta^*}^{\pi}(\tau_H) , \Pb_{\hat{\theta}^k}^{\pi}(\tau_H)\right).
    \end{align*}
\end{proof}

The following lemma can be proved via the $\ell_2$ Eluder argument \citep{chen2022partially,zhong2022posterior}. Since we have a slightly different estimation oracle and guarantee, we provide the full proof here for completeness.  
\begin{lemma}\label{lemma:summation of TV is sublinear}
    Under event $\Ec$, for any $h\in\{0,\ldots,H-1\}$, we have
    \begin{align*}
        \sum_k \mathtt{D}_{\TV}\left(  \Pb_{ \theta^* }^{\pi^k}(\tau_h) , \Pb_{\hat{\theta}^k}^{\pi^k}(\tau_h)\right) \lesssim \frac{ |\mathcal{A}|Q_A\sqrt{\beta} }{\gamma}\sqrt{rHK\log(1+dQ_AK/\gamma^4)}.
    \end{align*}
Here, $a\lesssim b$ indicates that there is an absolute positive constant $c$ such that $a\leq c\cdot b$.
\end{lemma}
\begin{proof}
First, by the first inequality in \Cref{prop: TV distance less than estimation error}, we have
\begin{align*}
    \mathtt{D}_{\TV}\left(  \Pb_{ \theta^* }^{\pi^k}(\tau_h) , \Pb_{\hat{\theta}^k}^{\pi^k}(\tau_h)\right) \leq \sum_h\sum_{\tau_H} \left| \hat{\mbf}^k(\omega_h)^{\top}\left(\hat{\Mbf}_h^k(o_h,a_h) - \Mbf_h^*(o_h,a_h)\right) \psi^* (\tau_{h-1}) \right| \pi^k(\tau_H).
\end{align*}

It suffices to upper bound $\sum_{\tau_H} \left| \hat{\mbf}^k(\omega_h)^{\top}\left(\hat{\Mbf}_h^k(o_h,a_h) - \Mbf_h^*(o_h,a_h)\right) \psi^* (\tau_{h-1}) \right| \pi(\tau_H)$ for any policy $\pi$.
For simplicity, we use similar notations as in \Cref{lemma:Estimation error of M less than feature score}.
We index the future trajectory $\omega_{h-1} = (o_{h},a_h,\ldots,o_H,a_H)$ by $i$, and the history trajectory $\tau_{h-1}$ by $j$. We represent $\bar{\psi}^*(\tau_{h-1})$ as $x_j$, and $\hat{\mbf}^k(\omega_h)^{\top}\left(\hat{\Mbf}_h^k(o_h,a_h) - \Mbf_h^* (o_h,a_h)\right)$ as $w_{i}$. We also denote $\pi(\omega_{h-1}|\tau_{h-1})$ by $\pi_{i|j}$.

Let $\lambda_0$ be a constant determined later and define the matrix
\begin{align*}
    \Lambda_h^k = \lambda_0 I + \sum_{t<k} \Eb_{j\sim\Pb_{\theta^*}^{\pi^t}} \left[ x_jx_j^{\top} \right]. 
\end{align*}

Then, for any $\pi$, we have
\begin{align*}
    \sum_{\tau_H} & \left| \hat{\mbf}^k(\omega_h)^{\top}\left(\hat{\Mbf}_h^k(o_h,a_h) - \Mbf_h^*(o_h,a_h)\right) \psi^* (\tau_{h-1}) \right| \pi(\tau_H)\\
    &= \Eb_{j\sim \Pb_{\theta^*}^{\pi^k}} \left[ \sum_{i} | \pi_{i|j} w_{i }^{\top} x_j|\right] \\
    & = \Eb_{j\sim \Pb_{\theta^*}^{\pi^k}} \left[ \left(\sum_{i} \pi_{i|j} \mathtt{sgn}(w_{i }^{\top} x_j) w_{i } \right)^{\top} x_j\right] \\
    &\overset{(a)} \leq \Eb_{j\sim \Pb_{\theta^*}^{\pi^k}} \left[ \|x_j\|_{\Lambda_h^{\dagger}} \left\|\sum_{i} \pi_{i|j} \mathtt{sgn}(w_{i }^{\top} x_j) w_{i } \right\|_{\Lambda_h} \right],
\end{align*}
where $(a)$ follows from the Cauchy's inequality.

Now we fix $j=j_0$ and consider the following term.
\begin{align*}
    &\left\|\sum_{i} \pi_{i|j_0}\mathtt{sgn}(w_{i }^{\top}x_{j_0}) w_{i} \right\|^2_{\Lambda_h} \\
    & \quad =  \underbrace{\lambda_0 \left\|\sum_{i} \pi_{i|j_0}\mathtt{sgn}(w_{i }^{\top}x_{j_0}) w_{i} \right\|^2_2}_{I_1}    + \underbrace{\sum_{t<k} \Eb_{j\sim \Pb_{\theta^*}^{\pi^t}} \left[ \left( \sum_{i} \pi_{i|j_0}\mathtt{sgn}(w_{i }^{\top}x_{j_0}) w_{i }^{\top} x_j \right)^2 \right] }_{I_2}
\end{align*}
For the first term $I_1$, we have
\begin{align*}
    \sqrt{I_1} &=   \sqrt{\lambda_0} \max_{x\in\mathbb{R}^{d_{h-1}}:\|x\|_2=1} \left|\sum_{i} \pi_{i|j_0}\mathtt{sgn}(w_{i }^{\top}x_{j_0}) w_{i}^{\top} x \right| \\
    & \leq  \sqrt{\lambda_0} \max_{x\in\mathbb{R}^{d_{h-1}}:\|x\|_2=1} \sum_{\omega_{h-1}} \pi(\omega_{h-1} | j_0 )  \left|   \hat{\mbf}^k(\omega_{h-1})^{\top}    x \right| \\
    &\quad + \sqrt{\lambda_0} \max_{x\in\mathbb{R}^{d_{h-1}}:\|x\|_2=1}   \sum_{\omega_{h-1}} \pi (\omega_{h-1} | j_0 ) \left| \hat{\mbf}^k(\omega_h)^{\top}  \Mbf_h^*(o_h,a_h)  x \right|  \\
    &\overset{(a)}\leq \frac{\sqrt{d\lambda_0}}{\gamma}   + \frac{ Q_A\sqrt{d \lambda_0 } }{\gamma^2}\\
    &\leq \frac{ 2Q_A\sqrt{d \lambda_0 } }{\gamma^2},
\end{align*}
where $(a)$ follows from \Cref{assmp:well-condition}, \Cref{prop: well-condition PSR M}, and the fact that $\max_{x\in\mathbb{R}^{d_{h-1}}:\|x\|_2=1} \|x\|_1\leq \sqrt{d_{h-1}}\leq \sqrt{d}$.

For the second term $I_2$, we have
\begin{align*}
    I_2  &\leq \sum_{t<k}\mathop{\Eb}_{\tau_{h-1}\sim\Pb_{\theta^*}^{\pi^t}} \left[ \left( \sum_{\omega_{h-1}} \pi(\omega_{h-1} | j_0) \left| \hat{\mbf}^k(\omega_h)^{\top} \left( \hat{\Mbf}_h^k(o_h,a_h) - \Mbf_h^*(o_h,a_h) \right) \bar{\psi}^*(\tau_{h-1}) \right| \right)^2 \right] \\
    & \leq \sum_{t<k}\mathop{\Eb}_{\tau_{h-1}\sim\Pb_{\theta^*}^{\pi^t}} \left[ \left( \sum_{\omega_{h-1}} \pi(\omega_{h-1} | j_0) \left| \hat{\mbf}^k(\omega_{h-1})^{\top} \left( \bar{\psi}^*(\tau_{h-1}) - \bar{\hat{\psi}}^k (\tau_{h-1}) \right)  \right| \right. \right.\\
    &\hspace{2cm} + \left. \left.  \sum_{\omega_{h-1}} \pi(\omega_{h-1}|j_0)\left| \hat{\mbf}^k(\omega_h)^{\top} \left( \hat{\Mbf}_h^k(o_h,a_h)\bar{\hat{\psi}}^k (\tau_{h-1}) - \Mbf_h^*(o_h,a_h) \bar{\psi}^*(\tau_{h-1}) \right)  \right| \right)^2 \right] \\
    &\overset{(a)}\leq \sum_{t<k} \mathop{\Eb}_{\tau_{h-1}\sim\Pb_{\theta^*}^{\pi^t}} \left[ \left( \frac{1}{\gamma} \left\|  \bar{\psi}^* (\tau_{h-1}) - \bar{\hat{\psi}}^k (\tau_{h-1})    \right\|_1 \right. \right.\\
    &\hspace{2cm} + \left. \left. \frac{1}{\gamma} \sum_{o_h,a_h} \pi^k(a_h|o_h,j_0)\left\|    \hat{\Mbf}^k(o_h,a_h)\bar{\hat{\psi}}^k (\tau_{h-1}) - \Mbf_h^*(o_h,a_h) \bar{\psi}^*(\tau_{h-1})    \right\|_1 \right)^2 \right] \\ 
    & = \frac{1}{\gamma^2} \sum_{t<k} \mathop{\Eb}_{\tau_{h-1}\sim\Pb_{\theta^*}^{\pi^t}} \left[ \left( \sum_{\ell=1}^{|\mathcal{Q}_{h-1}|}\left| \Pb_{\hat{\theta}^k}(\mathbf{o}_{h-1}^{\ell} | \tau_{h-1}, \mathbf{a}_{h-1}^{\ell} )  - \Pb_{\theta^*}(\mathbf{o}_{h-1}^{\ell} | \tau_{h-1}, \mathbf{a}_{h-1}^{\ell} )\right| \right. \right.\\
    &\hspace{2cm} + \left.\left.\sum_{o_h,a_h} \sum_{\ell=1}^{|\mathcal{Q}_h|} \pi^k(a_h|o_h,j_0)\left| \Pb_{\hat{\theta}^k}(\mathbf{o}_h^{\ell}, o_h | \tau_{h-1} ,a_h, \mathbf{a}_h^{\ell} )  - \Pb_{\theta^*} (\mathbf{o}_h^{\ell}, o_h | \tau_{h-1} ,a_h, \mathbf{a}_h^{\ell} ) \right|  \right)^2\right] \\
    & \leq \frac{|\mathcal{Q}_{h-1}^{\exp}|^2 }{\gamma^2} \sum_{t<k} \mathop{\Eb}_{\tau_{h-1}\sim\Pb_{\theta^*}^{\pi^t}} \left[ \mathtt{D}^2_{\TV} \left( \Pb_{\hat{\theta}^k }^{ \mathtt{u}_{\mathcal{Q}_{h-1}^{\exp}} } (\omega_{h-1} | \tau_{h-1} ) , \Pb_{\theta^* }^{ \mathtt{u}_{\mathcal{Q}_{h-1}^{\exp}} } (\omega_{h-1} | \tau_{h-1} )\right) \right] \\
    & \overset{(b)}\leq \frac{4|\mathcal{A}|^2Q_A^2}{\gamma^2} \sum_{t<k}     \mathtt{D}^2_{\mathtt{H}} \left( \Pb_{\hat{\theta}^k }^{ \nu_{h}(\pi^t, \mathtt{u}_{\mathcal{Q}_{h-1}^{\exp}}) } (  \tau_{H} ) , \Pb_{\theta^* }^{ \nu_h(\pi^t, \mathtt{u}_{\mathcal{Q}_{h-1}^{\exp}}) } ( \tau_{H} ) \right),
\end{align*}
where $(a)$ follows from \Cref{assmp:well-condition}, and $(b)$ follows because $|\mathcal{Q}_{h-1}^{\exp}|=|\mathcal{A}||\mathcal{Q}_h^A| + |\mathcal{Q}_{h-1}^A| \leq  2|\mathcal{A}|Q_A$.

Thus, we have
\begin{align*}
    \sum_{\tau_H}& \left| \hat{\mbf}^k(\omega_h)^{\top}\left(\hat{\Mbf}_h^k(o_h,a_h) - \Mbf_h^*(o_h,a_h)\right)\psi^*(\tau_{h-1}) \right| \pi(\tau_H)\\
    &\leq \mathop{\Eb}_{\tau_{h-1}\sim\Pb_{\theta^*}^{\pi} } \left[ \tilde{\alpha}_{h-1}^k \left\| \bar{\psi}^* (\tau_{h-1}) \right\|_{\Lambda_{h-1}^{\dagger} } \right],
\end{align*}
where 
\begin{align*}
    (\tilde{\alpha}_{h-1}^k)^2 =  \frac{4\lambda_0Q_A^2d}{\gamma^4} + \frac{4|\mathcal{A}|^2Q_A^2}{\gamma^2} \sum_{t<k}     \mathtt{D}^2_{\mathtt{H}} \left( \Pb_{\hat{\theta}^k }^{ \nu_h(\pi^t, \mathtt{u}_{\mathcal{Q}_{h-1}^{\exp}}) } (  \tau_{H} ) , \Pb_{\theta^* }^{ \nu_h(\pi^t, \mathtt{u}_{\mathcal{Q}_{h-1}^{\exp}}) } ( \tau_{H} ) \right) .
\end{align*}

By choosing $\lambda_0 = \frac{\gamma^4}{4Q_A^2d}$, and recalling \Cref{lemma:online MLE guarantee}, we further have
\begin{align*}
    \sum_{h}( \tilde{\alpha}_{h-1}^k)^2\leq  \frac{28 |\mathcal{A}|^2Q_A^2 \beta }{\gamma^2} \triangleq \tilde{\alpha}.
\end{align*}

Hence, with the Cauchy's inequality,  we have
\begin{align*}
    \mathtt{D}_{\TV}\left( \Pb_{\hat{\theta}^k}^{\pi^k}(\tau_H), \Pb_{\theta}^{\pi^k} (\tau_H) \right) \leq \min\left\{\tilde{\alpha} \sqrt{ \sum_{h=1}^H \mathop{\Eb}_{\tau_{h-1}\sim\Pb_{\theta^*}^{\pi^k}}\left[ \left\| \bar{\psi}^*(\tau_{h-1}) \right\|^2_{\Lambda_{h-1}^{-1}} \right] } , 2\right\}.
\end{align*}
 
Taking the summation and applying the elliptical potential lemma (i.e.,~\Cref{lemma:elliptical potential lemma}), we have
\begin{align*}
    \sum_{k=1}^K& \mathtt{D}_{\TV}\left( \Pb_{\hat{\theta}^k}^{\pi^k}(\tau_H), \Pb_{\theta^*}^{\pi^k} (\tau_H) \right) \\
    &\leq  \sqrt{K} \sqrt{ \sum_{k=1}^K \sum_{h=1}^H \min\left\{ \tilde{\alpha} \mathop{\Eb}_{\tau_{h-1}\sim\Pb_{\theta^*}^{\pi^k}}\left[ \left\| \bar{\psi}^*(\tau_{h-1}) \right\|^2_{\Lambda_{h-1}^{\dagger}} \right]  , 4 \right\}} \\
    & \lesssim \frac{|\mathcal{A}|Q_A}{\gamma} \sqrt{rHK\beta\log(1+dQ_AK/\gamma)},
\end{align*}
which yields the final result.
\end{proof}

The next lemma shows that the summation $\sum_{k=1}^K V_{\hat{\theta}^{k}, \hat{b}^k}^{\pi^k}$ grows sublinearly in $K$. 
\begin{lemma}\label{lemma:sublinear summation}
    Under the event $\Ec$, with probability at least $1-\delta$, we have
    \begin{align*}
        \sum_{k=1}^K V_{\hat{\theta}^{k}, \hat{b}^k}^{\pi^k} \lesssim \left( \sqrt{r} + \frac{Q_A\sqrt{H}}{\gamma}   \right) \frac{ |\mathcal{A}|Q_A^2H\sqrt{drH\beta K \beta_0} }{\gamma^2}, 
    \end{align*}
    where $\beta_0 = \max\{\log(1+K/\lambda), \log(1+dQ_AK/\gamma) \}$, and $\lambda = \frac{\gamma |\mathcal{A}|^2Q_A \beta  \max\{ \sqrt{r}, Q_A\sqrt{ H}/\gamma \}   }{   \sqrt{d H}  }  $. 
\end{lemma}

\begin{proof}
First, since $\hat{b}^k(\tau_H)\in[0,1]$, we have
\begin{align*}
    V_{\hat{\theta}^{k}, \hat{b}^k}^{\pi^k} &\leq V_{ \theta^* , \hat{b}^k}^{\pi^k} + \frac{1}{2}\mathtt{D}_{\TV}\left( \Pb_{\hat{\theta}^{k}}^{\pi^k}, \Pb_{\theta^*}^{\pi^k}\right).
\end{align*}

Then, following from \Cref{lemma:empirical bouns less than true bonus}, we have
\begin{align*}
    \sum_{k=1}^K&  V_{ \theta^* , \hat{b}^k}^{\pi^k} \\
    &  \leq \sum_{k=1}^K\min\left\{\alpha\left(1 + \frac{2|\mathcal{A}|Q_A\sqrt{7r\beta}}{\sqrt{\lambda}} \right)\sum_{h=0}^{H-1} \mathop{\Eb}_{\tau_H\sim \Pb_{\theta^*}^{\pi^k}} \left[   \left\| \bar{\psi}^*(\tau_h) \right\|_{(U_h^{k})^{-1}}  \right] + \sum_{k=1}^K\frac{\alpha HQ_A}{\sqrt{\lambda}} \mathtt{D}_{\TV}\left(  \Pb_{\theta^*}^{\pi^k}(\tau_H) , \Pb_{\hat{\theta}^k}^{\pi^k}(\tau_H)\right) ,1 \right\} \\
    &\leq \underbrace{ \sum_{k=1}^K \min\left\{ \alpha\left(1 + \frac{2|\mathcal{A}|Q_A\sqrt{7r\beta}}{\sqrt{\lambda}} \right) \sum_{h=0}^{H-1} \mathop{\Eb}_{\tau_h\sim \Pb_{\theta^*}^{\pi^k}} \left[    \left\| \bar{\psi}^*(\tau_h) \right\|_{(U_h^{k})^{-1}} \right] , 1 \right\} }_{I_1} + \sum_{k=1}^K\frac{\alpha HQ_A}{\sqrt{\lambda}} \mathtt{D}_{\TV}\left(  \Pb_{\theta^*}^{\pi^k}(\tau_H) , \Pb_{\hat{\theta}^k}^{\pi^k}(\tau_H)\right).
\end{align*}
 
We next analyze the term $I_1$. Recall that $U_h^k = \lambda I + \sum_{\tau_h\in\Dc_h^k} \bar{\psi}^*(\tau_h)\bar{\psi}^*(\tau_h)^{\top}$. Further note that 
$$\left\{ \Eb_{\tau_h\sim\Pb_{\theta^*}^{\pi^k}}\left[\|\bar{\psi}^*(\tau_h)\|_{(U_h^k)^{-1}} \right] - \|\bar{\psi}^*(\tau_h^{k+1,h+1})\|_{(U_h^k)^{-1}} \right\}_{k=1}^K$$ 
forms a martingale. Applying the Azuma-Hoeffding inequality (see \Cref{lemma:hoeffding}), we have, with probability at least $1-\delta$, 
\begin{align*}
I_1 & \leq \sqrt{2K\log(2/\delta)} + \sum_{k=1}^K\min\left\{ \alpha\left(1 + \frac{2|\mathcal{A}|Q_A\sqrt{7r\beta}}{\sqrt{\lambda}} \right) \sum_{h=0}^{H-1}      \left\| \bar{\psi}^*(\tau_h^{k+1,h+1}) \right\|_{(U_h^{k})^{-1}}   , 1 \right\}\\
&\overset{(a)}\lesssim \sqrt{2K\log(2/\delta)} + \alpha\left(1 + \frac{2|\mathcal{A}|Q_A\sqrt{7r\beta}}{\sqrt{\lambda}} \right) H \sqrt{rK\log(1+K/\lambda)}\\
&\lesssim \alpha\left(1 + \frac{|\mathcal{A}|Q_A\sqrt{r\beta}}{\sqrt{\lambda}} \right) H \sqrt{rK\log(1+K/\lambda)},
\end{align*}
where $(a)$ follows from \Cref{lemma:elliptical potential lemma}.

Let $\beta:=\max\{\log(1+K/\lambda), \log(1+dQ_AK/\gamma) \}$ for simplicity. By \Cref{lemma:summation of TV is sublinear}, we have,
\begin{align*}
    \sum_k & V_{\hat{\theta}^k,\hat{b}^k }^{\pi^k} \\
    & \lesssim \alpha\left(1 + \frac{|\mathcal{A}|Q_A\sqrt{r\beta}}{\sqrt{\lambda}} \right) H\sqrt{rK\beta_0}  + \frac{\alpha H  }{\sqrt{\lambda} } \frac{|\mathcal{A}|Q_A^2\sqrt{\beta} }{\gamma}\sqrt{rHK\beta_0} \\
    &\leq \alpha \left(1 + \frac{|\mathcal{A}|Q_A\sqrt{r\beta}}{\sqrt{\lambda}} + \frac{|\mathcal{A}|Q_A^2\sqrt{\beta H}}{\gamma\sqrt{\lambda}}\right) H\sqrt{rK\beta_0} \\
    &\lesssim \left( \frac{|\mathcal{A}|Q_A\sqrt{\beta}}{\gamma} + \frac{Q_A\sqrt{dH}}{\gamma^2} \sqrt{\lambda}  \right)\left( 1 +  \frac{|\mathcal{A}|Q_A\sqrt{\beta} \max\{ \sqrt{r}, Q_A\sqrt{ H}/\gamma \} }{ \sqrt{\lambda}}\right) H\sqrt{rK\beta_0}\\
    & = \left( 1 + \frac{ \sqrt{dH}}{ |\mathcal{A}|\sqrt{\beta}\gamma} \sqrt{\lambda}  \right)\left(1 +  \frac{ |\mathcal{A}| Q_A\sqrt{\beta} \max\{ \sqrt{r}, Q_A\sqrt{ H}/\gamma \} }{ \sqrt{\lambda}}\right) \frac{|\mathcal{A}| Q_A H \sqrt{r\beta K\beta_0}}{\gamma}   \\
    &\overset{(a)}\lesssim \left( 1 + \frac{ Q_A \sqrt{dH} \max\{ \sqrt{r}, Q_A\sqrt{ H}/\gamma \} }{\gamma} \right)  \frac{|\mathcal{A}| Q_A H\sqrt{r\beta K \beta_0}}{\gamma} \\
    &\leq \left( \sqrt{r} + \frac{Q_A\sqrt{H}}{\gamma}   \right) \frac{ |\mathcal{A}|Q_A^2H\sqrt{drH\beta K \beta_0} }{\gamma^2},
\end{align*}
where $(a)$ follows by choosing 
\begin{align*} 
    \lambda = \frac{\gamma |\mathcal{A}|^2Q_A \beta  \max\{ \sqrt{r}, Q_A\sqrt{ H}/\gamma \}   }{   \sqrt{d H}  }.
\end{align*}
This completes the proof.
\end{proof}

\subsection{Proof of \Cref{thm:online PSR}}
\begin{theorem}[Restatement of \Cref{thm:online PSR}]
 Suppose \Cref{assmp:well-condition} holds. Let $p_{\min} = O(\frac{\delta}{KH|\mathcal{O}|^H|\mathcal{A}|^H })$, $\beta = O(\log|\bar{\Theta}_{\varepsilon}|)$, where $\varepsilon=O(\frac{p_{\min}}{KH})$, $\lambda = \frac{\gamma |\mathcal{A}|^2Q_A \beta  \max\{ \sqrt{r}, Q_A\sqrt{ H}/\gamma \}   }{  \sqrt{d H} }$, and $\alpha = O\left( \frac{Q_A\sqrt{ H d}}{\gamma^2} \sqrt{\lambda } + \frac{|\mathcal{A}|Q_A\sqrt{\beta}}{\gamma}\right)$. Then, with probability at least $1-\delta$, PSR-UCB~ outputs a model $\theta^{\epsilon}$ and a policy $\bar{\pi}$ that satisfy
    \begin{align*}
        V_{\theta^*,R}^{\pi^*} - V_{\theta^*,R}^{\bar{\pi}} \leq \epsilon, \text{ and }~~ \forall \pi,~~\mathtt{D}_{\TV} \left( \Pb_{ \theta^{\epsilon} }^{\pi}(\tau_H), \Pb_{\theta^*}^{\pi}(\tau_H)\right) \leq \epsilon.
    \end{align*}
    In addition, PSR-UCB terminates with a sample complexity of 
    \[\tilde{O}\left( \left(  r + \frac{Q_A^2 H}{\gamma^2}\right) \frac{r d H^3 |\mathcal{A}|^2 Q_A^4 \beta}{\gamma^4\epsilon^2} \right).\]
\end{theorem}
\begin{proof}
The proof is under event $\Ec$, which occurs with probability at least $1-3\delta$.

Following from \Cref{corollary: valid UCB}, if PSR-UCB terminates, we have
\begin{align*}
    \forall \pi,~~ \mathtt{D}_{\TV}\left(\Pb_{ \theta^{\epsilon} }^{\pi}(\tau_H), \Pb_{\theta^*}^{\pi}(\tau_H) \right) = 2\max_{R} \left|V_{\theta^\epsilon,R}^{\pi} - V_{\theta^*,R}^{\pi}\right| \leq 2 V_{\theta^\epsilon,\hat{b}^{\epsilon}}^{\pi} \leq \epsilon,
\end{align*}
where the last inequality follows from the termination condition of PSR-UCB.

In addition, 
\begin{align*}
     V_{\theta^*,R}^{\pi^*} &- V_{\theta^*,R}^{\bar{\pi}}\\
     & = V_{\theta^*,R}^{\pi^*} -  V_{\theta^\epsilon,R}^{\pi^*} + V_{\theta^\epsilon,R}^{\pi^*} - V_{\theta^\epsilon,R}^{\bar{\pi}} + V_{\theta^\epsilon,R}^{\bar{\pi}} - V_{\theta^*,R}^{\bar{\pi}} \\
     &\overset{(a)}\leq 2\max_{\pi} V_{\theta^\epsilon,\hat{b}^{\epsilon}}^{\pi}\leq \epsilon,
\end{align*}
where $(a)$ is due to the design of $\bar{\pi}$ and \Cref{corollary: valid UCB}.

Finally, recall that \Cref{lemma:sublinear summation} states that
\begin{align*} 
    \sum_{k=1}^K V_{\hat{\theta}^k,\hat{b}^k} 
    & \lesssim  \left( \sqrt{r} + \frac{Q_A\sqrt{H}}{\gamma}   \right) \frac{ |\mathcal{A}|Q_A^2H\sqrt{drH\beta K \beta_0} }{\gamma^2}.
\end{align*}

By the pigeon-hole principle and the termination condition of PSR-UCB, if 
\[ K = \tilde{O}\left(  \left(  r + \frac{Q_A^2 H}{\gamma^2}\right) \frac{r d H^2 |\mathcal{A}|^2 Q_A^4 \beta}{\gamma^4\epsilon^2}  \right),\]
PSR-UCB must terminate within $K$ episodes, implying that the sample complexity of PSR-UCB is at most

 \[ \tilde{O}\left(  \left(  r + \frac{Q_A^2 H}{\gamma^2}\right) \frac{r d H^3 |\mathcal{A}|^2 Q_A^4 \beta}{\gamma^4\epsilon^2}  \right).\]
\end{proof}

\subsection{Proof of \Cref{coro:POMDP} (For POMDPs)}\label{sec:POMDP}

When PSR is specialized to an $m$-step decodable POMDP, we follow the setup described in Section D in \citet{liu2022optimistic}. In POMDPs, there exists a state space $\mathcal{S}$ such that the system dynamics proceeds as follows. When $h=1$, the system is at the state $s_1$ and the agent observes $o_1$ that is determined by $s_1$. At any step $h\geq 1$, if the system is at a state $s_h\in\mathcal{S}$ and the agent takes an action $a_h\in\mathcal{A}$, the system transits to a state $s_{h+1}$ sampled from a {\em transition distribution} $\mathbf{T}_{h,a_h}(\cdot|s_h)$, and then the agent observes $o_{h+1}$ sampled from an {\em emission distribution} $\mathbf{O}_h(\cdot|s_{h+1})$. Notably, all states are unobservable to the agent.

To define $m$-step decodability, we introduce the $m$-step emission-action matrix $\{\mathbf{G}_h\}$ where 
\begin{align*}
    [\mathbf{G}_h]_{(\mathbf{a},\mathbf{o}), s} = \Pb(o_h,\ldots, o_{\min\{h+m-1, H\} }= \mathbf{o} | s_h = s, a_h\ldots,a_{\min\{h+m-2,H\} } = \mathbf{a} ).
\end{align*}

Therefore, for $h\leq H-h$, we have $\mathbf{G}_h\in\mathbb{R}^{(|\mathcal{A}|^{m-1}|\mathcal{O}|^m)\times |\mathcal{S}|}$. To make a unified argument, we define $m$-step decodability as follows.

\begin{definition}[$m$-step decodable POMDPs] A POMDP with $m$-step emmision-action matrix $\{\mathbf{G}_h\}$ is said to be $m$-step decodable if 
    there exists an $\alpha>0$ such that $\min_{h\leq H} \sigma_{|\mathcal{S}|}(\mathbf{G}_h) \geq \alpha$. Here $\sigma_n(\mathbf{A})$ is the $n$-th singular value of matrix $\mathbf{A}$.
\end{definition}

We remark that the key difference between our definition and the definition introduced in \citet{liu2022partially} is that we allow $h>H-m$. This indicates that after step $h>H-m$, the process is $(H-h)$-step decodable, which is reasonable in most real-world examples, since the closer we approach to the end, the clearer our understanding of the current state becomes. Moreover, the new definition admits more convenient notations with almost not impact on the logic of the proof. 

In this section, a PSR $\theta = \{\phi_h, \Mbf_h(o_h,a_h)\}$ is reparametrized by $\{\mathbf{T}_{h,a}, \mathbf{O}_h\}_{h\in[H],a\in\mathcal{A}}$. More importantly, we have the following equations and properties.

\begin{align}
\left\{
\begin{aligned}
    &\mathbf{Y}_h \in \arg\min_{ \mathbf{Y}_h: \mathbf{Y}_h\mathbf{G}_h = 0 } \|\mathbf{G}_h^{\dagger} + \mathbf{Y}_h\|_1\\
    &\Mbf_h(o_h,a_h) = \mathbf{G}_{h+1} \mathbf{T}_{h,a_h} \text{diag}(\mathbf{O}_h(o_h|\cdot)) (\mathbf{G}_h^{\dagger} + \mathbf{Y}_h) \\
    &\phi_H^{\top} M_H(o_H,a_H) = \mathbf{e}_{(o_H,a_H)} \in \mathbb{R}^{\mathcal{O}\times\mathcal{A}}\\
    &\mathcal{Q}_h = (\mathcal{O}\times \mathcal{A})^{\min\{m-1,H-h\} }\\
    &\mathcal{Q}_h^A = \mathcal{A}^{ \min\{m-1,H-h\} }\\
    & d\leq |\mathcal{O}|^m|\mathcal{A}|^m\\
    &r = |\mathcal{S}|\\
    &Q_A = |\mathcal{A}|^{m-1}\\
    &\gamma = \frac{|\mathcal{S}| + |\mathcal{A}|^{m-1}}{\alpha} = \frac{r + Q_A}{\alpha}.
\end{aligned}
\right.\label{eqn:POMDP notation}
\end{align}

\begin{corollary}[Restatement of \Cref{coro:POMDP}]
    When PSR is specialized to $m$-step decodable POMDP, for each $k$, there exists a set of matrices $\{\hat{\mathbf{G}}_h^k\}_{h=1}^H \subset \mathbb{R}^{d\times r}$ such that if we replace $\bar{\hat{\psi}}^k(\tau_h)$ by $(\hat{\mathbf{G}}_h^k)^{\dagger}\bar{\hat{\psi}}^k(\tau_h)$  in \Cref{eqn: bonus}, then PSR-UCB terminates with a sample complexity of poly$(r,1/\gamma, Q_A, H, |\mathcal{A}|,\beta)/\epsilon^2$.
\end{corollary}

\begin{proof} We use the notations defined in this section and particularly those in \Cref{eqn:POMDP notation}.

The proof differentiates from the proof of \Cref{thm:online PSR} in two lemmas.

{\bf First, in \Cref{lemma:Estimation error of M less than feature score},} the meaning of $x_j$ changes from $\bar{\hat{\psi}}^{k}(\tau_{h-1})$ to $(\hat{\mathbf{G}}_h^k)^{\dagger} \bar{\hat{\psi}}^{k}(\tau_{h-1})$, and $w_i^{\top}$ changes from $\mathbf{m}^*(\omega_h)(\hat{\Mbf}_h^k(o_h,a_h) - \Mbf_h^*(o_h,a_h))$ to $\mathbf{m}^*(\omega_h)(\hat{\Mbf}_h^k(o_h,a_h) - \Mbf_h^*(o_h,a_h)) \mathbf{G}_h$.

Then, the proof of the upper bound of the term $I_2$ will become $Q_A\sqrt{r\lambda}/\gamma^2$ (without dependency on $d$), simply because the dimension of bonus term $\hat{b}^k$ is now $r = |\mathcal{S}|$ instead of $d$.

{\bf Second, in \Cref{lemma:empirical bouns less than true bonus},} we need to show that the new bonus term $\|(\hat{\mathbf{G}}_{h+1}^k)^{\dagger}\bar{\hat{\psi}}(\tau_h)\|_{(\hat{U}_h^k)^{-1}}$ can still be bounded in a similar way, namely, $\|(\hat{\mathbf{G}}_{h+1}^k)^{\dagger}\bar{\hat{\psi}}(\tau_h)\|_{(\hat{U}_h^k)^{-1}}\leq \tilde{O} \left(\|(\mathbf{G}_{h+1}^*)^{\dagger}\bar{\psi}^*(\tau_h)\|_{(U_h^k)^{-1}} \right)$. 

We note that the matrices $\hat{U}_h^{k}$ and $U_h^k$ now have new definitions:
\begin{align*}
    &\hat{U}_h^{k} = \lambda I + (\hat{\mathbf{G}}_{h+1}^k)^{\dagger}\sum_{\tau_h\in\mathcal{D}_h^k} \left[\bar{\hat{\psi}}^{k}(\tau_h) \bar{\hat{\psi}}^k(\tau_h)^{\top} \right ] (\hat{\mathbf{G}}_{h+1}^k)^{\dagger\top}, \\
    & U_h^{k} = \lambda I + (\mathbf{G}_{h+1}^*)^{\dagger}\sum_{\tau_h\in\mathcal{D}_h^k} \left[\bar{\psi}^{k}(\tau_h) \bar{\psi}^*(\tau_h)^{\top} \right ] (\mathbf{G}_{h+1}^*)^{\dagger\top}.
\end{align*}


To show such an upper bound, we next provide the key steps.

First, we apply \Cref{lemma:transfer vector score} and obtain
\begin{align*}
    &  \left\| (\hat{\mathbf{G}}_{h+1}^k)^{\dagger}\bar{\hat{\psi}}^{k}(\tau_h) \right\|_{\left(\hat{U}_h^{k}\right)^{-1}}  \\
        &\quad \leq \frac{1}{\sqrt{\lambda}} \left\| (\hat{\mathbf{G}}_{h+1}^k)^{\dagger}\bar{\hat{\psi}}^{k}(\tau_h) - (\mathbf{G}_{h-1}^*)^{\dagger}\bar{\psi}^* (\tau_h) \right\|_2 \\
        &\quad \quad + \left(1 + \frac{ \sqrt{r} \sqrt{ \sum_{\tau_h\in\mathcal{D}_h^k } \left\| (\hat{\mathbf{G}}_{h+1}^k)^{\dagger}\bar{\hat{\psi}}^{k}(\tau_h) - (\mathbf{G}_{h+1}^*)^{\dagger}\bar{\psi}^* (\tau_h) \right\|_2^2 } }{ \sqrt{\lambda} } \right) \left\| (\mathbf{G}_{h+1}^*)^{\dagger}\bar{\psi}^* (\tau_h) \right\|_{\left(U_h^{k}\right)^{-1}}.
    \end{align*}

For any $\tau_h$, we can derive 
\begin{align*}
    &\left\| (\hat{\mathbf{G}}_{h+1}^k)^{\dagger} \bar{\hat{\psi}}^{k}(\tau_h) - (\mathbf{G}_{h+1}^*)^{\dagger}\bar{\psi}^* (\tau_h) \right\|_2\\
    & = \left\| \left( (\hat{\mathbf{G}}_{h+1}^k)^{\dagger} + \hat{\mathbf{Y}}_{h+1}^k \right) \bar{\hat{\psi}}^{k}(\tau_h) - (\mathbf{G}_{h+1}^*)^{\dagger}\bar{\psi}^* (\tau_h) \right\|_2\\
    &\leq \left\| \left( (\hat{\mathbf{G}}_{h+1}^k)^{\dagger} + \hat{\mathbf{Y}}_{h+1}^k \right) (\bar{\hat{\psi}}^{k}(\tau_h) - \bar{\psi}^* (\tau_h) )\right\|_2 + \left\|  \left( (\hat{\mathbf{G}}_{h+1}^k)^{\dagger} + \hat{\mathbf{Y}}_{h+1}^k  - (\mathbf{G}_{h+1}^{* })^{\dagger} \right) \bar{\psi}^* (\tau_h) \right\|_2\\
    &\leq \left\| (\hat{\mathbf{G}}_{h+1}^k)^{\dagger} + \hat{\mathbf{Y}}_{h+1}^k \right\|_1 \left\|\bar{\hat{\psi}}^{k}(\tau_h) - \bar{\psi}^* (\tau_h) \right\|_1 + \left\|  \left((\hat{\mathbf{G}}_{h+1}^k)^{\dagger} + \hat{\mathbf{Y}}_{h+1}^k \right) \left( I -  \hat{\mathbf{G}}_{h+1}^k  (\mathbf{G}_{h+1}^{* })^{\dagger} \right) \bar{\psi}^* (\tau_h) \right\|_2 \\
    &\leq \frac{1}{\gamma}\left\|\bar{\hat{\psi}}^{k}(\tau_h) - \bar{\psi}^* (\tau_h) \right\|_1 + \left\|  \left((\hat{\mathbf{G}}_{h+1}^k)^{\dagger} + \hat{\mathbf{Y}}_{h+1}^k \right) \left( \mathbf{G}_{h+1}^* -  \hat{\mathbf{G}}_{h+1}^k  \right) (\mathbf{G}_{h+1}^{* })^{\dagger} \bar{\psi}^* (\tau_h) \right\|_2,
\end{align*}

where the last inequality is due to the $m$-step decodability.

The first term above has been bounded in \Cref{lemma:empirical bouns less than true bonus}. Then, we bound the second term by the Cauchy's inequality as follows:
\begin{align*}
    &\left\|  \left((\hat{\mathbf{G}}_{h+1}^k)^{\dagger} + \hat{\mathbf{Y}}_{h+1}^k \right) \left( \mathbf{G}_{h+1}^* -  \hat{\mathbf{G}}_{h+1}^k  \right) (\mathbf{G}_{h+1}^{* })^{\dagger} \bar{\psi}^* (\tau_h) \right\|_2\\
    &\quad \leq \left\| (\mathbf{G}_{h+1}^{* })^{\dagger} \bar{\psi}^* (\tau_h) \right\|_{\left( \bar{U}_h^{k}\right)^{-1}}  \\
    &\quad \quad \times \sqrt{ rQ_A\lambda/\gamma +  \sum_{t<k}\mathbb{E}_{\tau_h\sim \pi^t}  \left\|  \left((\hat{\mathbf{G}}_{h+1}^k)^{\dagger} + \hat{\mathbf{Y}}_{h+1}^k \right) \left( \mathbf{G}_{h+1}^* -  \hat{\mathbf{G}}_{h+1}^k  \right) (\mathbf{G}_{h+1}^{* })^{\dagger} \bar{\psi}^* (\tau_h) \right\|_2^2 }\\
    &\quad \leq 5\left\| (\mathbf{G}_{h+1}^{* })^{\dagger} \bar{\psi}^* (\tau_h) \right\|_{\left(U_h^{k}\right)^{-1}} \sqrt{ rQ_A\lambda/\gamma +  |\mathcal{A}|Q_A^2\beta/\gamma^2 },
\end{align*}

where $\bar{U}_h^k$ denotes the expected matrix of $U_h^k$, and has the following form
\begin{align*}
    \bar{U}_h^k = \Eb[U_h^k] = \lambda I + (\mathbf{G}_{h+1}^*)^{\dagger} \sum_{t<k}\mathbb{E}_{\tau_h\sim\pi^t}\left[ \bar{\psi}^{k}(\tau_h) \bar{\psi}^*(\tau_h)^{\top}\right] (\mathbf{G}_{h+1}^*)^{\dagger\top}.
\end{align*}

In particular, the last inequality above follows from two facts. First, we use Lemma 39 in \citet{zanette2021cautiously} to transform the expected matrix to the empirical matrix $U_h^k$. Second, 
we use Lemma \Cref{lemma:TV and hellinger} to upper-bound $\mathbb{E}_{\tau_h\sim \pi^t}  \left\|  \left((\hat{\mathbf{G}}_{h+1}^k)^{\dagger} + \hat{\mathbf{Y}}_{h+1}^k \right) \left( \mathbf{G}_{h+1}^* -  \hat{\mathbf{G}}_{h+1}^k  \right) (\mathbf{G}_{h+1}^{* })^{\dagger} \bar{\psi}^* (\tau_h) \right\|_2^2$ by the Hellinger squared distance between the estimated model $\hat{\theta}^k$ and  the true model $\theta^*$. Then,  due to the estimation guarantee (\Cref{lemma:online MLE guarantee}), we have
\begin{align*}
\sum_{t<k} &\mathbb{E}_{\tau_h\sim \Pb_{\theta^*}^{\pi^t} }  \left\|  \left((\hat{\mathbf{G}}_{h+1}^k)^{\dagger} + \hat{\mathbf{Y}}_{h+1}^k \right) \left( \mathbf{G}_{h+1}^* -  \hat{\mathbf{G}}_{h+1}^k  \right) (\mathbf{G}_{h+1}^{* })^{\dagger} \bar{\psi}^* (\tau_h) \right\|_2^2 \\
& \leq \frac{1}{\gamma^2} \sum_{t<k} \mathbb{E}_{\tau_h\sim \Pb_{\theta^*}^{\pi^t} }  \left\|  \left( \mathbf{G}_{h+1}^* -  \hat{\mathbf{G}}_{h+1}^k  \right) (\mathbf{G}_{h+1}^{* })^{\dagger} \bar{\psi}^* (\tau_h) \right\|_1 ^2 \\
& = \frac{1}{\gamma^2} \sum_{t<k} \Eb_{\tau_h\sim \Pb_{\theta^*}^{\pi^t} } \left( \sum_{\mathbf{o}\in\mathcal{Q}_{h+1} ,\mathbf{a}\in\mathcal{Q}_{h+1}^A } \sum_{s\in\mathcal{S} } \left( \Pb_{\hat{\theta}^k}(\mathbf{o}| \mathbf{a}, s_{h+1} =s )  - \Pb_{\theta^* }(\mathbf{o}| \mathbf{a}, s_{h+1} =s ) \right) \Pb_{\theta^*}(s|\tau_h) \right)^2 \\
&\leq \frac{Q_A^2}{\gamma^2} \sum_{t<k} \Eb_{\tau_h\sim \Pb_{\theta^*}^{\pi^t} } \left( \Eb_{s\sim \Pb_{\theta^*}(\cdot|\tau_h) } \mathtt{D}_{\TV}\left( \Pb_{\hat{\theta}^k}^{\mathtt{u}_{\mathcal{Q}_{h+1}^A}}(\cdot | s_{h+1} =s )  , \Pb_{\theta^* }^{\mathtt{u}_{\mathcal{Q}_{h+1}^A}}(\cdot| s_{h+1} =s ) \right)  \right)^2 \\
&\leq \frac{Q_A^2}{\gamma^2} \sum_{t<k} \Eb_{\tau_h\sim \Pb_{\theta^*}^{\pi^t} }  \Eb_{s\sim \Pb_{\theta^*}(\cdot|\tau_h) } \mathtt{D}_{\TV}^2\left( \Pb_{\hat{\theta}^k}^{\mathtt{u}_{\mathcal{Q}_{h+1}^A}}(\cdot | s_{h+1} =s )  , \Pb_{\theta^* }^{\mathtt{u}_{\mathcal{Q}_{h+1}^A}}(\cdot| s_{h+1} =s )  \right)\\
&\lesssim \frac{|\mathcal{A}|Q_A^2}{\gamma^2} \sum_{\pi\in\mathcal{D^k} } \mathtt{D}_{\mathtt{H}}^2 \left( 
 \Pb_{\hat{\theta}^k}^{\pi}  , \Pb_{\theta^* }^{\pi}  \right) = \tilde{O}(\beta).
\end{align*}

Combining above key steps, we obtain that
\begin{align*}
    &\mathbb{E}_{\tau_h\sim\pi}  \left\| (\hat{\mathbf{G}}_{h+1}^k)^{\dagger}\bar{\hat{\psi}}^{k}(\tau_h) \right\|_{\left(\hat{U}_h^{k}\right)^{-1}}  \\
        &\leq \text{poly}(r, |\mathcal{A}|,Q_A, 1/\gamma)\left( \mathbb{E}_{\tau_h\sim\pi} \|(\mathbf{G}_{h+1}^*)^{\dagger}\bar{\psi}^*(\tau_h)\|_{(U_h^k)^{-1}} + \mathtt{D}_{\TV}\left(  \mathbb{P}_{\theta^*}^{\pi}(\tau_H) , \mathbb{P}_{\hat{\theta}^k}^{\pi}(\tau_H)\right) \right).
\end{align*}    
    
The rest of the proof follow the same argument as that for \Cref{thm:online PSR}.
    
\end{proof}

\section{Proof of \Cref{thm:offline PSR} (for Offline PSR-LCB)}\label{sec:offline PSR}

In this section, we present the full analysis for the offline algorithm PSR-LCB to show \Cref{thm:offline PSR}. In particular, the proof of \Cref{thm:offline PSR} consists of three main steps. {\bf Step 1:} We provide the offline estimation guarantee for the estimated model $\hat{\theta}$.  {\bf Step 2:} Building up on the first step, we are able to show that $V_{\hat{\theta},\hat{b}}^{\pi}$ is a valid upper bound of the total variation distance between $\hat{\theta}$ and $\theta^*$. Hence, $V_{\hat{\theta}, R}^{\pi} - V_{\hat{\theta},\hat{b}}^{\pi}$ is a valid LCB for the true value $V_{\theta^*,R}^{\pi}$. {\bf Step 3:} We translate $V_{\hat{\theta},\hat{b}}^{\pi}$ to the ground-truth prediction feature $\bar{\psi}^*(\tau_h)$. Finally, we show that the  $V_{\hat{\theta},\hat{b}}^{\pi}$ scales in the order of $1/\sqrt{K}$, which characterizes the performance of PSR-LCB and completes the proof.

We first introduce the following definitions for good events.

{\bf Good Events. } Recall that $\Theta_{\min} = \left\{\theta: \forall h,  \tau_h \in \Dc_h,~~ \Pb_{\theta}^{\pi^b}(\tau_h) \geq    p_{\min} \right\}$, where $p_{\min}\leq \frac{\delta}{KH(|\mathcal{O}||\mathcal{A}|)^H}$. Let $\varepsilon\leq\frac{p_{\min}}{KH}$. Analogous to the proof for online learning, we introduce three events defined as follows. 
\begin{align*}
    \Ec_{\omega}^{o} &= \left\{   \forall \theta\in\Theta_{\min},  ~~  \sum_h \sum_{ \tau_h  \in \Dc_h  } \mathtt{D}_{\TV}^2 \left( \Pb_{\theta}^{\pi^b}(\omega_h|\tau_h ), \Pb_{\theta^*}^{\pi^b}(\omega_h|\tau_h ) \right) \right. \\
        &\hspace{4cm}\quad \left. \leq  6\sum_h\sum_{ \tau_H \in\Dc_h }\log\frac{\Pb_{\theta^*}^{ \pi^b }(\tau_H)}{\Pb_{ \theta}^{ \pi^b }(\tau_H)} + 31\log\frac{3K\left|\bar{\Theta}_{\varepsilon}\right|}{\delta}  \right\}, \\
    \Ec_{\pi}^{o} &=  \left\{ \forall \theta\in\Theta, K,~~  \mathtt{D}_{\mathtt{H}}^2  ( \Pb_{\theta}^{\pi^b}(\tau_H) , \Pb_{\theta^*}^{\pi} (\tau_H) ) \leq \sum_{ \tau_H \in\Dc } \log\frac{ \Pb_{ \theta^*}^{\pi^b}(\tau_H) }{\Pb_{ \theta }^{\pi^b}(\tau_H)} +  2\log\frac{3K|\bar{\Theta}_{\varepsilon}|}{\delta} \right\}, \\
    \Ec_{\min}^{o} &= \bigg\{  \forall h, \tau_h \in\Dc_h ,~~ \Pb_{\theta^*}^{\pi^b}(\tau_h) \geq p_{\min}\bigg\}.
\end{align*}

Following the proof steps similar to those in \Cref{sec:general MLE analysis}, we can conclude that $\Ec^o:= \Ec_{\omega}^{o}\cap \Ec_{\pi}^{o} \cap \Ec_{\min}^{o}$ occurs with probability at least $1-\delta$.

\subsection{Step 1: Estimation Guarantee}
\begin{lemma}[MLE guarantee]\label{lemma:offline MLE guarantee}
Under event $\Ec^o$, the estimated model $\hat{\theta}$ by PSR-LCB
 satisfies 
 \begin{align*}
        &\sum_{h}\sum_{\tau_h\in\Dc_h} \mathtt{D}_{\TV}^2\left( \Pb_{ \hat{\theta} }^{\pi^b} (\omega_h|\tau_h), \Pb_{ \theta^* }^{\pi^b} (\omega_h|\tau_h) \right) \leq 7\hat{\beta}, \\
        & \mathtt{D}_{\mathtt{H}}^2\left( \Pb_{ \hat{\theta} }^{\pi^b} (\tau_H), \Pb_{ \theta^* }^{\pi^b} (\tau_H) \right) \leq 7\hat{\beta}/K,
    \end{align*}
    where $\hat{\beta} = 31\log\frac{3K|\bar{\Theta}_{\varepsilon}|}{\delta}$.
\end{lemma}

\begin{proof}
    The proof follows the steps similar to those in \Cref{lemma:online MLE guarantee}, except that we replace the exploration policies $\nu_h(\pi^k,\mathtt{u}_{\mathcal{Q}_{h-1}^{\exp}})$ by the behavior policy. 
\end{proof}

\subsection{Step 2: UCB for Total Variation Distance and LCB for Value Function}

The following lemma provides an explicit upper bound on the total variation distance between the estimated model and the true model.
\begin{lemma}\label{lemma: offline TV distance less than true bonus}
Under event $\Ec^o$, for any policy $\pi$, we have
\begin{align}
    \mathtt{D}_{\TV}&\left(  \Pb_{\theta^*}^{\pi}(\tau_H) , \Pb_{\hat{\theta}}^{\pi}(\tau_H)\right) \lesssim \sqrt{\frac{\hat{\beta}}{H\iota^2\gamma^2}} \sum_{h=1}^H \mathop{\Eb}_{\tau_{h-1}\sim\Pb_{\theta^*}^{\pi}} \left[ \     \left\|\bar{\psi}^* (\tau_{h-1})\right\|_{  \Lambda_{h-1}^{-1}}  \right],
\end{align}
where $\Lambda_{h-1} = \lambda_0 I + \frac{K}{H} \Eb_{\tau_{h-1}\sim\Pb_{\theta^*}^{\pi^b}} \left[ \bar{\psi}^*(\tau_{h-1}) \bar{\psi}^*(\tau_{h-1})^{\top} \right] $, and $\lambda_0 = \frac{\gamma^4}{4Q_A^2d}$. 
\end{lemma}
\begin{proof}
Similarly to the analysis in that for \Cref{lemma:offline Estimation error of M less than feature score}, we index $\omega_{h-1} = (o_{h},a_h,\ldots,o_H,a_H)$ by $i$, and $\tau_{h-1}$ by $j$. In addition, we denote $\hat{\mathbf{m}}(\omega_{h})^{\top}\left(\hat{\Mbf}_h (o_h,a_h) - \Mbf_h^*(o_h,a_h) \right)  $ by $w_{i}^{\top}$, denote $ \bar{ \psi}^* (\tau_{h-1})$ by $x_{j}$, and denote $\pi(\omega_{h-1}|\tau_{h-1})$ by $\pi_{i|j}$. Then, we have
    \begin{align*}
        \sum_{\tau_H}& \left| \hat{\mathbf{m}} (\omega_h)^{\top} \left(\hat{\Mbf}_h (o_h,a_h) - \Mbf_h^*(o_h,a_h) \right) \psi^* (\tau_{h-1}) \right| \pi(\tau_H) \\
        & =  \sum_{i }\sum_{j }|w_{i}^{\top}x_{j}|\pi_{i|j}\Pb_{ \theta^* }^{\pi}(j)\\
        & =  \sum_{j }  \sum_{i } (\pi_{i|j}\cdot\mathtt{sgn}(w_{i}^{\top}x_{j})\cdot w_{i})^{\top} x_{j}\cdot\Pb_{\theta^*}^{\pi}(j)\\
        & =  \sum_{j} \left( \sum_{i} \pi_{i|j} \cdot \mathtt{sgn}(w_{i}^{\top}x_{j} )\cdot w_{i} \right)^{\top} x_{j}\cdot\Pb_{ \theta^* }^{\pi}(j)\\
        &\leq \mathop{\Eb}_{j\sim\Pb_{\theta^*}^{\pi}}\left[ \big\| x_j \big\|_{\Lambda_{h-1}^{-1}}  \sqrt{ \left\|\sum_{i} \pi_{i|j} \cdot \mathtt{sgn}(w_{i}^{\top}x_{j} )\cdot w_{i} \right\|_{ \Lambda_{h-1} }^2 } \right],
    \end{align*}
where $\Lambda_{h-1} = \lambda_0 I + \frac{K}{H} \Eb_{\tau_{h-1}\sim\Pb_{\theta^*}^{\pi^b}} \left[ \bar{\psi}^*(\tau_{h-1}) \bar{\psi}^*(\tau_{h-1})^{\top} \right]$ and $\lambda_0$ will be determined later. We fix $\tau_{h-1} = j_0$ and aim to analyze the coefficient of $\|x_{j_0}\|_{\Lambda_{h-1}^{-1}}$. We have
\begin{align*}
    &\left\| \sum_{i} \pi_{i|j_0} \cdot \mathtt{sgn}(w_{i}^{\top}x_{j_0} )\cdot w_{i} \right\|_{\Lambda_{h-1} }^2\\
    &\quad =  \underbrace{\lambda_0 \left\| \sum_{i} \pi_{i|j_0} \cdot \mathtt{sgn}(w_{i}^{\top}x_{j_0} )\cdot w_{i} \right\|_2^2}_{I_1} + \underbrace{\frac{K}{H}  \mathop{\Eb}_{j\sim\Pb_{\theta^*}^{\pi^b}} \left[ \left(\sum_{i} \pi_{i|j_0} \cdot \mathtt{sgn}(w_{i}^{\top}x_{j_0} )\cdot w_{i}\right)^{\top} x_{j} \right]^2  }_{I_2}.
\end{align*}
For the first term $I_1$, we have
\begin{align*}
    \sqrt{I_1}  & = \sqrt{\lambda_0} \max_{x\in\mathbb{R}^{d_{h-1}}: \|x\|_2=1} \left|\sum_{i} \pi_{i|j_0}\mathtt{sgn}(w_i^{\top}x_{j_0})w_i^{\top}  x\right|\\
    &\quad \leq \sqrt{\lambda_0}\max_{x\in\mathbb{R}^{d_{h-1}}:\|x\|_2=1} \sum_{\omega_{h-1}} \left| \mathbf{m}^*(\omega_h)^{\top} \left( \hat{\Mbf}_h (o_h,a_h) - \Mbf_h^*(o_h,a_h) \right)   x \right| \pi(\omega_{h-1}|j_0)\\
    &\quad \leq \sqrt{\lambda_0}\max_{x\in\mathbb{R}^{d_{h-1}}:\|x\|_2=1} \sum_{\omega_{h-1}}\left| \hat{\mathbf{m}} (\omega_h)^{\top} \hat{\Mbf}_h (o_h,a_h)   x \right| \pi(\omega_{h-1}|j_0) \\
    &\quad\quad + \sqrt{\lambda_0}\max_{x\in\mathbb{R}^{d_{h-1}}:\|x\|_2=1} \sum_{\omega_{h-1}}\left| \hat{\mathbf{m}}(\omega_h)^{\top}  \Mbf_h^*(o_h,a_h)   x \right| \pi(\omega_{h-1}|j_0) \\
    & \overset{(a)}\leq  \frac{\sqrt{d\lambda_0}}{\gamma}  + \frac{\sqrt{\lambda_0}}{\gamma}\max_{x\in\mathbb{R}^{d_{h-1}}:\|x\|_2=1} \sum_{o_h,a_h}\left\| \hat{\Mbf}_h (o_h,a_h)   x \right\|_1 \pi(a_h|o_h,j_0)  \\
    &\overset{(b)}\leq \frac{2Q_A\sqrt{d\lambda_0}}{\gamma^2},
\end{align*}
where $(a)$ follows from \Cref{assmp:well-condition}, and $(b)$ follows from \Cref{prop: well-condition PSR M}.

For the second term $I_2$,  we have
\begin{align*}
    I_2  &\leq \frac{K}{H}  \mathop{\Eb}_{\tau_{h-1}\sim\Pb_{\theta^*}^{\pi^b}} \left[\left( \sum_{\omega_{h-1}} \left| \hat{\mathbf{m}} (\omega_{h})^{\top}\left(\hat{\Mbf}_h (o_h,a_h) -  \Mbf_h^*(o_h,a_h) \right)\bar{\psi}^* (\tau_{h-1}) \right| \pi(\omega_{h-1}|j_0) \right)^2  \right] \\
    & \leq  \frac{K}{H} \mathop{\Eb}_{\tau_{h-1}\sim\Pb_{\theta^*}^{\pi^b}}  \left[ \left( \sum_{\omega_{h-1}} \left| \hat{\mathbf{m}}(\omega_h)^{\top} \left( \hat{\Mbf}_h (o_h,a_h) \bar{ \hat{\psi}} (\tau_{h-1}) - \Mbf_h^*(o_h,a_h)  \bar{\psi}^*(\tau_{h-1}) \right) \right| \pi(\omega_{h-1}|j_0) \right. \right.  \\
    &\hspace{3cm} +   \left. \left. \sum_{\omega_{h-1}}\left| \hat{\mathbf{m}} (\omega_h)^{\top}  \hat{\Mbf}_h(o_h,a_h) \left(\bar{\hat{\psi}}  (\tau_{h-1}) -  \bar{\psi}^*(\tau_{h-1}) \right) \right| \pi(\omega_{h-1}|j_0) \right)^2 \right]   \\
    & \overset{(a)}\leq \frac{K}{H} \mathop{\Eb}_{\tau_{h-1}\sim\Pb_{\theta^*}^{\pi^b} } \left[  \left( \frac{1}{\gamma} \sum_{o_h,a_h} \left\| \Pb_{ \hat{\theta} } (o_h|\tau_{h-1})\bar{\hat{\psi}}_{h} (\tau_{h}) - \Pb_{\theta}(o_h|\tau_{h-1}) \bar{\psi}_{h}^*(\tau_{h}) \right\|_1   \pi(a_h|o_h,j_0) \right. \right.\\
    & \hspace{3cm} +  \left. \left. \frac{1}{\gamma}       \left\| \bar{\hat{\psi}} (\tau_{h-1}) -  \bar{\psi}^*(\tau_{h-1})  \right\|_1  \right)^2  \right] \\
    & =  \frac{K}{H\gamma^2} \mathop{\Eb}_{\tau_{h-1}\sim\Pb_{\theta^*}^{\pi^b}}  \left[\left(  \sum_{o_h,a_h}  \sum_{\ell=1}^{|\mathcal{Q}_h|} \left| \Pb_{ \hat{\theta} }(\mathbf{o}_h^{\ell},o_h|\tau_{h-1},a_h,\mathbf{a}_h^{\ell}) - \Pb_{ \theta^* }(\mathbf{o}_h^{\ell},o_h|\tau_{h-1},a_h,\mathbf{a}_h^{\ell}) \right|\pi(a_h|o_h,j_0) \right.    \right.\\
    & \hspace{3cm} +  \left. \left.   \sum_{\ell=1}^{|\mathcal{Q}_{h-1}|}   \left| \Pb_{ \hat{\theta} }(\mathbf{o}_{h-1}^{\ell} | \tau_{h-1}, \mathbf{a}_{h-1}^{\ell} ) - \Pb_{ \theta^* } (\mathbf{o}_{h-1}^{\ell} | \tau_{h-1}, \mathbf{a}_{h-1}^{\ell} )  \right|    \right)^2 \right] \\
    &\leq \frac{K}{H\gamma^2} \mathop{\Eb}_{\tau_{h-1}\sim\Pb_{\theta^*}^{\pi^b}}  \left[  \left( \ \sum_{\mathbf{a}_{h-1}\in \mathcal{Q}_h^{\exp} }   \sum_{\omega_{h-1}^o }\left|\Pb_{ \hat{\theta} }(\omega^o_{h-1}|\tau_{h-1}, \mathbf{a}_{h-1} ) - \Pb_{\theta^*}(\omega^o_{h-1}|\tau_{h-1}, \mathbf{a}_{h-1} ) \right|    \right)^2  \right] \\
    &\overset{(b)}\leq \frac{ K }{ H \iota^2 \gamma^2} \mathop{\Eb}_{\tau_{h-1}\sim\Pb_{\theta^*}^{\pi^b}} \left[\mathtt{D}_{\TV}^2\left(    \Pb_{\hat{\theta} }^{ \pi^b }(\omega_{h-1} | \tau_{h-1}  ) , \Pb_{ \theta^* }^{ \pi^b } (\omega_{h-1} | \tau_{h-1}  )  \right) \right]\\
    &\overset{(c)}\leq \frac{64K}{H\iota^2\gamma^2}\mathtt{D}_{\mathtt{H}}^2 \left(\Pb_{\hat{\theta}}^{\pi^b}(\tau_H), \Pb_{\theta^*}^{\pi^b}(\tau_H) \right)\\
    &\overset{(d)}\lesssim \frac{ \hat{\beta}}{H\iota^2\gamma^2},
\end{align*} 
where $(a)$ follows from \Cref{assmp:well-condition} and \Cref{eqn:Mpsi is conditional probability}, $(b)$ follows because $\pi^b(\mathbf{a}_{h-1})\geq\iota$ for all $\mathbf{a}_{h-1}\in\mathcal{Q}_{h-1}^{\exp}$, $(c)$ follows from \Cref{lemma:TV and hellinger}, and $(d)$ follows from \Cref{lemma:offline MLE guarantee}.

Thus, by choosing $\lambda_0 = \frac{\gamma^4}{4Q_A^2d}$, we have
\begin{align*}
        \sum_{\tau_H}& \left| \hat{\mathbf{m}}(\omega_h)^{\top} \left(\hat{\Mbf}_h (o_h,a_h) - \Mbf_h^*(o_h,a_h) \right) \psi^* (\tau_{h-1}) \right| \pi(\tau_H) \\
        &\lesssim \sqrt{\frac{\hat{\beta}}{H\iota^2\gamma^2}} \mathop{\Eb}_{\tau_{h-1}\sim \Pb_{ \theta^* }^{\pi}}\left[  \left\|\bar{\psi}^* (\tau_{h-1})\right\|_{  \Lambda_{h-1}^{-1}} \right].
\end{align*}
Following from \Cref{prop: TV distance less than estimation error}, we conclude that 
\begin{align*}
    \mathtt{D}_{\TV}&\left(  \Pb_{\theta^*}^{\pi}(\tau_H) , \Pb_{\hat{\theta}}^{\pi}(\tau_H)\right) \lesssim \min \left\{\sqrt{\frac{\hat{\beta}}{H\iota^2\gamma^2}} \sum_{h=1}^H \mathop{\Eb}_{\tau_{h-1}\sim\Pb_{\theta^*}^{\pi}} \left[ \     \left\|\bar{\psi}^* (\tau_{h-1})\right\|_{  \Lambda_{h-1}^{-1}}  \right] ,2 \right\}.
\end{align*}
\end{proof}

Following from \Cref{prop: TV distance less than estimation error}, we provide the following lemma that upper bounds the estimation error by the bonus function.
\begin{lemma}\label{lemma:offline Estimation error of M less than feature score}
Under event $\Ec^o$, for any policy $\pi$, we have
\begin{align*}
        \sum_{\tau_H}& \left| \mathbf{m}^*(\omega_{h})^{\top} \left(\hat{\Mbf}_h (o_h,a_h) - \Mbf_h^*(o_h,a_h) \right) \hat{\psi} (\tau_{h-1}) \right| \pi(\tau_H) \leq \Eb_{\tau_{h-1}\sim \Pb_{ \hat{\theta} }^{\pi}}\left[ \hat{\alpha}_{h-1} \left\|\bar{\hat{\psi}} (\tau_{h-1})\right\|_{\left(\hat{U}_{h-1}\right)^{-1}} \right],
\end{align*}
where 
\begin{align*}
    &\hat{U}_{h-1}  = \hat{\lambda} I + \sum_{\tau_{h-1}\in\Dc_{h-1} }   \bar{\hat{\psi}}  (\tau_{h-1} )\bar{\hat{\psi}} (\tau_{h-1} )^{\top}, \\
    & \hat{\alpha}_{h-1} = \frac{ \hat{\lambda}  Q_A^2d}{\gamma^4} + \frac{4 }{ \iota^2 \gamma^2} \sum_{ \tau_{h-1}\in\Dc_{h-1} } \mathtt{D}_{\TV}^2\left(    \Pb_{ \hat{\theta} }^{ \pi^b }(\omega_{h-1} | \tau_{h-1}  ) , \Pb_{ \theta^* }^{ \pi^b } (\omega_{h-1} | \tau_{h-1}  )  \right).
\end{align*}
\end{lemma}
\begin{proof}
We index $\omega_{h-1} = (o_{h},a_h,\ldots,o_H,a_H)$ by $i$, and $\tau_{h-1}$ by $j$. In addition, we denote \\
$\mathbf{m}^*(\omega_{h})^{\top}\left(\hat{\Mbf}_h (o_h,a_h) - \Mbf_h^*(o_h,a_h) \right)  $ by $w_{i}^{\top}$, $ \bar{\hat{\psi}}_{h-1} (\tau_{h-1}) = \frac{ \hat{\psi}_{h-1}(\tau_{h-1})}{\hat{\phi}_{h-1}^{\top}\hat{\psi}_{h-1} (\tau_{h-1})}$ by $x_{j}$, and $\pi(\omega_{h-1}|\tau_{h-1})$ by $\pi_{i|j}$. Then, we have
    \begin{align*}
        \sum_{\tau_H} &\left| \mathbf{m}^* (\omega_h)^{\top} \left(\hat{\Mbf}_h (o_h,a_h) - \Mbf_h^*(o_h,a_h) \right) \hat{\psi}_{h-1} (\tau_{h-1}) \right| \pi(\tau_H) \\
        & =  \sum_{i }\sum_{j }|w_{i}^{\top}x_{j}|\pi_{i|j}\Pb_{ \hat{\theta} }^{\pi}(j)\\
        & =  \sum_{j }  \sum_{i } (\pi_{i|j}\cdot\mathtt{sgn}(w_{i}^{\top}x_{j})\cdot w_{i})^{\top} x_{j}\cdot\Pb_{\theta}^{\pi}(j)\\
        & =  \sum_{j} \left( \sum_{i} \pi_{i|j} \cdot \mathtt{sgn}(w_{i}^{\top}x_{j} )\cdot w_{i} \right)^{\top} x_{j}\cdot\Pb_{ \hat{\theta} }^{\pi}(j)\\
        &\leq \mathop{\Eb}_{j\sim\Pb_{ \hat{\theta} }^{\pi}}\left[ \big\| x_j \big\|_{\left(\hat{U}_{h-1}^{k}\right)^{-1}}  \sqrt{ \left\|\sum_{i} \pi_{i|j} \cdot \mathtt{sgn}(w_{i}^{\top}x_{j} )\cdot w_{i} \right\|_{ \hat{U}_{h-1} }^2 } \right].
    \end{align*}

We fix an index $j = j_0$, and aim to analyze $\left\| \sum_{i} \pi_{i|j_0} \cdot \mathtt{sgn}(w_{i}^{\top}x_{j_0} )\cdot w_{i} \right\|_{\hat{U}_{h-1}^{k}}^2$, which can be written as
\begin{align*}
    &\left\| \sum_{i} \pi_{i|j_0} \cdot \mathtt{sgn}(w_{i}^{\top}x_{j_0} )\cdot w_{i} \right\|_{\hat{U}_{h-1}^{k}}^2\\
    &\quad =  \underbrace{\hat{\lambda} \left\| \sum_{i} \pi_{i|j_0} \cdot \mathtt{sgn}(w_{i}^{\top}x_{j_0} )\cdot w_{i} \right\|_2^2}_{I_1} + \underbrace{\sum_{ j\in\Dc_{h-1} }   \left[ \left(\sum_{i} \pi_{i|j_0} \cdot \mathtt{sgn}(w_{i}^{\top}x_{j_0} )\cdot w_{i}\right)^{\top} x_{j} \right]^2}_{I_2}  .
\end{align*}

For the first term $I_1$, we have
\begin{align*}
    \sqrt{I_1} & = \sqrt{\hat{\lambda } } \max_{x\in\mathbb{R}^{d_{h-1}}: \|x\|_2=1}   \left|\sum_{i} \pi_{i|j_0}\mathtt{sgn}(w_i^{\top}x_{j_0})w_i^{\top}  x\right|\\
    & \leq \sqrt{\hat{\lambda } } \max_{x\in\mathbb{R}^{d_{h-1}}:\|x\|_2=1} \sum_{\omega_{h-1}} \left| \mathbf{m}^*(\omega_h)^{\top} \left( \hat{\Mbf}_h (o_h,a_h) - \Mbf_h^*(o_h,a_h) \right)   x \right| \pi(\omega_{h-1}|j_0)\\
    & \leq \sqrt{\hat{\lambda } } \max_{x:\|x\|_2=1} \sum_{\omega_{h-1}}\left|\mathbf{m}^{*}(\omega_h)^{\top} \hat{\Mbf}_h (o_h,a_h)   x \right| \pi(\omega_{h-1}|j_0) \\
    &\quad + \sqrt{\hat{\lambda } } \max_{x:\|x\|_2=1} \sum_{\omega_{h-1}}\left|\mathbf{m}^*(\omega_h)^{\top}  \Mbf_h^*(o_h,a_h)   x \right| \pi(\omega_{h-1}|j_0) \\
    &\leq \frac{\sqrt{\hat{\lambda } }}{\gamma}\max_{x:\|x\|_2=1} \sum_{o_h,a_h}\left\| \hat{\Mbf}_h (o_h,a_h)   x \right\|_1 \pi(a_h|o_h,j_0)  + \frac{\sqrt{\hat{\lambda } }}{\gamma} \| x\|_1 \\
    &\leq \frac{ 2Q_A\sqrt{d\hat{\lambda}} }{\gamma^2}.
\end{align*}

For the second term $I_2$, we have
\begin{align*}
  I_2 & \leq \sum_{ \tau_{h-1}\in \Dc_{h-1} }   \left( \sum_{\omega_{h-1}} \left| \mathbf{m}^* (\omega_{h})^{\top}\left(\hat{\Mbf}_h (o_h,a_h) -  \Mbf_h^*(o_h,a_h) \right)\bar{\hat{\psi}} (\tau_{h-1}) \right| \pi(\omega_{h-1}|j_0) \right)^2   \\
    & \leq   \sum_{ \tau_{h-1}\in \Dc_{h-1} }   \left( \sum_{\omega_{h-1}} \left|\mathbf{m}^*(\omega_h)^{\top} \left( \hat{\Mbf}_h (o_h,a_h) \bar{\hat{\psi}} (\tau_{h-1}) - \Mbf_h^*(o_h,a_h)  \bar{\psi}^*(\tau_{h-1}) \right) \right| \pi(\omega_{h-1}|j_0) \right.   \\
    &\hspace{3cm} +    \left. \sum_{\omega_{h-1}}\left| \mathbf{m}^*(\omega_h)^{\top}  \Mbf_h^*(o_h,a_h) \left( \bar{\hat{\psi}} (\tau_{h-1}) -  \bar{\psi}^*(\tau_{h-1}) \right) \right| \pi(\omega_{h-1}|j_0) \right)^2     \\
    & \overset{(a)}\leq  \sum_{\tau_{h-1}\in \Dc_{h-1} }   \left( \frac{1}{\gamma} \sum_{o_h,a_h} \left\| \Pb_{ \hat{\theta} } (o_h|\tau_{h-1})\bar{\hat{\psi}} (\tau_{h}) - \Pb_{\theta}(o_h|\tau_{h-1}) \bar{\psi}^*(\tau_{h}) \right\|_1   \pi(a_h|o_h,j_0) \right. \\
    & \hspace{3cm} + \left. \frac{1}{\gamma}     \left\| \bar{\hat{\psi}}  (\tau_{h-1}) -  \bar{\psi}^*(\tau_{h-1})  \right\|_1  \right)^2   \\
    & =  \frac{1}{\gamma^2} \sum_{\tau_{h-1}\in \Dc_{h-1} }   \left(  \sum_{o_h,a_h}  \sum_{\ell=1}^{|\mathcal{Q}_h|} \left| \Pb_{ \hat{\theta} }(\mathbf{o}_h^{\ell},o_h|\tau_{h-1},a_h,\mathbf{a}_h^{\ell}) - \Pb_{ \theta^* }(\mathbf{o}_h^{\ell},o_h|\tau_{h-1},a_h,\mathbf{a}_h^{\ell}) \right|\pi(a_h|o_h,j_0) \right.    \\
    & \hspace{3cm} +  \left.   \sum_{\ell=1}^{|\mathcal{Q}_{h-1}|}   \left| \Pb_{ \hat{\theta} }(\mathbf{o}_{h-1}^{\ell} | \tau_{h-1}, \mathbf{a}_{h-1}^{\ell} ) - \Pb_{ \theta^* } (\mathbf{o}_{h-1}^{\ell} | \tau_{h-1}, \mathbf{a}_{h-1}^{\ell} )  \right|    \right)^2  \\
    &\leq \frac{1}{\gamma^2} \sum_{\tau_{h-1}\in \Dc_{h-1} }   \left( \ \sum_{\mathbf{a}_{h-1}\in \mathcal{Q}_{h-1}^{\exp} }   \sum_{\omega_{h-1}^o }\left|\Pb_{ \hat{\theta} }(\omega^o_{h-1}|\tau_{h-1}, \mathbf{a}_{h-1} ) - \Pb_{\theta^*}(\omega^o_{h-1}|\tau_{h-1}, \mathbf{a}_{h-1} ) \right|    \right)^2  \\
    &\overset{(b)}\leq \frac{ 1 }{ \iota^2 \gamma^2} \sum_{\tau_{h-1}\in\Dc_{h-1} } \mathtt{D}_{\TV}^2\left(    \Pb_{\hat{\theta} }^{ \pi^b }(\omega_{h-1} | \tau_{h-1}  ) , \Pb_{ \theta^* }^{ \pi^b } (\omega_{h-1} | \tau_{h-1}  )  \right),
\end{align*} 
where $(a)$ follows from \Cref{assmp:well-condition}, and $(b)$ follows from the condition $\pi^b(\mathbf{a}_h)\geq \iota$ for any $\mathbf{a}_h\in\mathcal{Q}_{h-1}^{\exp}$.

Thus, we conclude that
\begin{align*}
        \sum_{\tau_H}& \left| \mathbf{m}^*(\omega_h)^{\top} \left(\hat{\Mbf}_h (o_h,a_h) - \Mbf_h^*(o_h,a_h) \right) \hat{\psi} (\tau_{h-1}) \right| \pi(\tau_H) \\
        &\leq \Eb_{\tau_{h-1}\sim \Pb_{ \hat{\theta} }^{\pi}}\left[ \hat{\alpha}_{h-1} \left\|\bar{\hat{\psi}}  (\tau_{h-1})\right\|_{ \left( \hat{U}_{h-1} \right)^{-1}} \right],
    \end{align*}
where 
\begin{align*}
    ( \hat{\alpha}_{h-1} )^2 = \frac{  4\hat{\lambda}Q_A^2d}{\gamma^4} + \frac{1 }{ \iota^2 \gamma^2} \sum_{\tau_{h-1} \in\Dc_{h-1} } \mathtt{D}_{\TV}^2\left(    \Pb_{ \hat{\theta} }^{\pi^b}(\omega_{h-1} | \tau_{h-1}  ) , \Pb_{ \theta^* }^{ \pi^b } ( \omega_{h-1} | \tau_{h-1}  )  \right).
\end{align*}
\end{proof}
\begin{corollary}\label{coro:LCB}
Under event $\Ec^o$, for any reward $R$, we have,
\begin{align*}
    \left| V_{\hat{\theta}  , R }^{\pi} - V_{\theta^* , R}^{\pi} \right| \leq V_{ \hat{\theta}, \hat{b} }^{\pi},
\end{align*}
where $\hat{b}^k(\tau_H) = \min\left\{  \hat{\alpha} \sqrt{ \sum_{h }  \left\|\bar{\hat{\psi}}_h (\tau_h) \right\|^2_{ (\hat{U}_h )^{-1}} } , 1 \right\}$, and $\hat{\alpha} = \sqrt{\frac{ 4\lambda H Q_A^2d }{\gamma^4} + \frac{ 7\hat{\beta} }{ \iota^2 \gamma^2}}.$
\end{corollary}
\begin{proof}
By the definition of the total variation distance, we have
\begin{align*}
    \left| V_{\hat{\theta}  , R }^{\pi} - V_{\theta^* , R}^{\pi} \right| &\leq \mathtt{D}_{\TV}\left( \Pb_{\hat{\theta}}^{\pi}, \Pb_{\theta^*}^{\pi} \right)\\
    &\overset{(a)} \leq \sum_{h=1}^H \sum_{\tau_H} \left| \mathbf{m}^*(\omega_{h})^{\top} \left(\hat{\Mbf}_h (o_h,a_h) - \Mbf_h^*(o_h,a_h) \right) \hat{\psi} (\tau_{h-1}) \right| \pi(\tau_H)\\
    &\overset{(b)} \leq \min\left\{ \sum_{h=1}^H \Eb_{\tau_{h-1}\sim \Pb_{ \hat{\theta} }^{\pi}}\left[ \hat{\alpha}_{h-1} \left\|\bar{\hat{\psi}}  (\tau_{h-1})\right\|_{ \left( \hat{U}_{h-1} \right)^{-1}} \right] , 1 \right\}\\
    &\overset{(c)} \leq \min\left\{ \sqrt{\sum_{h=1}^H \hat{\alpha}_{h-1}^2}\sqrt{ \sum_{h = 0}^{H-1}  \left\| \bar{\hat{\psi}}_h (\tau_h) \right\|^2_{ (\hat{U}_h )^{-1}} }  , 1 \right \},
\end{align*}
where $(a)$ follows from \Cref{prop: TV distance less than estimation error}, $(b)$ follows from \Cref{lemma:offline Estimation error of M less than feature score} and because $R(\tau_H)\in[0,1]$, and $(c)$ follows from the Cauchy's inequality.

By \Cref{lemma:offline MLE guarantee}, we further have
\begin{align*}
    \sum_{h=1}^H \hat{\alpha}_{h-1}^2 
    & \leq \frac{ 4\hat{\lambda} H Q_A^2d  }{ \gamma^4 }  + \frac{1}{ \iota^2 \gamma^2} \sum_h\sum_{ \tau_{h-1}\in\Dc_{h-1} } \mathtt{D}_{\TV}^2\left(    \Pb_{\hat{\theta}^k }^{ \pi^b }(\omega_{h-1} | \tau_{h-1}  ) , \Pb_{ \theta^* }^{ \pi^b } (\omega_{h-1} | \tau_{h-1}  )  \right) \\
    &\leq \frac{ 4\lambda H Q_A^2d }{\gamma^4} + \frac{ 7 \hat{\beta} }{ \iota^2 \gamma^2},
\end{align*}
which concludes the proof.
\end{proof}

\subsection{Step 3: Relationship between Empirical Bonus and Ground-Truth Bonus}

\begin{lemma}[Offline empirical bonus and true bonus term]\label{lemma:offline empirical bonus and true bonus}
 Under event $\Ec^o$, for any $\pi$, we have
    \begin{align*}
        \mathop{\Eb}_{\tau_H\sim\Pb_{\theta^*}^{\pi}}& \left[\sqrt{ \sum_{h=0}^{H-1} \left\|\bar{\hat{\psi}}_h (\tau_h) \right\|^2_{\left(\hat{U}_h\right)^{-1}} } \right] \\
        &\leq \left(1 + \frac{2 \sqrt{7 r\hat{\beta}}}{ \iota \sqrt{\hat{\lambda}}} \right)  \sum_{h=0}^{H-1} \mathop{\Eb}_{\tau_h\sim\Pb_{\theta^*}^{\pi}} \left[\left\|\bar{\psi}_h^* (\tau_h) \right\|_{\left(U_h\right)^{-1}} \right] + \frac{2HQ_A}{ \sqrt{\hat{\lambda}}}  \mathtt{D}_{\TV}\left(  \Pb_{\theta^*}^{\pi}(\tau_H) , \Pb_{\hat{\theta}}^{\pi}(\tau_H)\right).
    \end{align*}
\end{lemma}
\begin{proof}
Recall that 
\begin{align*}
    \hat{U}_h = \hat{\lambda} I + \sum_{\tau_h\in\Dc_h} \bar{\hat{\psi}}(\tau_h)\bar{\hat{\psi}}(\tau_h)^{\top}.
\end{align*}
We define the ground-truth counterpart of $\hat{U}_h$ as follows:
\begin{align*}
    U_h = \hat{\lambda} I + \sum_{\tau_h\in\Dc_h} \bar{ \psi }(\tau_h)\bar{ \psi }(\tau_h)^{\top}.
\end{align*}

Then, by \Cref{lemma:transfer vector score}, we have
\begin{align*}
     &\sqrt{ \sum_h \left\|\bar{\hat{\psi}} (\tau_h) \right\|^2_{\left( \hat{U}_h \right)^{-1}} }\\
     &\quad \leq \sum_h \left\|\bar{\hat{\psi}}_h (\tau_h) \right\|_{\left( \hat{U}_h \right)^{-1}} \\
     &\quad \leq \frac{1}{\sqrt{\hat{\lambda}}} \sum_h \left\| \bar{\hat{\psi}} (\tau_h) - \bar{\psi}_h^* (\tau_h) \right\|_2 + \sum_h \left(1 + \frac{ \sqrt{r} \sqrt{ \sum_{\tau_h\in\Dc_h } \left\| \bar{\hat{\psi}}(\tau_h) - \bar{\psi}^* (\tau_h) \right\|_2^2 } }{ \sqrt{\hat{\lambda}} } \right) \left\|\bar{\psi}^* (\tau_h) \right\|_{\left(U_h\right)^{-1}} .
\end{align*}

Furthermore, note that
\begin{align*}
    \left\| \bar{\hat{\psi}} (\tau_h) - \bar{\psi}^* (\tau_h) \right\|_2 & \leq \left\| \bar{\hat{\psi}}_h (\tau_h) - \bar{\psi}_h^* (\tau_h) \right\|_1 \\
    & \leq \frac{1}{\iota} \mathtt{D}_{\TV}\left( \Pb_{ \hat{\theta}  }^{ \pi^b }(\omega_h|\tau_h), \Pb_{   \theta^* }^{ \pi^b }(\omega_h|\tau_h) \right),
\end{align*}
where the second inequality follows because $\pi^b(\mathbf{a}_h) \geq \iota$ for all $\mathbf{a}_h\in\mathcal{Q}_h^{\exp}$. 
By \Cref{lemma:offline MLE guarantee}, we conclude that
\begin{align*}
    &\sqrt{ \sum_h \left\|\bar{\hat{\psi}}_h (\tau_h) \right\|^2_{\left(\hat{U}_h \right)^{-1}} }\\
    &\quad  \leq \frac{1}{\sqrt{\hat{\lambda}}} \sum_h \left\| \bar{\hat{\psi}}_h (\tau_h) - \bar{\psi}_h^* (\tau_h) \right\|_2 + \left(1 + \frac{  \sqrt{7r\hat{\beta}}   }{ \iota \sqrt{\hat{\lambda}} } \right) \sum_h  \left\|\bar{\psi}_h^* (\tau_h) \right\|_{\left(U_h^{k}\right)^{-1}}.
\end{align*}

For the first term, following an argument similar to those in \Cref{lemma:empirical bouns less than true bonus} and taking the expectation, we have
\begin{align*}
    \sum_h & \Eb_{\tau_h\sim \Pb_{\theta^*}^{\pi}} \left[  \left\| \bar{\hat{\psi}}(\tau_h) - \bar{\psi}(\tau_h) \right\|_1   \right] \\
    &\leq  \sum_h \sum_{\tau_h} \left( \left\| \bar{\hat{\psi}}(\tau_h)\left( \Pb_{\theta}^{\pi}(\tau_h) - \Pb_{\hat{\theta}}^{\pi}(\tau_h) \right) + \bar{\hat{\psi}}(\tau_h) \Pb_{\hat{\theta}}^{\pi}(\tau_h) - \bar{\psi}(\tau_h) \Pb_{\theta^*}^{\pi}(\tau_h) \right\|_1 \right)\\
    &\leq  \sum_h \sum_{\tau_h} \left( \left\| \bar{\hat{\psi}}(\tau_h)  \right\|_1 \left| \Pb_{\theta}^{\pi}(\tau_h) - \Pb_{\hat{\theta}}^{\pi}(\tau_h) \right|  + \left\|  \hat{\psi}(\tau_h)   - \bar{\psi}^*(\tau_h)   \right\|_1 \pi(\tau_h)\right)\\
    &\leq 2\sum_h |\mathcal{Q}_h^A| \mathtt{D}_{\TV}\left(  \Pb_{\theta^*}^{\pi}(\tau_h) , \Pb_{\hat{\theta}}^{\pi}(\tau_h)\right) \\
    &\leq 2HQ_A \mathtt{D}_{\TV}\left(  \Pb_{\theta^*}^{\pi}(\tau_H) , \Pb_{\hat{\theta}}^{\pi}(\tau_H)\right),
\end{align*}
which completes the proof.
\end{proof}

\subsection{Proof of \Cref{thm:offline PSR}}
Now, we are ready to prove \Cref{thm:offline PSR}.
\begin{theorem}[Restatement of \Cref{thm:offline PSR}]
   Suppose \Cref{assmp:well-condition} holds. Let $\iota = \min_{\mathbf{a}_h\in\mathcal{Q}_h^{\exp}} \pi^b(\mathbf{a}_h) $, $p_{\min} = O(\frac{\delta}{KH(|\mathcal{O}||\mathcal{A}|)^H})$, $\varepsilon=O(\frac{p_{\min}}{KH})$, $\hat{\beta} = O(\log|\bar{\Theta}_{\varepsilon}|)$, $\hat{\lambda} = \frac{\gamma C_{\pi^b,\infty}^{\pi}\hat{\beta} \max\{\sqrt{r}, Q_A\sqrt{H}/\gamma\}}{\iota^2Q_A\sqrt{dH}}$, and $\hat{\alpha} = O\left(\frac{  Q_A\sqrt{dH} }{\gamma^2}\sqrt{\lambda} + \frac{  \sqrt{\beta} }{ \iota \gamma} \right)$. Then, with probability at least $1-\delta$, the output $\bar{\pi}$ of \Cref{alg:offline PSR} satisfies that
   \begin{align*}
       \forall \pi,~~ V_{\theta^*, R }^{\pi} - V_{\theta^*,R}^{\bar{\pi}} \leq  \tilde{O}\left( \left(\sqrt{r} + \frac{ Q_A\sqrt{H}}{\gamma}\right)\frac{ C_{\pi^b,\infty}^{\pi}  Q_A H^2  }{  \iota  \gamma^2 }\sqrt{\frac{r d \hat{\beta}}{K}} \right).
   \end{align*}
\end{theorem}
\begin{proof}
Under event $\Ec^o$, the performance difference can be upper bounded as follows. 
    \begin{align*}
        V_{\theta^*,r}^{\pi} &- V_{\theta^*,r}^{\hat{\pi}}   \\
        & = V_{\theta^*,r}^{\pi} - (V_{\hat{\theta}, r}^{\pi}  - V_{\hat{\theta}, \hat{b}}^{\pi}) + \underbrace{( V_{\hat{\theta}, r}^{\pi} - V_{\hat{\theta}, \hat{b}}^{\pi} ) - ( V_{\hat{\theta}, r}^{\hat{\pi}} - V_{\hat{\theta}, \hat{b}}^{\hat{\pi}} )}_{I_1} +  \underbrace{V_{\hat{\theta}, r}^{\hat{\pi}} - V_{\hat{\theta}, \hat{b}}^{\hat{\pi}}  - V_{\theta^*,r}^{\hat{\pi}}}_{I_2}\\
        &\overset{(a)}\leq \frac{1}{2}\mathtt{D}_{\TV} \left(\Pb_{\hat{\theta}}^{\pi} , \Pb_{\theta^*}^{\pi} \right) +  V_{\hat{\theta}, \hat{b}}^{\pi} \\
        &\overset{(b)}\leq  \mathtt{D}_{\TV} \left(\Pb_{\hat{\theta}}^{\pi} , \Pb_{\theta^*}^{\pi} \right) +  V_{\theta^*, \hat{b}}^{\pi},
    \end{align*}
where $(a)$ follows from the design of $\bar{\pi}$ that results in $I_1\leq 0$, and from \Cref{coro:LCB} which implies $I_2\leq 0$, and $(b)$ follows from \Cref{coro:LCB}.

By \Cref{lemma: offline TV distance less than true bonus}, we have
\begin{align*}
    \mathtt{D}_{\TV} \left(\Pb_{\hat{\theta}}^{\pi} , \Pb_{\theta^*}^{\pi} \right) & \lesssim\sqrt{\frac{\hat{\beta}}{H\iota^2\gamma^2}} \sum_{h=1}^H \mathop{\Eb}_{\tau_{h-1}\sim\Pb_{\theta^*}^{\pi}} \left[ \     \left\|\bar{\psi}^* (\tau_{h-1})\right\|_{  \Lambda_{h-1}^{-1}}  \right] \\
    &\overset{(a)}\leq C_{\pi^b,\infty}^{\pi}\sqrt{\frac{\hat{\beta}}{H\iota^2\gamma^2}} \sum_{h=1}^H \mathop{\Eb}_{\tau_{h-1}\sim\Pb_{\theta^*}^{\pi^b}} \left[ \     \left\|\bar{\psi}^* (\tau_{h-1})\right\|_{  \Lambda_{h-1}^{-1}}  \right]  \\
    & \leq C_{\pi^b,\infty}^{\pi}H\sqrt{\frac{\hat{\beta}}{H\iota^2\gamma^2}} \sqrt{\frac{rH}{K}}  = C_{\pi^b,\infty}^{\pi}H\sqrt{\frac{r\hat{\beta}}{K\iota^2\gamma^2}},
\end{align*}
where $(a)$ follows from the definition of the coverage coefficient. 

By \Cref{lemma:offline empirical bonus and true bonus}, we have
\begin{align*}
    &V_{\theta^*, \hat{b}}^{\pi} \\
    & \leq \min\left\{ \hat{\alpha} \left(1 + \frac{2 \sqrt{7 r\hat{\beta}}}{ \iota \sqrt{\hat{\lambda}}} \right)  \sum_{h=0}^{H-1} \mathop{\Eb}_{\tau_h\sim\Pb_{\theta^*}^{\pi}} \left[\left\|\bar{\psi}^* (\tau_h) \right\|_{\left(U_h\right)^{-1}} \right] + \frac{2\hat{\alpha}HQ_A}{ \sqrt{\hat{\lambda}}}  \mathtt{D}_{\TV}\left(  \Pb_{\theta^*}^{\pi}(\tau_H) , \Pb_{\hat{\theta}}^{\pi}(\tau_H)\right), 1  \right\} \\
    &\leq    \min\left\{ \hat{\alpha} \left(1 + \frac{2 \sqrt{7 r\hat{\beta}}}{ \iota \sqrt{\hat{\lambda}}} \right)  \sum_{h=0}^{H-1} \mathop{\Eb}_{\tau_h\sim\Pb_{\theta^*}^{\pi}} \left[\left\|\bar{\psi}^* (\tau_h) \right\|_{\left(U_h\right)^{-1}} \right], 1  \right\}  + \frac{2\hat{\alpha}HQ_A}{ \sqrt{\hat{\lambda}}}  \mathtt{D}_{\TV}\left(  \Pb_{\theta^*}^{\pi}(\tau_H) , \Pb_{\hat{\theta}}^{\pi}(\tau_H)\right)\\
    &\leq   \underbrace{ \min\left\{ \hat{\alpha}C_{\pi^b,\infty}^{\pi} \left(1 + \frac{2 \sqrt{7 r\hat{\beta}}}{ \iota \sqrt{\hat{\lambda}}} \right)  \sum_{h=0}^{H-1} \mathop{\Eb}_{\tau_h\sim\Pb_{\theta^*}^{\pi^b}} \left[\left\|\bar{\psi}^* (\tau_h) \right\|_{\left(U_h\right)^{-1}} \right], 1  \right\} }_{I_3} + \frac{2\hat{\alpha}HQ_A}{ \sqrt{\hat{\lambda}}}  \mathtt{D}_{\TV}\left(  \Pb_{\theta^*}^{\pi}(\tau_H) , \Pb_{\hat{\theta}}^{\pi}(\tau_H)\right).
\end{align*}

Recall that $U_h = \hat{\lambda} I + \sum_{\tau_h\in\Dc_h} \bar{\psi}^*(\tau_h)\bar{\psi}^*(\tau_h)^{\top}$ and the distribution of $\tau_h\in\Dc_h$ follows $\Pb_{\theta^*}^{\pi^b}$. Therefore, by the Azuma-Hoeffding's inequality (\Cref{lemma:hoeffding}), with probability at least $1-\delta$, we have
\begin{align*}
    KI_3 &\leq \sqrt{2K\log(2/\delta)} +  \min\left\{ \hat{\alpha}C_{\pi^b,\infty}^{\pi} \left(1 + \frac{2 \sqrt{7 r\hat{\beta}}}{ \iota \sqrt{\hat{\lambda}}} \right)  \sum_{h=0}^{H-1} \sum_{\tau_h\in\Dc_h} \left\|\bar{\psi}^* (\tau_h) \right\|_{\left(U_h\right)^{-1}}  , 1  \right\} \\
    &\overset{(a)}\lesssim \sqrt{K\log(2/\delta)} + \hat{\alpha}C_{\pi^b,\infty}^{\pi}\left(1+\frac{ \sqrt{r\hat{\beta}} }{ \iota\sqrt{\hat{\lambda}} } \right) H \sqrt{\frac{rK}{H}} \\
    &\lesssim \hat{\alpha}C_{\pi^b,\infty}^{\pi}\left(1+\frac{ \sqrt{r\hat{\beta}} }{ \iota\sqrt{\hat{\lambda}} } \right)  \sqrt{rHK\log(2/\delta)},
\end{align*}
where $(a)$ follows from the Cauchy's inequality.

Combining the above results, we have
\begin{align*}
    V_{\theta^*,r}^{\pi} &- V_{\theta^*,r}^{\hat{\pi}}\\
    &\lesssim  \hat{\alpha}C_{\pi^b,\infty}^{\pi}\left(1+\frac{ \sqrt{r\hat{\beta}} }{ \iota\sqrt{\hat{\lambda}} } \right)  \sqrt{\frac{rH\log(2/\delta)}{K}} + \frac{\hat{\alpha}HQ_A}{\sqrt{\hat{\lambda}}} \frac{C_{\pi^b,\infty}^{\pi} H}{\iota\gamma}\sqrt{\frac{r\hat{\beta}}{K}} \\
    &\lesssim C_{\pi^b,\infty}^{\pi}\left(\frac{Q_A\sqrt{dH} }{\gamma^2}\sqrt{\hat{\lambda}} + \frac{\sqrt{\hat{\beta}}}{\iota\gamma}\right) \left( 1 + \frac{ \max\{ \sqrt{r} , \sqrt{H}Q_A/\gamma\}  }{\iota\sqrt{\hat{\lambda}}} \right) H\sqrt{\frac{rH\hat{\beta}\log(2/\delta)}{K}}\\
    & = C_{\pi^b,\infty}^{\pi}\left(\frac{\iota Q_A\sqrt{dH} }{\gamma\sqrt{\hat{\beta}}}\sqrt{\hat{\lambda}} + 1 \right) \left( 1 + \frac{ \max\{ \sqrt{r} , \sqrt{H}Q_A/\gamma\}  }{\iota\sqrt{\hat{\lambda}}} \right) \frac{H\sqrt{\hat{\beta}}}{\iota\gamma}\sqrt{\frac{rH\hat{\beta}\log(2/\delta)}{K}}\\
    &\overset{(a)}\lesssim \max\left\{ \sqrt{r} , \frac{\sqrt{H}Q_A}{\gamma} \right\} C_{\pi^b,\infty}^{\pi} \frac{Q_AH^2\sqrt{d}}{\iota\gamma^2} \sqrt{\frac{r\hat{\beta}\log(2/\delta)}{K}},
\end{align*}
where $(a)$ follows because
 \begin{align*}
        \lambda = \frac{\gamma  \sqrt{\hat{\beta}} \max\{\sqrt{r}, Q_A\sqrt{H}/\gamma\}}{\iota^2Q_A\sqrt{dH}}.
    \end{align*}
This completes the proof.
\end{proof}

\section{Auxiliary Lemmas}

\begin{lemma}[Azuma–Hoeffding inequality]\label{lemma:hoeffding}
    Consider a domain $\mathcal{X}$, and a filtration $\mathcal{F}_1\subset\ldots\subset\mathcal{F}_k\subset\ldots $ on the domain $\mathcal{X} $. Suppose that  $\{X_1,\ldots,X_k,\ldots\}\subset[-B,B]$ is adapted to the filtration $(\mathcal{F}_t)_{t=1}^{\infty}$, i.e., $X_k$ is $\mathcal{F}_k$-measurable. Then
\begin{align}
    \Pb\left( \left|\sum_{t=1}^{k-1}  X_t - \Eb[ X_t] \right| \geq \sqrt{2kB^2\log(2/\delta)} \right) \leq \delta.
\end{align}
\end{lemma}

The following lemma characterize the relationship between the total variation distance and the Hellinger-squared distance. Note that the result for probability measures has been proved in Lemma H.1 in \citet{zhong2022posterior}. Since we consider  more general bounded measures, we provide the full proof for completeness.
\begin{lemma}\label{lemma:TV and hellinger}
Given two bounded measures $P$ and $Q$ defined on the set $\mathcal{X}$. Let $|P| = \sum_{x\in\mathcal{X}} P(x)$ and $|Q| = \sum_{x\in\mathcal{X}}Q(x).$ We have
\[ \mathtt{D}_{\TV}^2(P,Q)  \leq 4(|P|+|Q|)\mathtt{D}_{\mathtt{H}}^2(P,Q)  \]
In addition, if $P_{Y|X}, Q_{Y|X}$ are two conditional distributions over a random variable $Y$, and $P_{X,Y} = P_{Y|X}P$, $Q_{X,Y}= Q_{Y|X}Q$ are the joint distributions when $X$ follows the distributions $P$ and $Q$, respectively, we have
\[\mathop{\Eb}_{X\sim P }\left[ \mathtt{D}_{\mathtt{H}}^2(P_{Y|X},Q_{Y|X})\right] \leq 8\mathtt{D}_{\mathtt{H}}^2 (P_{X,Y},Q_{X,Y}). \]
\end{lemma}
\begin{proof}
We first prove the first inequality.  By the definition of total variation distance, we have
    \begin{align*}
        \mathtt{D}_{\TV}^2(P,Q) & = \left(\sum_x |P(x) - Q(x)| \right)^2\\
        & = \left(\sum_x \left(\sqrt{P(x)} - \sqrt{Q(x)}\right)\left(\sqrt{P(x)} + \sqrt{Q(x)} \right) \right)^2\\
        &\overset{(a)}\leq \left(\sum_x\left(\sqrt{P(x)} - \sqrt{Q(x)}\right)^2\right) \left( 2\sum_{x}\left( P(x) + Q(x)\right) \right)\\
        &\leq 4(|P| + |Q|) \mathtt{D}_{\mathtt{H}}^2(P,Q),
    \end{align*}
where $(a)$ follows from the Cauchy's inequality and because $(a+b)^2\leq 2a^2+2b^2$.

For the second inequality, we have,
\begin{align*}
    \mathop{\Eb}_{X\sim P } &\left[ \mathtt{D}_{\mathtt{H}}^2(P_{Y|X},Q_{Y|X})\right]\\
    &= \sum_{x}P(x)\left( \sum_y \left(\sqrt{P_{Y|X}(y)} - \sqrt{Q_{Y|X}(y)} \right)^2 \right) \\
    & = \sum_{x,y} \left( \sqrt{P_{X,Y}(x,y)} - \sqrt{Q_{X,Y}(x,y)} + \sqrt{Q_{Y|X}(y) Q(x)} - \sqrt{Q_{Y|X}(y) P(x)}  \right)^2 \\
    &\leq 2\sum_{x,y} \left( \sqrt{P_{X,Y}(x,y)} - \sqrt{Q_{X,Y}(x,y)} \right)^2 + 2\sum_{x,y} Q_{Y|X}(y)\left( \sqrt{  Q(x)} - \sqrt{  P(x)}  \right)^2 \\
    & = 4\mathtt{D}_{\mathtt{H}}^2(P_{X,Y},Q_{X,Y}) + 2(|P|+|Q| - 2\sum_x\sqrt{P(x)Q(x)}) \\
    &\overset{(a)}\leq 4\mathtt{D}_{\mathtt{H}}^2(P_{X,Y},Q_{X,Y}) + 2(|P|+|Q| - 2\sum_x\sum_y\sqrt{P_{Y|X}(y)P(x)Q_{Y|X}(y)Q(x)}) \\
    & = 8 \mathtt{D}_{\mathtt{H}}^2(P_{X,Y},Q_{X,Y}),
\end{align*}
where $(a)$ follows from the Cauchy's inequality that applies on $\sum_y\sqrt{P_{Y|X}(y)Q_{Y|X}(y)}$.
\end{proof}

\begin{lemma}\label{lemma:transfer vector score}
    Consider two vector sequences $\{x_i\}_{i\in\mathcal{I}}$ and $\{y_i\}_{i\in \mathcal{I}}$ and an index subset $\mathcal{J}\subset \mathcal{I}$. Suppose $A = \lambda I + \sum_{j\in\mathcal{J}} x_jx_j^{\top}$ and $B = \sum_{j\in\mathcal{J}}y_jy_j^{\top}$, and $\mathtt{rank} \left( \{x_i\}_{i\in\mathcal{I}} \right) = \mathtt{rank} \left( \{y_i\}_{i\in\mathcal{I}} \right) = r$. Then
    \begin{align*}
        \forall i\in\mathcal{I},~~ \|x_i\|_{A^{-1}} \leq \frac{1}{\sqrt{\lambda}} \| x_i - y_i \|_2 + \left(  1 + \frac{2\sqrt{r}\sqrt{\sum_{j\in\mathcal{J}} \|x_j-y_j\|_2^2 } }{\sqrt{\lambda}}\right) \|y_i\|_{B^{-1}}.
    \end{align*}
\end{lemma}
\begin{proof}
We first write
\begin{align*}
    & \|x_{i}\|_{A^{-1}} - \|y_{i}\|_{B^{-1}} \\
    & \quad =  \|x_{i}\|_{A^{-1}} - \|y_{i}\|_{A^{-1}} + \|y_{i}\|_{A^{-1}} - \|y_{i}\|_{B^{-1}}\\
    &\quad = \frac{ \overbrace{ \|x_{i}\|^2_{A^{-1}} -    \|y_{i}\|^2_{A^{-1}} }^{I_1} }{ \|x_{i}\|_{A^{-1}} +  |y_{i}\|_{A^{-1}} } + \frac{  \overbrace{  \|y_{i}\|^2_{A^{-1}}  - \|y_{i}\|^2_{B^{-1}} }^{I_2} }{ \|y_{i}\|_{A^{-1}} + \|y_{i}\|_{B^{-1}} }.
\end{align*}

For the first term $I_1$, we repeatedly apply the Cauchy's inequality, and have
\begin{align*}
    \|x_{i} &\|_{A^{-1}}^2 - \|y_{i}\|_{A^{-1}}^2\\
    & = x_{i}^{\top} A^{-1} (x_{i} - y_{i}) + y_{i}^{\top} A^{-1} (x_{i} - y_{i})\\
    & \leq \|x_{i}\|_{A^{-1}}\|x_{i} - y_{i}\|_{A^{-1}} + \|y_{i}\|_{A^{-1}}\|x_{i}-y_{i}\|_{A^{-1}}\\
    & \leq \frac{1}{\sqrt{\lambda}} \left(\|x_{i}\|_{A^{-1}} + \|y_{i}\|_{A^{-1}} \right) \left\| x_{i} - y_{i} \right\|_{2}.
\end{align*}

For the second term $I_2$, we repeatedly apply the Cauchy's inequality, and have
\begin{align*}
    \|y_{i} &\|_{A^{-1}}^2 - \|y_{i}\|_{B^{-1}}^2\\
    & = y_{i}^{\top}A^{-1}(B-A)B^{-1}y_{i}\\
    & = \sum_{j\in\mathcal{J}} \bigg(y_{i}^{\top}A^{-1}x_{j} (y_{j} - x_{j})^{\top}B^{-1}y_{i} +    y_{i}^{\top}A^{-1}(y_{j} -x_{j})y_{j}^{\top}B^{-1}y_{i} \bigg)\\
    &\leq  \sum_{j\in\mathcal{J} }   \|x_{j}\|_{A^{-1}}    \|y_{j}-x_{j}\|_{B^{-1}}    \|y_{i}\|_{A^{-1}}\|y_{i}\|_{B^{-1}} \\
    & \quad + \sum_{j\in\mathcal{J}}  \|y_{j}-x_{j}\|_{A^{-1}} \|y_{j}\|_{B^{-1}}   \|y_{i}\|_{A^{-1}} \|y_{i}\|_{B^{-1}}\\
    & \leq \frac{1}{\sqrt{\lambda}}\|y_{i}\|_{A^{-1}}\|y_{i}\|_{B^{-1}} \sqrt{ \sum_{j\in\mathcal{J}}   \|x_{j}-y_{j}\|_{2}^2  } \sqrt{ \sum_{j\in\mathcal{J}} \|x_{j}\|^2_{A^{-1}} }\\
    &\quad + \frac{1}{\sqrt{\lambda}} \|y_{i}\|_{A^{-1}} \|y_{i}\|_{B^{-1}} \sqrt{ \sum_{j\in\mathcal{J} } \|x_{j}-y_{j}\|_{2}^2  }\sqrt{ \sum_{j\in\mathcal{J}}   \|y_{j}\|_{B^{-1}}^2 } \\
    &\leq \frac{2\sqrt{r}}{\sqrt{\lambda}} \|y_{i}\|_{A^{-1}} \|y_{i}\|_{B^{-1}} \sqrt{ \sum_{j\in\mathcal{J} } \|x_{j_h}-y_{j_h}\|_{2}^2  }.
\end{align*}
Therefore, the lemma follows from the fact that $\|y_i\|_{A^{-1}} \leq \|y_i\|_{A^{-1}} + \|y_i\|_{B^{-1}}.$ 
\end{proof}

The following lemma and its variants has been developed in  \citet{dani2008stochastic,abbasi2011improved,carpentier2020elliptical}. We slightly generalize the padding term from $1$ to an arbitrary  positive number $B$ and provide the full proof for completeness.

\begin{lemma}[Elliptical potential lemma  ]\label{lemma:elliptical potential lemma}
    For any sequence of vectors $\mathcal{X} = \{x_1,\ldots,x_n,\ldots\}\subset\mathbb{R}^d$, let $U_k = \lambda I + \sum_{t<k}x_kx_k^{\top}$, where $\lambda $ is a positive constant, and $B>0$ is a real number.  If the rank of $\mathcal{X}$ is at most $r$, then, we have
    \[ 
    \begin{aligned}
        & \sum_{k=1}^K \min\left\{\|x_k\|_{U_k^{-1}}^2 , B\right\}\leq   (1+B)r\log(1+K/\lambda), \\
        &\sum_{k=1}^K \min\left\{\|x_k\|_{U_k^{-1}} , \sqrt{B}\right\} \leq \sqrt{(1+B)rK\log(1+K/\lambda)}.
    \end{aligned}
    \]
\end{lemma}
\begin{proof}
    Note that the second inequality is an immediate result from the first inequality by the Cauchy's inequality. Hence, it suffices to prove the first inequality. To this end, we have
\begin{align*}
    \sum_{k=1}^K  \min\left\{\|x_k\|_{U_k^{-1}}^2 , B \right\} & \overset{(a)}\leq (1+B)\sum_{k=1}^K  \log\left(1 + \|x_k \|_{U_k^{-1}}^2 \right)\\
    & =   (1+B)\sum_{k=1}^K  \log \left( 1+ \mathtt{trace}\left( \left(U_{k+1} - U_k \right) U_k^{-1} \right) \right) \\
    &= (1+B)\sum_{k=1}^K \log\left( 1+ \mathtt{trace}\left( U_k^{-1/2}\left(U_{k+1} - U_k \right) U_k^{-1/2} \right) \right)\\
    &\leq (1+B)\sum_{k=1}^K \log\mathtt{det}\left(I_d + U_k^{-1/2}\left(U_{k+1} - U_k \right) U_k^{-1/2} \right)\\
    &=  (1+B)\sum_{k=1}^K \log \frac{\mathtt{det}\left(U_{k+1}\right)}{\mathtt{det}(U_k)}\\
    & = (1+B) \log \frac{\mathtt{det}(U_{K+1})}{\mathtt{det}(U_1)} \\
    & = (1+B)\log\mathtt{det}\left( I + \frac{1}{\lambda}\sum_{k=1}^{K}x_kx_k^{\top}  \right) \\
    &\overset{(b)}\leq (1+B)r\log(1+K/\lambda),
\end{align*}
where $(a)$ follows because $x\leq (1+B)\log(1+x)$ if $0<x\leq B$, and $(b)$ follows because $\mathtt{rank}(\mathcal{X}) \leq r.$
\end{proof}

\end{document}